\documentclass{article} 
\usepackage{iclr2027_conference,times}

\usepackage{booktabs}
\usepackage{tabularx}
\usepackage{array}
\usepackage{amssymb} 
\usepackage{amsmath}
\usepackage{booktabs}
\usepackage{graphicx}
\usepackage{bbm}
\usepackage{amssymb}
\usepackage{amsthm}
\usepackage{placeins}
\usepackage{fvextra}
\usepackage{amsmath,amsfonts,bm}

\def\eqref#1{equation~\ref{#1}}

\def\1{\bm{1}}

\DeclareMathAlphabet{\mathsfit}{\encodingdefault}{\sfdefault}{m}{sl}
\SetMathAlphabet{\mathsfit}{bold}{\encodingdefault}{\sfdefault}{bx}{n}

\usepackage{booktabs}
\usepackage{hyperref}
\usepackage{url}
\usepackage{graphicx}
\usepackage{amsthm}

\newtheorem{theorem}{Theorem}

\newtheorem{corollary}[theorem]{Corollary}

\theoremstyle{definition}

\theoremstyle{remark}

\title{The GUI Is Not the State: Diagnosing State Aliasing in GUI World Models}

\author{
\textbf{Dongsheng Liu}\textsuperscript{1,2}\thanks{Equal contribution.}\quad
\textbf{Chao Jin}\textsuperscript{2,3}\footnotemark[1]\quad
\textbf{Wenkui Yang}\textsuperscript{1}\quad
\textbf{Hejin Wang}\textsuperscript{4}\quad
\textbf{Junwei Yang}\textsuperscript{4}
\\
\textbf{ Zeren Zhang}\textsuperscript{4}\quad
\textbf{Ziwei Chen}\textsuperscript{4}\quad
\textbf{Huaibo Huang}\textsuperscript{2,3}\quad
\textbf{Jie Cao}\textsuperscript{2,3}\quad
\textbf{Ran He}\textsuperscript{1,2,3}
\\[0.7em]
\textsuperscript{1}School of Advanced Interdisciplinary Sciences,
University of Chinese Academy of Sciences
\\
\textsuperscript{2}MAIS\&NLPR, Institute of Automation,
Chinese Academy of Sciences
\\
\textsuperscript{3}School of Artificial Intelligence, University of Chinese Academy of Sciences
\\
\textsuperscript{4}Huawei Noah's Ark Lab
}

\iclrfinalcopy

\begin{document}

\maketitle
\lhead{}
\chead{}
\rhead{}

\begin{abstract}
GUI World Models (GUI-WMs) are increasingly used to predict future \emph{states} for agent planning and simulation, yet most existing formulations condition only on the current GUI \emph{observation} and \emph{action}. We identify \emph{state aliasing}, where the visible interface omits transition-relevant environment state, so
identical observable conditions can correspond to different valid futures.
To diagnose this failure mode, we introduce \textbf{StateAliasBench}, a
diagnostic benchmark that explicitly isolates such ambiguities via strict pairing. We further
propose lightweight predictive-state recovery that infers structured
state from history and augments otherwise frozen GUI-WMs through
a deterministic state interface. Family-specific specialists provide state recovery across heterogeneous state types, and multi-teacher
distillation consolidates them into a single unified estimator. Experiments show that existing GUI-WMs exhibit systematic failures under
observation-only conditioning, while predictive-state augmentation
substantially restores state-sensitive prediction across evaluated WMs,
preserves generative fidelity, and improves downstream performance of GUI agents on AndroidWorld. These results suggest that reliable GUI world
modeling should account not only for what is visible, but also for the
hidden transition state that determines what happens next.
\end{abstract}

\section{Introduction}

Multimodal Large Language Models (MLLMs) \citep{qwen3.8,opus5.5,gpt-6,kimi-k3} have substantially advanced GUI agents \citep{aguvis,os-atlas,uitars,opencua,guiowl} across web \citep{webarena, mind2web, visualwebarena}, mobile \citep{androidworld,mobileworld}, and desktop \citep{osworld,osworld-2} environments. Recent work further equips these agents with GUI World Models (GUI-WMs) \citep{luo2025vimo,koh2026gworld,zheng2026code2world,xu2026howmobileworldmodel}, which predict future interfaces under candidate actions and support planning, action selection, and simulated interaction.

Most existing GUI-WMs adopt an observation-conditioned formulation, predicting the next GUI from only the current \emph{observation} and \emph{action}, since complete environment \emph{state} is difficult to collect at scale. However, the current GUI need not contain all transition-relevant information. For example, after an agent moves a file, the resulting file-system state may no longer be visible, causing an observation-only world model to later predict the destination without the moved file. An agent can proactively recover missing information through further interaction, whereas the world model must predict from its available conditioning.

We identify this observation--state mismatch as \textbf{state aliasing} \citep{chrisman1992perceptual,subramanian2022approximate,halluclear}. Identical observable conditions can correspond to different transition-relevant states and therefore different valid futures: \emph{what is currently observed} is not always \emph{what state is currently true} (Figure~\ref{fig:state_aliasing_teaser}, left).

\begin{figure*}[t]
    \centering
    \includegraphics[width=\textwidth]{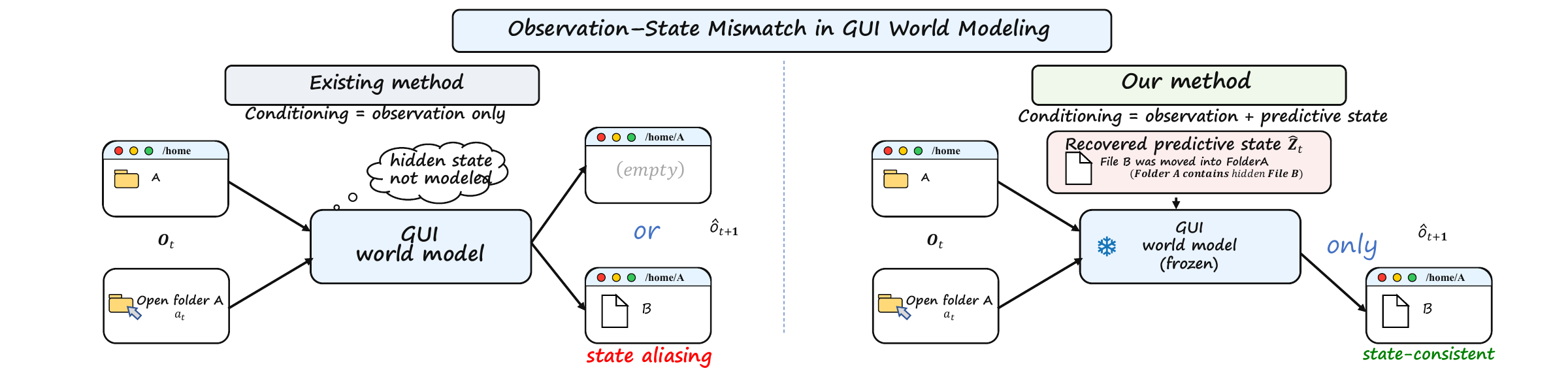}
    \caption{
    \textbf{State aliasing and predictive-state augmentation.}
    Observation-only conditioning can alias histories with different
    transition-relevant states (left). We recover predictive state from history
    and use it to condition a frozen GUI world model for state-consistent
    prediction (right).
    }
    \label{fig:state_aliasing_teaser}
\end{figure*}

To systematically diagnose this failure mode, we construct \textbf{StateAliasBench}, a controlled benchmark built around \emph{strict pairs}, since conventional trajectory-level evaluation protocols sidestep state aliasing. In contrast, StateAliasBench constructs such cases explicitly: two real interaction branches provide the world model with the same current observation and next action, while differing in transition-relevant hidden state and requiring different validated future outcomes. The benchmark covers four qualitatively different state families with 25 causal templates and 800 verified strict pairs. This design isolates a fundamental question that conventional prediction metrics may miss: whether the visible GUI provides sufficient state information to determine which future should occur.

Our experiments show that this distinction matters substantially. Representative GUI-WMs fail to distinguish strict-pair transitions under standard observation-only conditioning. The raw interaction history recovers only a limited portion of these failures, whereas providing the correct transition-relevant state restores most of the missing predictive capability, indicating that more context is not necessarily a sufficient state representation. Motivated by this diagnosis, we introduce a lightweight state-augmentation component that recovers structured predictive state from interaction history and conditions the original frozen GUI world model (Figure~\ref{fig:state_aliasing_teaser}, right). To support heterogeneous state types within a single inference-time model, we train family-specific specialists and consolidate their knowledge through multi-teacher distillation.

The resulting predictive-state augmentation substantially improves state-sensitive future prediction while preserving next-observation fidelity, and further improves two WM-assisted agents on AndroidWorld. These results indicate that state aliasing is not only a controlled diagnostic phenomenon: explicitly representing transition-relevant state can improve the practical reliability of world-model-assisted GUI agents.

In summary, our contributions are threefold:
\begin{itemize}
\item We identify \textbf{state aliasing} as a structural limitation of observation-conditioned GUI world modeling, where the current GUI can omit transition-relevant environment state and consequently induce incorrect future predictions.

\item We introduce \textbf{StateAliasBench}, a controlled strict-pair framework that exposes state ambiguities overlooked by conventional trajectory-level evaluation and directly tests whether GUI world models possess sufficient predictive state.

\item We develop a \textbf{lightweight predictive-state augmentation} mechanism for frozen GUI world models, substantially improving state-sensitive prediction and downstream WM-assisted agents without retraining the underlying world model.

\end{itemize}
\section{State Aliasing: Formulation and Diagnosis}
\label{sec:state_aliasing}

\subsection{State Aliasing and Strict-Pair Diagnosis}
\label{sec:formulation}

\textbf{Observation-conditioned GUI world modeling.}
Following the standard partially observable formulation used for GUI agents,
we consider a controlled process with latent environment state
$S_t\in\mathcal S$, action $A_t\in\mathcal A$, and model-visible GUI
observation $O_t\in\mathcal O$:
\begin{equation}
S_{t+1}\sim T(\cdot\mid S_t,A_t),
\qquad
O_t\sim\Omega(\cdot\mid S_t),
\qquad
H_t=(O_0,A_0,\ldots,A_{t-1},O_t),
\label{eq:gui_dynamics}
\end{equation}
where $H_t$ denotes the interaction history available up to time $t$.

For a reachable history $h$ and intervention $a$, define the
\textbf{controlled next-observation law}
\begin{equation}
P_h^a(B)
\triangleq
\Pr\!\left(
O_{t+1}\in B
\mid
H_t=h,\operatorname{do}(A_t=a)
\right),
\label{eq:controlled_future}
\end{equation}
for any measurable $B\subseteq\mathcal O$.
Existing observation-conditioned GUI-WMs approximate this law using only the
current GUI and action,
\begin{equation}
q_\theta(\cdot\mid O_t,A_t)
\approx
P_{H_t}^{A_t}(\cdot).
\label{eq:obs_wm}
\end{equation}
This formulation is sufficient only when the current observation preserves all
information required to distinguish the relevant future transition laws.

\textbf{State aliasing.}
We distinguish \textbf{observation equivalence} from
\textbf{predictive equivalence}:
\begin{equation}
\begin{aligned}
h\sim_{\mathrm{obs}}h'
&\iff
O_t(h)=O_t(h'),\\
h\sim_{\mathrm{pred}}h'
&\iff
P_h^a=P_{h'}^a,
\qquad
\forall a\in\mathcal A_c,
\end{aligned}
\label{eq:two_equivalences}
\end{equation}
where $\mathcal A_c$ denotes the action scope under consideration.
Observation-conditioned prediction implicitly assumes that observationally
equivalent histories are also predictively equivalent. We define
\textbf{state aliasing} as a violation of this condition:
\begin{equation}
h^A\sim_{\mathrm{obs}}h^B,
\qquad
h^A\not\sim_{\mathrm{pred}}h^B.
\label{eq:state_aliasing}
\end{equation}
Thus, two histories may expose the same current GUI while differing in
transition-relevant state, such that at least one shared action induces
different controlled future laws.

An observation-only WM must assign the same conditional prediction to the two
histories although their true controlled future laws differ. Under log loss,
the best shared prediction therefore incurs
\begin{equation}
\min_q
\sum_{b\in\{A,B\}}
\pi_b
D_{\mathrm{KL}}
\!\left(
P_{h^b}^{a}
\Vert
q
\right)
=
\operatorname{JS}_{\boldsymbol{\pi}}
\!\left(
P_{h^A}^{a},
P_{h^B}^{a}
\right)
>
0,
\label{eq:aliasing_js}
\end{equation}
for $\pi_A,\pi_B>0$, $\pi_A+\pi_B=1$, whenever
$P_{h^A}^{a}\neq P_{h^B}^{a}$. Thus, the ambiguity arises from insufficient
conditioning information rather than from a particular prediction
architecture.

\textbf{Strict-pair diagnosis.}
The complete controlled transition laws are not directly observable in real
applications. We therefore use \textbf{strict pairs} as controlled empirical
witnesses of state aliasing. For each state family $f$, let
$\widetilde Z_t^f=\zeta_f(S_t)$ denote the structured transition-relevant state
used by the benchmark, and let
\[
Y_{t+1}^f
=
\psi_f\!\left(\operatorname{Parse}(O_{t+1})\right)
\]
denote the corresponding state-dependent semantic outcome.
Ground-truth values of $\widetilde Z_t^f$ are obtained from read-only
environment evidence and are used only for construction, validation, and
supervision; they are never included in the observation-only WM input.

A \textbf{strict pair} $\mathcal P_i^f$ consists of two real environment
branches $A$ and $B$ satisfying
\begin{equation}
\underbrace{
O_{i,t}^{A}=O_{i,t}^{B},
\qquad
A_{i,t}^{A}=A_{i,t}^{B}=a_i
}_{\text{matched model-visible condition}}
\qquad
\underbrace{
\widetilde Z_{i,t}^{f,A}
\neq
\widetilde Z_{i,t}^{f,B}
}_{\text{different hidden state}}
\qquad
\underbrace{
Y_{i,t+1}^{f,A}
\neq
Y_{i,t+1}^{f,B}
}_{\text{different semantic outcome}}.
\label{eq:strict_pair}
\end{equation}
The matched observation and action remove branch information from the
observation-only prediction input; the state contrast establishes the
transition-relevant distinction under diagnosis; and the semantic contrast
verifies that the shared action exposes this distinction in the realized
environment outcome. Importantly, the final constraint is defined through
family-specific semantic validation rather than arbitrary screenshot
inequality. Strict pairs therefore provide semantically validated witnesses of
observation--state mismatch without requiring estimation of the full
transition kernel.

\subsection{StateAliasBench}
\label{sec:statealiasbench}

\textbf{State families.}
StateAliasBench instantiates the strict-pair protocol through programmatic
interaction with real Android applications. It covers four qualitatively
different forms of transition-relevant state:
\textbf{Persistent Artifact}, \textbf{Clipboard},
\textbf{Ordered Collection}, and \textbf{Temporal Deadline}.
Table~\ref{tab:state_families} summarizes their operational structured states
and semantic outcomes.

\begin{table*}[t]
\centering
\small
\caption{State families in StateAliasBench. $\bot$ denotes absence or an empty
state.}
\label{tab:state_families}
\begin{tabular}{lcll}
\toprule
\textbf{Family}
& \textbf{\# Templates}
& \textbf{Structured state $\widetilde Z_t^f$}
& \textbf{Semantic outcome $Y_{t+1}^f$} \\
\midrule

Persistent Artifact
& 9
& $m_t:\mathcal X\rightarrow\mathcal L\cup\{\bot\}$
& object existence / location \\

Clipboard
& 5
& $c_t\in\mathcal V\cup\{\bot\}$
& pasted payload \\

Ordered Collection
& 5
& $\mathbf q_t=(x_1,\ldots,x_n)\in\mathcal X^\ast$
& effective ordered sequence \\

Temporal Deadline
& 6
& $(\ell_t,\rho_t)$
& timer fire / lifecycle outcome \\

\bottomrule
\end{tabular}
\end{table*}

Detailed family definitions and causal templates are provided in
Appendix~\ref{app:statealiasbench}.

\paragraph{Benchmark construction.}
Each template constructs two histories with distinct $\widetilde Z_t^f$,
matches their complete WM-visible observation and next action at a common
prediction boundary, and independently validates the resulting semantic
outcomes in the real environment. Equality is checked after the evaluated
WM's exact input transformation, while state annotations and targets come
from read-only environment evidence. StateAliasBench contains 25 causal
templates and 800 verified strict pairs; full construction and validation
details are provided in Appendix~\ref{app:statealiasbench}.

\textbf{Generalization settings.}
We additionally evaluate cross-application transfer, long-history stress
($d\in\{4,8,16,32\}$), and semantic interference; full protocols are provided
in Appendix~\ref{app:robustness}.

We also audit train--evaluation separation and find no duplicate examples or
complete trajectories; full criteria and results are provided in
Appendix~\ref{app:statealiasbench}.
\section{Predictive State Recovery}
\label{sec:predictive_state}

We recover transition-relevant predictive state from interaction history to
augment frozen GUI world models.

\begin{figure*}[t]
    \centering
    \includegraphics[width=\textwidth]{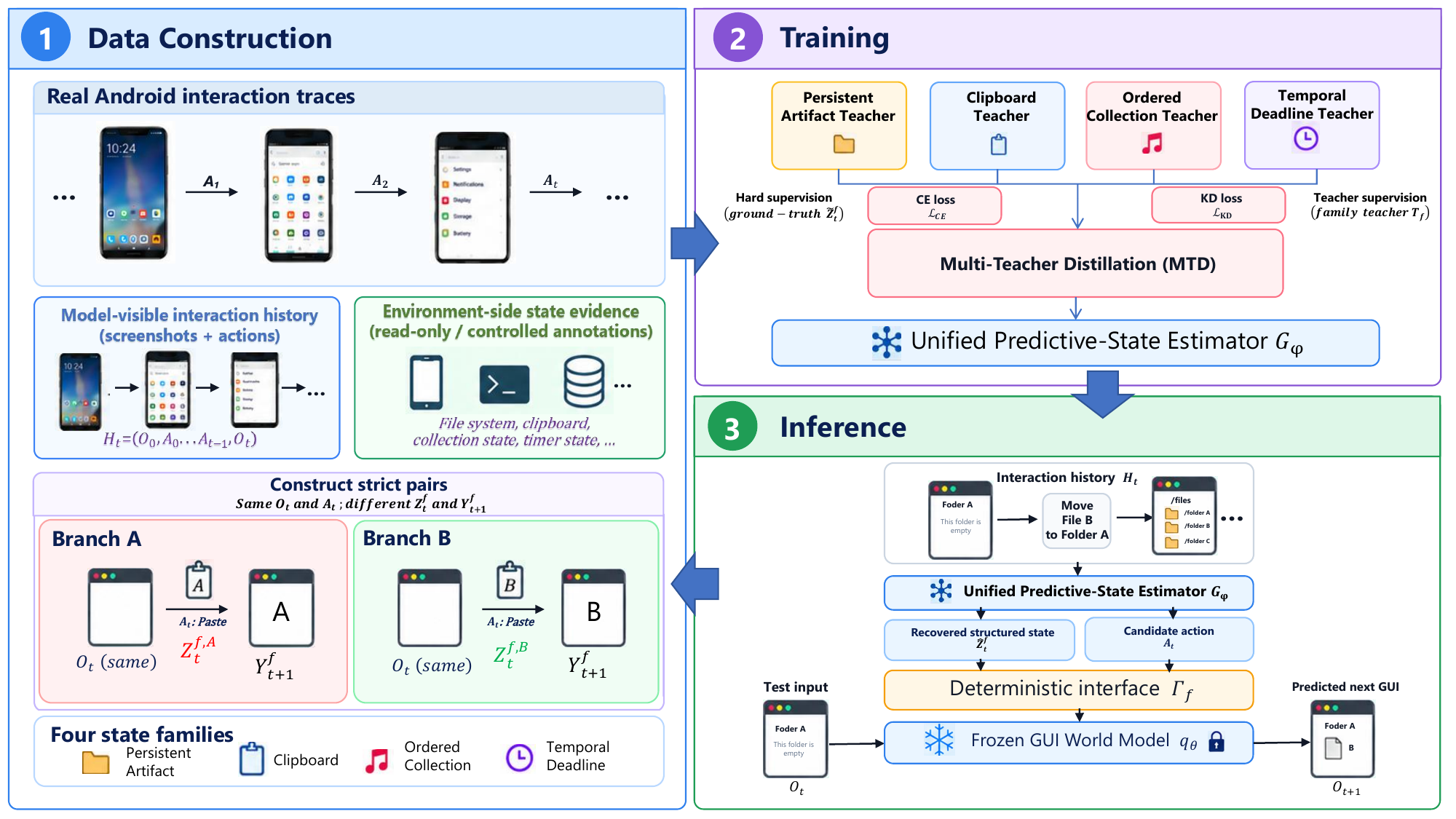}
    \caption{
    \textbf{Overview of our framework.}
    We construct strict pairs from real Android interactions, distill
    family-specific state specialists into a unified estimator $G_\phi$,
    recover $\widehat Z_t^f$ from history, and condition a frozen GUI world
    model through a deterministic interface $\Gamma_f$.
    }
    \label{fig:framework}
\end{figure*}

\subsection{Predictive Information State}
\label{sec:predictive_information_state}

Predictive equivalence from Section~\ref{sec:formulation} need not preserve the
current observation. To characterize the information missing \emph{given the
current GUI}, we define the observation-preserving predictive equivalence
relation
\begin{equation}
h\sim_{+}h'
\iff
O_t(h)=O_t(h')
\;\land\;
P_h^a=P_{h'}^a,
\qquad
\forall a\in\mathcal A_c .
\label{eq:obs_predictive_equivalence}
\end{equation}

The resulting equivalence class defines the canonical predictive representation
\begin{equation}
X_t^\star
=
[H_t]_{+}
\cong
(O_t,Z_t^\star),
\label{eq:canonical_predictive_state}
\end{equation}
where $Z_t^\star$ contains the residual predictive information beyond $O_t$,
distinguishing observationally aliased histories that require different future
transition laws.

\begin{theorem}[Predictive sufficiency and minimality]
\label{thm:predictive_sufficiency}
For every $a\in\mathcal A_c$, the canonical representation $X_t^\star$
preserves the controlled next-observation law:
\begin{equation}
p(O_{t+1}\mid H_t,\operatorname{do}(a))
=
p(O_{t+1}\mid X_t^\star,\operatorname{do}(a)).
\label{eq:predictive_sufficiency}
\end{equation}
Moreover, let $R_t=r(H_t)$ be any deterministic representation such that
$(O_t,R_t)$ is predictively sufficient over the same action scope. Then there
exists a deterministic map $f$ such that
\begin{equation}
X_t^\star
=
f(O_t,R_t).
\label{eq:predictive_minimal_map}
\end{equation}
Whenever the corresponding conditional entropies are finite,
\begin{equation}
H(Z_t^\star\mid O_t)
\le
H(R_t\mid O_t).
\label{eq:predictive_minimality}
\end{equation}
\end{theorem}

Thus, strict-pair branches with different future laws belong to distinct
$\sim_{+}$ equivalence classes; the proof is provided in
Appendix~\ref{app:predictive_state_proofs}.

At the distribution level, the amount of predictive information omitted by
the current GUI is
\begin{equation}
\begin{aligned}
\mathcal G
&\triangleq
I(H_t;O_{t+1}\mid O_t,A_t) \\
&=
H(O_{t+1}\mid O_t,A_t)
-
H(O_{t+1}\mid H_t,A_t).
\end{aligned}
\label{eq:predictive_information_gap}
\end{equation}
Thus, $\mathcal G>0$ is the distribution-level counterpart of
Eq.~\ref{eq:aliasing_js}: the history contains transition-relevant information
not determined by the current observation. Throughout this work,
\emph{predictive sufficiency} refers only to the one-step controlled transition
over $\mathcal A_c$; long-history evaluation tests recovery of more distant
evidence rather than multi-step sufficiency of $X_t^\star$.

\subsection{Learning a Structured Predictive State}
\label{sec:learning_predictive_state}

The canonical residual state $Z_t^\star$ induced by
Eq.~\ref{eq:obs_predictive_equivalence} is not directly available as a
supervision target. StateAliasBench therefore provides a family-specific
operational structured state $\widetilde Z_t^f$.

We do not assume that $\widetilde Z_t^f$ is identical to $Z_t^\star$.
Instead, we quantify the predictive information that remains missing after
supplying the structured state by
\begin{equation}
\delta_f
\triangleq
I(H_t;O_{t+1}
\mid
O_t,\widetilde Z_t^f,A_t).
\label{eq:family_residual_information}
\end{equation}
When $\delta_f=0$, $(O_t,\widetilde Z_t^f)$ is predictively sufficient over
the evaluated scope; otherwise, transition-relevant information remains
missing. This yields the practical approximation chain
\begin{equation}
Z_t^\star
\;\longrightarrow\;
\widetilde Z_t^f
\;\longrightarrow\;
\widehat Z_t^f,
\label{eq:state_approximation_chain}
\end{equation}
where $\widetilde Z_t^f$ is the supervised family-specific approximation and
$\widehat Z_t^f$ its learned recovery.

We train one specialist teacher $T_f$ for each state family and distill their
family-specific knowledge into a unified estimator $G_\phi$. Given interaction
history $H_t$ and a family/schema instruction $q_f$, the estimator predicts
\begin{equation}
\widehat Z_t^f
=
\operatorname{Decode}\,
G_\phi(H_t,q_f).
\label{eq:state_estimator}
\end{equation}
The target is the canonical JSON representation of $\widetilde Z_t^f$.
The unified estimator is used for single-model deployment, while the
family-specific specialists are retained as diagnostic baselines.

For a target JSON sequence $y=(y_1,\ldots,y_L)$, the hard-label objective is
\begin{equation}
\mathcal L_{\mathrm{CE}}
=
-\frac{1}{L}
\sum_{j=1}^{L}
\log
p_\phi(y_j\mid x,y_{<j}),
\label{eq:state_ce}
\end{equation}
where $x=(H_t,q_f)$ and only assistant-side JSON tokens contribute to the
training loss.

For an example belonging to family $f$, the corresponding specialist $T_f$
additionally provides a soft target. Let
$\widetilde p_{T_f}^{\tau,K}(v\mid x,y_{<j})$ denote the temperature-scaled
teacher distribution restricted to and renormalized over its top-$K$ tokens.
The distillation loss is
\begin{equation}
\mathcal L_{\mathrm{KD}}
=
-\frac{\tau^2}{L}
\sum_{j=1}^{L}
\sum_{v\in\mathcal K_j}
\widetilde p_{T_f}^{\tau,K}(v\mid x,y_{<j})
\log
p_\phi^\tau(v\mid x,y_{<j}).
\label{eq:kd_loss}
\end{equation}

The unified estimator combines the exact structured-state target with the
corresponding family specialist's soft distribution:
\begin{equation}
\boxed{
\mathcal L_{\mathrm{MTD}}^{(i)}
=
w_{f_i}
\!\left[
(1-\lambda)\mathcal L_{\mathrm{CE}}^{(i)}
+
\lambda\mathcal L_{\mathrm{KD}}^{(i)}
\right].
}
\label{eq:mtd_loss}
\end{equation}
We use $\lambda=0.25$, $\tau=2$, and $K=64$, with family-balancing weights
$w_f$ defined in Appendix~\ref{app:state_training_data}. Each example receives
soft supervision only from its corresponding family specialist. Additional
optimization and implementation details are provided in
Appendix~\ref{app:state_recovery}.

\subsection{State-Augmented GUI World Modeling}
\label{sec:state_augmented_wm}

Because the evaluated WMs are frozen and were not trained to consume our
structured state schema directly, the recovered state must be expressed in a
form compatible with their native prediction interfaces. We therefore use a
fixed deterministic state-conditioning interface. For family $f$, we write
\begin{equation}
C_t^f
=
\Gamma_f(\widehat Z_t^f,A_t),
\label{eq:state_interface}
\end{equation}
where $\Gamma_f$ denotes the \emph{complete} model-compatible realization of
the recovered state. It resolves the action-relevant consequence implied by
the recovered current state and candidate action, and expresses this
information through family- and WM-compatible surface constraints and
serialization. The interface introduces no additional environment
information: all branch-dependent content is deterministically derived from
the recovered state and prediction-time inputs.

For compactness, Eq.~\ref{eq:state_interface} suppresses fixed task context
and WM-specific serialization. The implemented interface is
\begin{equation}
C_t^{f,m}
=
\Gamma_{f,m}(\widehat Z_t^f,A_t,c_t),
\label{eq:implemented_state_interface}
\end{equation}
where $m$ identifies the WM-specific realization and $c_t$ denotes
prediction-time task/entity context available independently of the hidden
state value. Appendix~\ref{app:state_wm_interface} specifies the complete
interface, distinguishes its action-resolved consequence from its
model-compatible realization, and evaluates progressively richer
realizations of the same recovered state.

The implemented state-augmented WM is therefore
\begin{equation}
H_t
\xrightarrow{G_\phi}
\widehat Z_t^f
\xrightarrow{\Gamma_f(\cdot,A_t)}
C_t^f,
\qquad
q_\theta(O_{t+1}\mid O_t,A_t,C_t^f).
\label{eq:state_augmented_wm}
\end{equation}
Oracle State and Predictive State use the same complete deterministic
interface and differ only in the source of the supplied state value.

\paragraph{Predictive-risk decomposition.}
To separate information missing from the supplied representation from the
ability of a frozen WM to exploit that representation, let
$U_t=u(H_t,A_t)$ denote any deterministic conditioning variable supplied in
addition to $(O_t,A_t)$. We define
\begin{equation}
\mathcal L(\theta,U)
=
\mathbb E
\left[
-\log q_\theta(O_{t+1}\mid O_t,A_t,U_t)
\right],
\qquad
\mathcal L_H^\star
=
H(O_{t+1}\mid H_t,A_t).
\label{eq:conditioning_risk}
\end{equation}

\begin{theorem}[Predictive-risk decomposition]
\label{thm:risk_decomposition}
For any deterministic conditioning variable $U_t=u(H_t,A_t)$,
\begin{align}
\mathcal L(\theta,U)-\mathcal L_H^\star
&=
\underbrace{
I(H_t;O_{t+1}\mid O_t,U_t,A_t)
}_{\delta_{\mathrm{rep}}(U)}
\nonumber\\[-1mm]
&\quad+
\underbrace{
\mathbb E
D_{\mathrm{KL}}
\!\left(
p(O_{t+1}\mid O_t,U_t,A_t)
\Vert
q_\theta(O_{t+1}\mid O_t,U_t,A_t)
\right)
}_{\epsilon_{\mathrm{WM}}(U)} .
\label{eq:risk_decomposition}
\end{align}
\end{theorem}

Here, $\delta_{\mathrm{rep}}(U)$ measures transition-relevant information still
missing after conditioning, while $\epsilon_{\mathrm{WM}}(U)$ measures the
frozen WM's residual error given that conditioning. Three limiting cases are
\begin{equation}
\delta_{\mathrm{rep}}(\varnothing)
=
\mathcal G,
\qquad
\delta_{\mathrm{rep}}(H_t)
=
0,
\qquad
\delta_{\mathrm{rep}}(Z_t^\star)
=
0.
\label{eq:rep_error_extremes}
\end{equation}
Thus, observation-only conditioning incurs the full information gap
$\mathcal G$, whereas complete history and the canonical predictive state
eliminate the representation deficit. The same decomposition applies to the
conditioning variable produced by the implemented deterministic state
interface. The proof is provided in
Appendix~\ref{app:proof_risk_decomposition}.

\begin{corollary}[Improvement over observation-only prediction]
\label{cor:strict_improvement}
If a conditioning representation $U_t$ satisfies
\begin{equation}
\delta_{\mathrm{rep}}(U)
+
\epsilon_{\mathrm{WM}}(U)
<
\mathcal G,
\label{eq:improvement_condition}
\end{equation}
then
\begin{equation}
\mathcal L(\theta,U)
<
H(O_{t+1}\mid O_t,A_t).
\label{eq:strict_improvement}
\end{equation}
\end{corollary}

Hence predictive-state augmentation improves over observation-only prediction
when the supplied conditioning removes enough missing predictive information
and the frozen WM can exploit it with sufficiently small residual error.
Both Oracle State and Predictive State use the same complete deterministic
interface; the evaluated GUI-WMs remain frozen.
\section{Experiments}
\label{sec:experiments}

\subsection{Experimental Setup}
\label{sec:experimental_setup}

\noindent\textbf{Evaluation overview.}
We evaluate predictive state along three complementary axes:
whether state aliasing limits state-sensitive transition prediction,
whether recovered state remains robust under distribution shift, and whether
its benefit transfers to downstream GUI-agent performance.

\paragraph{Models and evaluation settings.}
We train the unified \textbf{Predictive State (Ours)} estimator from
Qwen3.5-2B \citep{qwen35} using the multi-teacher distillation objective in
Section~\ref{sec:learning_predictive_state}, and evaluate it on two frozen GUI
world models: Code2World-8B \citep{zheng2026code2world} and gWorld-8B
\citep{koh2026gworld}. On AndroidWorld \citep{androidworld}, Code2World-8B is
paired with GPT-5.4 and Qwen3-VL-8B-Instruct (Qwen3-VL-8B)
\citep{qwen3-vl}. We additionally evaluate cross-application transfer,
long-history stress, and semantic interference.

\paragraph{Conditioning protocols.}
For each StateAliasBench strict pair, the two branches are matched over the
complete model-visible prediction condition. We compare
\textbf{Observation}, using $(O_t,A_t)$;
\textbf{Raw History}, directly supplying the complete, untruncated interaction
history to the WM without state recovery;
\textbf{Oracle State}, using the annotated structured state; and
\textbf{Predictive State (Ours)}, using the recovered structured state.
Oracle and recovered states are processed through the same complete
deterministic state-conditioning interface described in
Section~\ref{sec:state_augmented_wm}. Oracle State relies on ground-truth
transition-relevant state unavailable at inference time and therefore serves
as a diagnostic upper bound.

\paragraph{Metrics.}
\textbf{BranchExact} measures whether an individual generated next observation
realizes the correct family-specific semantic outcome:
\begin{equation}
\operatorname{BC}_i^b
=
\mathbbm{1}
\!\left[
\widehat Y_{i,t+1}^{f,b}
=
Y_{i,t+1}^{f,b}
\right],
\qquad
\operatorname{PairCorrect}_i
=
\operatorname{BC}_i^A
\operatorname{BC}_i^B.
\label{eq:paircorrect}
\end{equation}
\textbf{PairCorrect} requires both branches of a strict pair to be correct and
is our primary state-aliasing diagnostic. It is deliberately stricter than
branch-level accuracy: predicting one plausible future cannot resolve an
aliased pair unless both state-dependent futures are distinguished correctly.
State-recovery accuracy is measured by exact canonical matching between
$\widehat Z_t^f$ and $\widetilde Z_t^f$. We additionally report DINOv2
\citep{oquab2024dinov2} and SigLIP \citep{zhai2023siglip} similarity for
next-observation fidelity, and task success rate for AndroidWorld.
Semantic scorer specifications and image-level fidelity implementation details
are provided in Appendices~\ref{app:semantic_validation}
and~\ref{app:visual_fidelity}, respectively.

\subsection{Main Results}
\label{sec:main_results}

\subsubsection{StateAliasBench}
\label{sec:statealias_results}

\paragraph{Observation-only prediction cannot resolve aliased futures.}
As shown in Table~\ref{tab:conditioning_full}, Observation obtains
\textbf{0\% PairCorrect} in every state family for both WMs despite
occasionally predicting an individual branch correctly. Raw History recovers some missing information, most notably for Clipboard,
but remains near zero on several other families when supplied directly to the
evaluated frozen WMs.

\paragraph{Recovered predictive state resolves aliased futures.}
The unified MTD estimator restores state-sensitive prediction across all four
families, reaching 79.38\% macro PairCorrect on Code2World-8B and 80.63\% on
gWorld-8B, compared with 0\% under Observation-only conditioning. It also
achieves 88.63\% exact branch-level structured-state recovery; family-specific
specialists provide a stronger learned-state reference.

\begin{table*}[t]
\centering
\scriptsize
\setlength{\tabcolsep}{4pt}
\renewcommand{\arraystretch}{0.92}

\caption{
StateAliasBench performance under different conditioning.
Oracle State is a diagnostic upper bound and excluded from ranking.
Best non-oracle results are in \textbf{bold}; second-best results are underlined.
}
\label{tab:conditioning_full}

\begin{tabular}{@{}lcccc@{\hspace{6pt}}cccc@{}}
\toprule
&
\multicolumn{4}{c}{\textbf{Code2World-8B}~\citep{zheng2026code2world}}
&
\multicolumn{4}{c}{\textbf{gWorld-8B}~\citep{koh2026gworld}}
\\
\cmidrule(lr){2-5}
\cmidrule(lr){6-9}

\textbf{Condition}
& \textbf{Pers.}
& \textbf{Clip.}
& \textbf{Ord.}
& \textbf{Temp.}
& \textbf{Pers.}
& \textbf{Clip.}
& \textbf{Ord.}
& \textbf{Temp.}
\\
\midrule

\multicolumn{9}{c}{\textit{BranchExact (\%)}} \\[1pt]
\cmidrule(lr){1-9}

Observation
& 29.00
& 0.00
& 0.00
& 0.00
& 42.00
& 0.00
& 0.00
& 15.75
\\

Raw History
& 12.00
& 58.75
& 5.25
& 0.00
& 45.75
& 59.50
& 10.50
& 1.00
\\

Oracle State
& 92.50
& 100.00
& 90.75
& 97.50
& 96.50
& 99.00
& 90.50
& 100.00
\\

Specialist State
& \textbf{88.25}
& \textbf{94.75}
& \textbf{83.75}
& \textbf{90.00}
& \textbf{92.00}
& \textbf{93.75}
& \textbf{83.25}
& \textbf{92.50}
\\

Predictive State (MTD, Ours)
& \underline{86.75}
& \underline{89.00}
& \underline{81.50}
& \underline{85.75}
& \underline{88.50}
& \underline{88.50}
& \underline{79.50}
& \underline{88.50}
\\

\midrule

\multicolumn{9}{c}{\textit{PairCorrect (\%)}} \\[1pt]
\cmidrule(lr){1-9}

Observation
& 0.00
& 0.00
& 0.00
& 0.00
& 0.00
& 0.00
& 0.00
& 0.00
\\

Raw History
& 2.00
& 51.00
& 0.00
& 0.00
& 14.00
& 46.00
& 1.50
& 0.00
\\

Oracle State
& 85.00
& 100.00
& 84.00
& 95.00
& 93.00
& 98.00
& 84.50
& 100.00
\\

Specialist State
& \underline{78.50}
& \textbf{89.50}
& \textbf{73.50}
& \textbf{84.00}
& \textbf{85.50}
& \textbf{87.50}
& \textbf{72.50}
& \textbf{88.50}
\\

Predictive State (MTD, Ours)
& \textbf{80.00}
& \underline{85.00}
& \underline{71.50}
& \underline{81.00}
& \underline{84.00}
& \underline{84.00}
& \underline{68.00}
& \underline{86.50}
\\

\bottomrule
\end{tabular}

\end{table*}

\paragraph{State augmentation also improves next-observation fidelity.}
As shown in Table~\ref{tab:fidelity_main}, \textbf{Predictive State (Ours)}
improves both DINOv2 and SigLIP similarity for Code2World-8B and gWorld-8B,
indicating that improved state-sensitive correctness does not come at the
expense of next-observation fidelity.

\begin{table}[t]
\centering
\scriptsize
\setlength{\tabcolsep}{3pt}
\renewcommand{\arraystretch}{0.95}

\caption{
Next-observation fidelity. Best results are in \textbf{bold};
second-best results are underlined.
}
\label{tab:fidelity_main}

\begin{tabular}{lcccccc}
\toprule
&
\multicolumn{3}{c}{\textbf{Code2World-8B}} &
\multicolumn{3}{c}{\textbf{gWorld-8B}} \\
\cmidrule(lr){2-4}
\cmidrule(lr){5-7}

\textbf{Metric}
& \textbf{Obs.}
& \textbf{Ours}
& $\boldsymbol{\Delta}$
& \textbf{Obs.}
& \textbf{Ours}
& $\boldsymbol{\Delta}$ \\
\midrule

DINOv2
& \underline{0.4718}
& \textbf{0.6218}
& +0.1499
& \underline{0.5534}
& \textbf{0.6685}
& +0.1151 \\

SigLIP
& \underline{0.7234}
& \textbf{0.8552}
& +0.1318
& \underline{0.8021}
& \textbf{0.8285}
& +0.0265 \\

\bottomrule
\end{tabular}
\end{table}

\subsubsection{Downstream GUI Agents}
\label{sec:downstream}

AndroidWorld is a general-purpose GUI-agent benchmark and is not designed to
isolate state aliasing as an evaluation dimension. Its task distribution
therefore only partially overlaps with the hidden-state transitions targeted
by StateAliasBench, so benefits specific to predictive state may be diluted
when aggregated over the full benchmark. We therefore additionally report
\textbf{Relative Gain}, defined as the additional task-success improvement
from predictive-state conditioning normalized by the gain obtained from
adding the WM alone.

As shown in Table~\ref{tab:androidworld_main}, predictive-state conditioning
increases task success from 58.62\% to 60.34\% ($+1.72$ pp) for GPT-5.4 and
from 44.83\% to 47.41\% ($+2.58$ pp) for Qwen3-VL-8B. These correspond to
Relative Gains of 66.4\% and 74.8\%, respectively. For Qwen3-VL-8B, the paired
task-level comparison contains three state-conditioned rescues and no
regressions; the complete downstream protocol and paired analysis are
provided in Appendix~\ref{app:androidworld}.

\begin{table}[t]
\centering
\scriptsize
\setlength{\tabcolsep}{4pt}
\renewcommand{\arraystretch}{0.95}

\caption{
AndroidWorld task success (\%). Both WM-assisted conditions use Code2World-8B.
Best results are in \textbf{bold}; second-best results are underlined.
}
\label{tab:androidworld_main}

\begin{tabular}{lcccc}
\toprule
\textbf{Agent}
& \textbf{Base}
& \textbf{+WM}
& \textbf{+WM+Ours}
& \textbf{Rel. Gain} \\
\midrule

GPT-5.4
& 56.03
& \underline{58.62}
& \textbf{60.34}
& +66.4\% \\

Qwen3-VL-8B
& 41.38
& \underline{44.83}
& \textbf{47.41}
& +74.8\% \\

\bottomrule
\end{tabular}
\end{table}

\subsection{State Recovery and Robustness}
\label{sec:additional_results}

\paragraph{Unified state recovery.}
Multi-teacher distillation consolidates heterogeneous state families into a
single estimator while reducing the cross-family imbalance observed under
direct joint training. At the structured-state level, compared with
family-balanced Joint CE, MTD raises the worst-family PairCorrect from
\textbf{69.5\%} to \textbf{82.0\%} and the macro average from
\textbf{83.63\%} to \textbf{84.88\%}. Independently deployed specialists
remain stronger on average at \textbf{87.88\%}, reflecting the trade-off
between family-specific specialization and a single unified inference-time
model.

We further test whether this difference propagates to final world-model
prediction. Under a matched Code2World-8B evaluation using the same complete
state-conditioning interface, MTD improves final-WM macro PairCorrect from
\textbf{76.75\%} to \textbf{79.38\%} and worst-family PairCorrect from
\textbf{69.50\%} to \textbf{71.50\%} over Joint CE. The improvement is not
uniform across families: MTD substantially improves Clipboard and modestly
improves Persistent Artifact and Temporal Deadline, while Joint CE remains
stronger on Ordered Collection. Detailed recovery and end-to-end comparisons
are provided in Appendix~\ref{app:mtd_ablation}.

The state-conditioned Code2World-8B request adds only 5.7\% marginal
prompt-token overhead, and the unified estimator requires 0.768\,s mean
generation-and-decoding latency per transition. Computational details are
provided in Appendix~\ref{app:state_efficiency}.

\paragraph{Robustness under distribution shift.}
We further evaluate cross-application transfer, controlled long-history stress
($d$-Stress), and semantic interference. The three distribution-shift suites
are evaluated using the same single-trajectory Exact metric and can therefore
be compared directly with one another. As shown in
Table~\ref{tab:recovery_robustness}, cross-application transfer is the most
challenging of these shifts for the unified MTD estimator, with 79.75\%
average Exact, compared with 85.25\% under long-history stress and 84.88\%
under semantic interference. The lowest cross-application Exact scores occur
for Persistent Artifact and Temporal Deadline, while Clipboard remains
comparatively stable. These results suggest that, among the distribution
shifts evaluated here, changing application context is a stronger source of
difficulty than increased history depth or semantically related interference.

\begin{table*}[!t]
\centering
\scriptsize
\setlength{\tabcolsep}{4.3pt}
\renewcommand{\arraystretch}{0.88}

\caption{
State-recovery robustness. Strict reports PairCorrect on 200 A/B pairs per
family; distribution shifts report Exact on 200 trajectories per family and
are not directly comparable to Strict. Best results are in \textbf{bold};
second-best results are underlined.
}
\label{tab:recovery_robustness}

\begin{tabular*}{\textwidth}{
@{\extracolsep{\fill}}
lcccccccc
@{}
}
\toprule
&
\multicolumn{4}{c}{\textbf{Specialists}} &
\multicolumn{4}{c}{\textbf{MTD (Ours)}} \\
\cmidrule(lr){2-5}
\cmidrule(lr){6-9}

\textbf{Family}
& \textbf{Strict}
& \textbf{Cross-app}
& \textbf{$d$-Stress}
& \textbf{Interf.}
& \textbf{Strict}
& \textbf{Cross-app}
& \textbf{$d$-Stress}
& \textbf{Interf.} \\
\midrule

Persistent
& \underline{85.5}
& \textbf{87.0}
& \textbf{87.0}
& \textbf{87.0}
& \textbf{86.0}
& \underline{75.5}
& \underline{82.5}
& \underline{82.5} \\

Clipboard
& \textbf{89.5}
& \underline{85.5}
& \textbf{89.5}
& \textbf{89.5}
& \underline{85.0}
& \textbf{88.5}
& \underline{85.0}
& \underline{88.5} \\

Ordered
& \textbf{88.0}
& \textbf{89.0}
& \textbf{87.0}
& \textbf{86.0}
& \underline{82.0}
& \underline{83.0}
& \underline{85.5}
& \underline{79.0} \\

Temporal
& \textbf{88.5}
& \textbf{83.5}
& \underline{85.0}
& \underline{87.5}
& \underline{86.5}
& \underline{72.0}
& \textbf{88.0}
& \textbf{89.5} \\

\midrule

\textbf{Avg.}
& \textbf{87.88}
& \textbf{86.25}
& \textbf{87.13}
& \textbf{87.50}
& \underline{84.88}
& \underline{79.75}
& \underline{85.25}
& \underline{84.88} \\

\bottomrule
\end{tabular*}
\end{table*}
\section{Conclusion}

We identify \emph{state aliasing} as an input-sufficiency failure in GUI world
models: identical visible conditions can require different futures when
transition-relevant state is hidden. StateAliasBench diagnoses this failure
through strict pairs, while predictive-state recovery infers structured state
from history and improves frozen GUI-WMs across state-sensitive and downstream
evaluations. Our results suggest that reliable GUI world modeling should
account not only for what is visible, but also for the state that determines
what happens next.
\subsection*{AI Use Statement}

Generative AI tools were used during manuscript preparation for language
polishing and phrasing assistance. All AI-assisted edits were reviewed and
revised by the authors, who take full responsibility for the final manuscript
and its scientific content. Generative AI tools were not treated as
authoritative sources for experimental results or citations.

\subsection*{Ethics Statement}

This work studies state aliasing and predictive-state recovery in GUI world
models using controlled Android interaction environments, model-generated
predictions, and benchmark-derived state annotations. The experiments do not
involve human-subject studies or deployment with end users. GUI agents and
world models may nevertheless introduce risks when used for automated
interaction with real systems, particularly when incorrect state assumptions
lead to unintended actions. Our evaluation focuses on diagnosing such
state-dependent failures under controlled conditions and on improving the
reliability of world-model predictions without modifying the underlying world
models. Any released artifacts will follow the licenses and usage requirements
of the underlying datasets, applications, and models.

\subsection*{Reproducibility Statement}

We provide the formal definition of state aliasing and the strict-pair
diagnostic protocol in Section~2, and the predictive-state formulation,
structured state-recovery method, multi-teacher distillation objective, and
state-conditioned world-model interface in Section~3. Section~4 specifies the
evaluated world models, conditioning protocols, metrics, robustness settings,
and downstream AndroidWorld evaluation.

Additional theoretical proofs are provided in Appendix~B.
Appendix~C details StateAliasBench construction, causal templates,
ground-truth state provenance, semantic validation, model-visible pair
matching, information access, and repeated-execution stability.
Appendix~D specifies the predictive-state training data, state schemas,
specialist teachers, multi-teacher distillation procedure, optimization
settings, ablations, and computational cost.
Appendix~E documents the deterministic state-conditioning interface and its
prediction-time information provenance.
Appendix~F provides the cross-application, long-history, semantic-interference,
and next-observation-fidelity evaluation protocols.
Appendix~G describes the downstream AndroidWorld protocol and a representative
paired rescue trajectory. These details are intended to support reproduction
of the benchmark, predictive-state recovery, world-model conditioning, and
downstream evaluations.

\bibliography{iclr2027_conference}

@String(ICCV= {IEEE/CVF Int. Conf. Comput. Vis.})

@String(ICLR= {Int. Conf. Learn. Represent.})

@String(ICML= {Int. Conf. Mach. Learn.})

@String(NIPS= {Adv. Neural Inform. Process. Syst.})

@String(AAAI= {AAAI Conf. Artif. Intell.})

@String(ACL = {Annu. Meeting Assoc. Comput. Linguist.})

@String(ICCV  = {ICCV})

@String(NIPS  = {NeurIPS})

@String(ICLR  = {ICLR})

@String(ICML  = {ICML})

@String(AAAI = {AAAI})

@String(ACL = {ACL})

@article{qwen3-vl,
  title={Qwen3-VL Technical Report},
  author={{Qwen Team}},
  journal={arXiv preprint arXiv:2511.21631},
  year={2025}
}

@misc{qwen35,
  title  = {{Qwen3.5}: Towards Native Multimodal Agents},
  author = {{Qwen Team}},
  month  = feb,
  year   = {2026},
  url    = {https://qwen.ai/blog?id=qwen3.5}
}

@article{qwen3.8,
  title   = {Qwen3.8-Omni: Towards Native Omni-Modal Agents},
  author  = {{Qwen Team}},
  journal = {arXiv preprint arXiv:2609.25611},
  year    = {2026}
}

@misc{opus5.5,
  title        = {Claude Opus 5.5},
  author       = {{Anthropic}},
  year         = {2026},
  month        = sep,
  howpublished = {Anthropic},
  note         = {Accessed: 2026-09-26}
}

@misc{gpt-6,
  title        = {{GPT-6 Astra}: A New Generation of Intelligence},
  author       = {{OpenAI}},
  year         = {2026},
  month        = sep,
  howpublished = {OpenAI},
  note         = {Accessed: 2026-09-26}
}

@article{kimi-k3,
  title   = {Kimi K3: Open Frontier Intelligence},
  author  = {{Kimi Team}},
  journal = {arXiv preprint arXiv:2607.24653},
  year    = {2026}
}

@article{halluclear,
  title={HalluClear: Diagnosing, Evaluating and Mitigating Hallucinations in GUI Agents},
  author={Jin, Chao and Yang, Wenkui and Sun, Hao and Liao, Yuqi and Jiang, Qianyi and Zhou, Kai and Cao, Jie and He, Ran and Huang, Huaibo},
  journal={arXiv preprint arXiv:2604.17284},
  year={2026}
}

@article{pomdp,
  title={Planning and acting in partially observable stochastic domains},
  author={Kaelbling, Leslie Pack and Littman, Michael L and Cassandra, Anthony R},
  journal={Artificial intelligence},
  volume={101},
  number={1-2},
  pages={99--134},
  year={1998},
  publisher={Elsevier}
}

@article{whitehead1991learning,
  title   = {Learning to Perceive and Act by Trial and Error},
  author  = {Whitehead, Steven D. and Ballard, Dana H.},
  journal = {Machine Learning},
  volume  = {7},
  pages   = {45--83},
  year    = {1991},
  doi     = {10.1023/A:1022619109594}
}

@inproceedings{chrisman1992perceptual,
  title     = {Reinforcement Learning with Perceptual Aliasing:
               The Perceptual Distinctions Approach},
  author    = {Chrisman, Lonnie},
  booktitle = AAAI,
  pages     = {183--188},
  year      = {1992}
}

@inproceedings{littman2001predictive,
  title     = {Predictive Representations of State},
  author    = {Littman, Michael L. and Sutton, Richard S. and Singh, Satinder},
  booktitle = NIPS,
  volume    = {14},
  year      = {2001}
}

@inproceedings{singh2004predictive,
  title     = {Predictive State Representations:
               A New Theory for Modeling Dynamical Systems},
  author    = {Singh, Satinder and James, Michael R. and Rudary, Matthew R.},
  booktitle = {Proceedings of the Twentieth Conference on Uncertainty in Artificial Intelligence},
  pages     = {512--518},
  year      = {2004}
}

@article{shalizi2001computational,
  title   = {Computational Mechanics:
             Pattern and Prediction, Structure and Simplicity},
  author  = {Shalizi, Cosma Rohilla and Crutchfield, James P.},
  journal = {Journal of Statistical Physics},
  volume  = {104},
  pages   = {817--879},
  year    = {2001},
  doi     = {10.1023/A:1010388907793}
}

@article{subramanian2022approximate,
  title   = {Approximate Information State for Approximate Planning
             and Reinforcement Learning in Partially Observed Systems},
  author  = {Subramanian, Jayakumar and Sinha, Amit and Seraj, Raihan and Mahajan, Aditya},
  journal = {Journal of Machine Learning Research},
  volume  = {23},
  number  = {12},
  pages   = {1--83},
  year    = {2022},
  url     = {https://jmlr.org/papers/v23/20-1165.html}
}

@article{oquab2024dinov2,
  title   = {{DINOv2}: Learning Robust Visual Features without Supervision},
  author  = {Oquab, Maxime and Darcet, Timoth{\'e}e
             and Moutakanni, Th{\'e}o and Vo, Huy V.
             and Szafraniec, Marc and Khalidov, Vasil
             and Fernandez, Pierre and Haziza, Daniel
             and Massa, Francisco and El-Nouby, Alaaeldin
             and Assran, Mahmoud and Ballas, Nicolas
             and Galuba, Wojciech and Howes, Russell
             and Huang, Po-Yao and Li, Shang-Wen
             and Misra, Ishan and Rabbat, Michael
             and Sharma, Vasu and Synnaeve, Gabriel
             and Xu, Hu and J{\'e}gou, Herv{\'e}
             and Mairal, Julien and Labatut, Patrick
             and Joulin, Armand and Bojanowski, Piotr},
  journal = {Transactions on Machine Learning Research},
  year    = {2024},
  url     = {https://openreview.net/forum?id=a68SUt6zFt}
}

@inproceedings{zhai2023siglip,
  title     = {Sigmoid Loss for Language Image Pre-Training},
  author    = {Zhai, Xiaohua and Mustafa, Basil
               and Kolesnikov, Alexander and Beyer, Lucas},
  booktitle = ICCV,
  pages     = {11975--11986},
  year      = {2023}
}

@inproceedings{luo2025vimo,
  title     = {{ViMo}: A Generative Visual {GUI} World Model for App Agents},
  author    = {Luo, Dezhao and Tang, Bohan and Li, Kang
               and Papoudakis, Georgios and Song, Jifei
               and Gong, Shaogang and Hao, Jianye
               and Wang, Jun and Shao, Kun},
  booktitle = ICLR,
  year      = {2026}
}

@inproceedings{mei2026rwom,
  title     = {{R-WoM}: Retrieval-Augmented World Model for Computer-Use Agents},
  author    = {Mei, Kai and Guo, Jiang and Chang, Shuaichen
               and Dong, Mingwen and Lee, Dongkyu and Niu, Xing
               and Jiang, Jiarong},
  booktitle = ICLR,
  year      = {2026}
}

@inproceedings{koh2026gworld,
  title     = {Generative Visual Code Mobile World Models},
  author    = {Koh, Woosung and Han, Sungjun and Lee, Segyu
               and Yun, Se-Young and Shin, Jamin},
  booktitle = ICML,
  year      = {2026}
}

@article{li2025mobileworldbench,
  title   = {{MobileWorldBench}: Towards Semantic World Modeling for Mobile Agents},
  author  = {Li, Shufan and Kallidromitis, Konstantinos and Gokul, Akash
             and Kato, Yusuke and Kozuka, Kazuki and Grover, Aditya},
  journal = {arXiv preprint arXiv:2512.14014},
  year    = {2025}
}

@article{xiang2025uisim,
  title   = {{UISim}: An Interactive Image-Based {UI} Simulator
             for Dynamic Mobile Environments},
  author  = {Xiang, Jiannan and Zhu, Yun and Shu, Lei
             and Wang, Maria and Yu, Lijun and Barcik, Gabriel
             and Lyon, James and Sunkara, Srinivas and Chen, Jindong},
  journal = {arXiv preprint arXiv:2509.21733},
  year    = {2025}
}

@article{cao2026mobiledreamer,
  title   = {{MobileDreamer}: Generative Sketch World Model for {GUI} Agent},
  author  = {Yilin Cao and Yufeng Zhong and Zhixiong Zeng and Siran Dai and Liming Zheng and Jing Huang and Haibo Qiu and Peng Shi and Wenji Mao},
  journal = {arXiv preprint arXiv:2601.04035},
  year    = {2026}
}

@article{zheng2026code2world,
  title   = {{Code2World}: A {GUI} World Model via Renderable Code Generation},
  author  = {Zheng, Yuhao and Zhong, Li'an and Wang, Yi
             and Dai, Rui and Liu, Kaikui and Chu, Xiangxiang
             and Lv, Linyuan and Torr, Philip and Lin, Kevin Qinghong},
  journal = {arXiv preprint arXiv:2602.09856},
  year    = {2026}
}

@article{xu2026howmobileworldmodel,
  title   = {How Mobile World Model Guides {GUI} Agents?},
  author  = {Xu, Weikai and Huang, Kun and Feng, Yunren
             and Li, Jiaxing and Chen, Yuhan and Liu, Yuxuan
             and Jiang, Zhizheng and Qu, Heng and Gao, Pengzhi
             and Liu, Wei and Luan, Jian and Hu, Xiaolin and An, Bo},
  journal = {arXiv preprint arXiv:2605.10347},
  year    = {2026}
}

@article{xu2026appdeltaworld,
  title   = {{AppDeltaWorld}: Transition-Grounded Delta Code World Model
             for Mobile {GUI} Agents},
  author  = {Xu, Weikai and Feng, Yunren and Lei, Haoxiang
             and Huang, Kun and Liu, Yuxuan and Zhao, Kang
             and Hu, Xiaolin and Shang, Shuo and An, Bo},
  journal = {arXiv preprint arXiv:2608.05891},
  year    = {2026}
}

@inproceedings{liu2025guirise,
  title     = {{GUI-Rise}: Structured Reasoning and History Summarization
               for {GUI} Navigation},
  author    = {Liu, Tao and Wang, Chongyu and Li, Rongjie
               and Yu, Yingchen and He, Xuming and Bai, Song},
  booktitle = {Advances in Neural Information Processing Systems},
  volume    = {38},
  year      = {2025},
  doi       = {10.52202/085713-5505}
}

@article{liu2025palui,
  title   = {{PAL-UI}: Planning with Active Look-back for Vision-Based {GUI} Agents},
  author  = {Liu, Zikang and Li, Junyi and Zhao, Wayne Xin
             and Gao, Dawei and Li, Yaliang and Wen, Ji-Rong},
  journal = {arXiv preprint arXiv:2510.00413},
  year    = {2025},
  url     = {https://arxiv.org/abs/2510.00413}
}

@article{gao2025chainofmemory,
  title   = {Chain-of-Memory: Enhancing {GUI} Agents
             for Cross-Application Navigation},
  author  = {Gao, Xinzge and Hu, Chuanrui and Chen, Bin and Li, Teng},
  journal = {arXiv preprint arXiv:2506.18158},
  year    = {2025},
  url     = {https://arxiv.org/abs/2506.18158}
}

@inproceedings{osworld,
  title={Osworld: Benchmarking multimodal agents for open-ended tasks in real computer environments},
  author={Xie, Tianbao and Zhang, Danyang and Chen, Jixuan and Li, Xiaochuan and Zhao, Siheng and Cao, Ruisheng and Hua, Toh J and Cheng, Zhoujun and Shin, Dongchan and Lei, Fangyu and others},
  booktitle=NIPS,
  volume={37},
  pages={52040--52094},
  year={2024}
}

@article{osworld-2,
  title={OSWorld2.0: Benchmarking Computer Use Agents on Long-Horizon Real-World Tasks},
  author={Yuan, Mengqi and Zhou, Zilong and Xiong, Xinzhuang and Wu, Weiming and Sun, Jiayang and Song, Jiamin and Cui, Kaiqian and Wang, Bowen and Wu, Haoyuan and Li, Yitong and others},
  journal={arXiv preprint arXiv:2606.29537},
  year={2026}
}

@inproceedings{androidworld,
  title={Androidworld: A dynamic benchmarking environment for autonomous agents},
  author={Rawles, Christopher and Clinckemaillie, Sarah and Chang, Yifan and Waltz, Jonathan and Lau, Gabrielle and Fair, Marybeth and Li, Alice and Bishop, William and Li, Wei and Campbell-Ajala, Folawiyo and others},
  booktitle=ICLR,
  year={2025}
}

@inproceedings{mobileworld,
  title={Mobileworld: Benchmarking autonomous mobile agents in agent-user interactive and mcp-augmented environments},
  author={Kong, Quyu and Zhang, Xu and Yang, Zhenyu and Gao, Nolan and Liu, Chen and Tong, Panrong and Cai, Chenglin and Zhou, Hanzhang and Zhang, Jianan and Chen, Liangyu and others},
  booktitle=ACL,
  pages={6142--6167},
  year={2026}
}

@inproceedings{os-atlas,
  title={Os-atlas: A foundation action model for generalist gui agents},
  author={Wu, Zhiyong and Wu, Zhenyu and Xu, Fangzhi and Wang, Yian and Sun, Qiushi and Jia, Chengyou and Cheng, Kanzhi and Ding, Zichen and Chen, Liheng and Liang, Paul Pu and others},
  booktitle=ICLR,
  year={2025}
}

@inproceedings{webarena,
  title={Webarena: A realistic web environment for building autonomous agents},
  author={Zhou, Shuyan and Xu, Frank F and Zhu, Hao and Zhou, Xuhui and Lo, Robert and Sridhar, Abishek and Cheng, Xianyi and Ou, Tianyue and Bisk, Yonatan and Fried, Daniel and others},
  booktitle=ICLR,
  volume={2024},
  pages={15585--15606},
  year={2024}
}

@inproceedings{mind2web,
  title={Mind2web: Towards a generalist agent for the web},
  author={Deng, Xiang and Gu, Yu and Zheng, Boyuan and Chen, Shijie and Stevens, Sam and Wang, Boshi and Sun, Huan and Su, Yu},
  booktitle=NIPS,
  volume={36},
  pages={28091--28114},
  year={2023}
}

@inproceedings{visualwebarena,
  title={Visualwebarena: Evaluating multimodal agents on realistic visual web tasks},
  author={Koh, Jing Yu and Lo, Robert and Jang, Lawrence and Duvvur, Vikram and Lim, Ming and Huang, Po-Yu and Neubig, Graham and Zhou, Shuyan and Salakhutdinov, Russ and Fried, Daniel},
  booktitle=ACL,
  pages={881--905},
  year={2024}
}

@article{uitars,
  title={UI-TARS: Pioneering automated gui interaction with native agents},
  author={Qin, Yujia and Ye, Yining and Fang, Junjie and Wang, Haoming and Liang, Shihao and Tian, Shizuo and Zhang, Junda and Li, Jiahao and Li, Yunxin and Huang, Shijue and others},
  journal={arXiv preprint arXiv:2501.12326},
  year={2025}
}

@inproceedings{opencua,
  title={OpenCUA: Open Foundations for Computer-Use Agents},
  author={Wang, Xinyuan and Wang, Bowen and Lu, Dunjie and Yang, Junlin and Xie, Tianbao and Wang, Junli and Deng, Jiaqi and Guo, Xiaole and Xu, Yiheng and Wu, Chen Henry and others},
  booktitle=NIPS,
  year={2025}
}

@article{guiowl,
  title={Mobile-Agent-V3: Fundamental Agents for GUI Automation},
  author={Ye, Jiabo and Zhang, Xi and Xu, Haiyang and Liu, Haowei and Wang, Junyang and Zhu, Zhaoqing and Zheng, Ziwei and Gao, Feiyu and Cao, Junjie and Lu, Zhengxi and others},
  journal={arXiv preprint arXiv:2508.15144},
  year={2025}
}

@inproceedings{aguvis,
  title={Aguvis: Unified pure vision agents for autonomous gui interaction},
  author={Xu, Yiheng and Wang, Zekun and Wang, Junli and Lu, Dunjie and Xie, Tianbao and Saha, Amrita and Sahoo, Doyen and Yu, Tao and Xiong, Caiming},
  booktitle = ICML,
  year={2025}
}
\bibliographystyle{iclr2027_conference}

\appendix

\section{Related Work}
\label{sec:related_work}

\paragraph{GUI world models.}
GUI world models predict the consequences of interface actions and have increasingly been used to support planning, simulation, and policy improvement for computer-use agents. Existing approaches represent future GUI states in different forms. MobileWorldBench studies semantic transition modeling for mobile interfaces \citep{li2025mobileworldbench}, while ViMo, UISim, and MobileDreamer explore visual or compact sketch-based prediction of future interfaces \citep{luo2025vimo,xiang2025uisim,cao2026mobiledreamer}. More recent approaches exploit structured representations: gWorld and Code2World generate renderable code for future GUI states, and AppDeltaWorld predicts transition-grounded code updates \citep{koh2026gworld,zheng2026code2world,xu2026appdeltaworld}. R-WoM instead grounds computer-use world-model rollouts with retrieved external knowledge to improve long-horizon simulation \citep{mei2026rwom}, while complementary work studies how different world-model modalities affect GUI-agent training and inference \citep{xu2026howmobileworldmodel}.

These works primarily investigate how accurately or usefully a world model can predict future interfaces under the conditioning information it receives. Our work studies a different failure mode: whether that conditioning information is \emph{sufficient to determine the transition in the first place}. This distinction is important because stronger generation, retrieval, or representation alone cannot resolve two histories that appear identical to the world model but imply different valid next states. Indeed, gWorld explicitly identifies its single-frame Markov assumption as a limitation and motivates incorporating working memory for long-range GUI dependencies \citep{koh2026gworld}. We therefore do not claim that history dependence in GUI dynamics is itself new. Instead, we formalize the resulting information deficiency as state aliasing, construct strict pairs that control the complete model-visible observation and action while varying transition-relevant environment state, and test whether recovering this missing predictive state can repair otherwise frozen GUI world models.

\paragraph{Partial observability, aliasing, and predictive state.}
The distinction between an observation and the underlying state relevant to future behavior is classical in partially observable control \citep{pomdp}. Early reinforcement-learning work studies \emph{perceptual aliasing}, where distinct environment states induce indistinguishable immediate observations despite differing in behaviorally relevant ways \citep{whitehead1991learning,chrisman1992perceptual}. Predictive state representations characterize dynamical state through action-conditioned predictions of future observations rather than solely through latent environment variables \citep{littman2001predictive,singh2004predictive}. Computational mechanics similarly groups histories according to their predictive distributions and establishes sufficiency and minimality properties of the resulting causal states \citep{shalizi2001computational}; related information-state formulations study compact sufficient representations for partially observed systems \citep{subramanian2022approximate}.

Our formulation builds on these classical ideas rather than introducing partial observability, aliasing, or predictive state as new general concepts. We specialize them to GUI world modeling by asking what \emph{residual predictive information is missing given the current GUI}. In particular, we consider one-step controlled next-observation prediction over a specified action scope and use semantically grounded structured states as operational approximations to the corresponding residual predictive state. This viewpoint turns a general partial-observability issue into an auditable world-model diagnostic: strict pairs identify when the current GUI fails to constitute a sufficient predictive state, while our risk decomposition separates missing predictive information from the residual error of a world model after that information is supplied.

\paragraph{History and memory in GUI agents.}
History and memory have also been studied extensively on the \emph{agent side} of GUI interaction. GUI-Rise learns compact history summaries together with structured reasoning to improve subsequent action prediction \citep{liu2025guirise}. PAL-UI combines hierarchical summarization with active retrieval of earlier visual observations when historical evidence is needed for long-horizon planning \citep{liu2025palui}, while Chain-of-Memory explicitly maintains short- and long-term memory to preserve task-relevant information across cross-application interactions \citep{gao2025chainofmemory}. These methods address an important agent-side problem: how to retain and retrieve useful information from interaction history so that the policy can choose better future actions.

Our problem is different in both the \emph{target model} and the \emph{prediction objective}. Agent-memory methods can be viewed as enriching an action policy with a history-derived memory, e.g., $\pi(A_t \mid O_t,M_t)$, where $M_t$ helps the agent decide \emph{what action to take}. We instead study the world model after a candidate action has already been specified: given $(O_t,A_t)$, does the world model possess the transition-relevant state required to predict \emph{what will happen next}? Our recovered state therefore conditions the transition model, $q_\theta(O_{t+1}\mid O_t,A_t,\widehat Z_t)$, rather than directly serving as memory for action selection. StateAliasBench is correspondingly designed to evaluate predictive sufficiency of the world-model input, not long-horizon agent memory.

The distinction is also conceptual. A general-purpose memory system may preserve task goals, previous screenshots, reasoning traces, successful or failed actions, and other information useful to an agent, whereas the predictive state considered here is deliberately narrower: it retains only history-dependent environment information relevant to the controlled next-state transition over the evaluated action scope. Consequently, our results should not be interpreted as showing that structured predictive state is a replacement for agent memory or history retrieval. The two mechanisms address different interfaces in the agent--world-model pipeline and can in principle be complementary: agent memory supports decision making, while predictive-state recovery addresses information missing from the world model used to simulate the consequences of those decisions.

\section{Proofs for Predictive State Recovery}
\label{app:predictive_state_proofs}

This appendix provides proofs for the theoretical results in
Section~\ref{sec:predictive_state}. Throughout, $O_t$ is a deterministic
function of the interaction history $H_t$. For the risk decomposition,
all information-theoretic quantities are taken under the fixed evaluation
distribution, and we assume the relevant conditional entropies and KL
divergences are well defined and finite.

\subsection{Proof of Predictive Sufficiency and Minimality}
\label{app:proof_predictive_sufficiency}

Recall the observation-preserving predictive equivalence relation
\begin{equation}
h\sim_{+}h'
\iff
O_t(h)=O_t(h')
\;\land\;
P_h^a=P_{h'}^a,
\qquad
\forall a\in\mathcal A_c,
\end{equation}
where
\begin{equation}
P_h^a(\cdot)
=
p\!\left(
O_{t+1}\in\cdot
\mid
H_t=h,\operatorname{do}(a)
\right).
\end{equation}
The canonical predictive representation is
\begin{equation}
X_t^\star=[H_t]_{+}
\cong
(O_t,Z_t^\star).
\end{equation}

\begin{proof}[Proof of Theorem~\ref{thm:predictive_sufficiency}]
We first establish predictive sufficiency. By construction,
$X_t^\star$ is the equivalence class of $H_t$ under $\sim_{+}$.
Hence, for any two histories $h$ and $h'$ satisfying
\[
[ h ]_{+}=[ h' ]_{+},
\]
the definition of $\sim_{+}$ implies
\[
P_h^a=P_{h'}^a,
\qquad
\forall a\in\mathcal A_c.
\]
Therefore, for every fixed action $a\in\mathcal A_c$, the conditional
next-observation law is constant within each equivalence class. It
follows that there exists a class-indexed conditional distribution
$\kappa_a$ such that
\begin{equation}
P_h^a
=
\kappa_a([h]_{+})
\qquad
\text{for every reachable }h.
\label{eq:class_kernel}
\end{equation}
Since $X_t^\star=[H_t]_{+}$ is a deterministic function of $H_t$,
Eq.~\ref{eq:class_kernel} gives
\begin{equation}
p(O_{t+1}\mid H_t,\operatorname{do}(a))
=
p(O_{t+1}\mid X_t^\star,\operatorname{do}(a)),
\end{equation}
which proves Eq.~\ref{eq:predictive_sufficiency}.

We next establish minimality. Let $R_t=r(H_t)$ be any deterministic
representation such that $(O_t,R_t)$ is predictively sufficient over
the same action scope. Consider any two reachable histories $h,h'$ for
which
\begin{equation}
O_t(h)=O_t(h'),
\qquad
r(h)=r(h').
\label{eq:same_or}
\end{equation}
Predictive sufficiency of $(O_t,R_t)$ implies, for every
$a\in\mathcal A_c$,
\begin{align}
P_h^a
&=
p\!\left(
O_{t+1}\mid
O_t=O_t(h),R_t=r(h),\operatorname{do}(a)
\right)
\\
&=
p\!\left(
O_{t+1}\mid
O_t=O_t(h'),R_t=r(h'),\operatorname{do}(a)
\right)
\\
&=
P_{h'}^a.
\end{align}
Together with Eq.~\ref{eq:same_or}, this gives
\[
h\sim_{+}h'.
\]
Thus, histories that share the same value of $(O_t,R_t)$ must belong
to the same $\sim_{+}$ equivalence class. Consequently, the equivalence
class is determined by $(O_t,R_t)$: there exists a deterministic map
$f$ such that
\begin{equation}
X_t^\star=f(O_t,R_t),
\end{equation}
which proves Eq.~\ref{eq:predictive_minimal_map}.

Finally, because
\[
X_t^\star\cong(O_t,Z_t^\star),
\]
the residual coordinate $Z_t^\star$ is determined by
$(O_t,X_t^\star)$ and hence, from the result above, by $(O_t,R_t)$.
Therefore
\begin{equation}
H(Z_t^\star\mid O_t,R_t)=0.
\end{equation}
Using conditional mutual information,
\begin{align}
H(Z_t^\star\mid O_t)
&=
I(Z_t^\star;R_t\mid O_t)
+
H(Z_t^\star\mid O_t,R_t)
\\
&=
I(Z_t^\star;R_t\mid O_t)
\\
&\le
H(R_t\mid O_t).
\end{align}
This proves Eq.~\ref{eq:predictive_minimality}.
\end{proof}

The theorem therefore formalizes the sense in which
$X_t^\star$ is the coarsest observation-preserving representation that
retains all one-step controlled transition information over
$\mathcal A_c$. Any other deterministic sufficient representation may
retain additional history information, but must be rich enough to
determine the same predictive equivalence class once $O_t$ is known.

\subsection{Predictive Information Gap}
\label{app:predictive_information_gap}

The predictive information gap used in the main text follows directly
from the definition of conditional mutual information:
\begin{align}
\mathcal G
&=
I(H_t;O_{t+1}\mid O_t,A_t)
\\
&=
H(O_{t+1}\mid O_t,A_t)
-
H(O_{t+1}\mid H_t,O_t,A_t).
\end{align}
Because $O_t$ is determined by $H_t$,
\[
H(O_{t+1}\mid H_t,O_t,A_t)
=
H(O_{t+1}\mid H_t,A_t),
\]
and hence
\begin{equation}
\mathcal G
=
H(O_{t+1}\mid O_t,A_t)
-
H(O_{t+1}\mid H_t,A_t).
\label{eq:appendix_gap_identity}
\end{equation}
Thus, under log-loss, $\mathcal G$ is exactly the difference between
the Bayes-optimal observation-only risk and the Bayes-optimal
full-history risk.

\subsection{Proof of the Predictive-Risk Decomposition}
\label{app:proof_risk_decomposition}

\begin{proof}[Proof of Theorem~\ref{thm:risk_decomposition}]
Let
\[
Y\triangleq O_{t+1},
\qquad
V_t\triangleq(O_t,A_t,U_t),
\]
where $U_t=u(H_t,A_t)$ is deterministic. The expected world-model
log-loss can be decomposed into conditional entropy and conditional
KL divergence:
\begin{align}
\mathcal L(\theta,U)
&=
\mathbb E
\left[
-\log q_\theta(Y\mid V_t)
\right]
\\
&=
H(Y\mid V_t)
+
\mathbb E_{V_t}
D_{\mathrm{KL}}
\!\left(
p(Y\mid V_t)
\Vert
q_\theta(Y\mid V_t)
\right).
\label{eq:cross_entropy_decomposition}
\end{align}

Subtracting the full-history Bayes risk
\[
\mathcal L_H^\star
=
H(Y\mid H_t,A_t)
\]
from Eq.~\ref{eq:cross_entropy_decomposition} gives
\begin{align}
\mathcal L(\theta,U)-\mathcal L_H^\star
&=
H(Y\mid O_t,A_t,U_t)
-
H(Y\mid H_t,A_t)
\nonumber\\
&\quad+
\mathbb E
D_{\mathrm{KL}}
\!\left(
p(Y\mid O_t,A_t,U_t)
\Vert
q_\theta(Y\mid O_t,A_t,U_t)
\right).
\label{eq:risk_before_mi}
\end{align}

Since $O_t$ is a deterministic function of $H_t$ and
$U_t=u(H_t,A_t)$ is deterministic,
conditioning additionally on $(O_t,U_t)$ does not change the
conditional distribution once $(H_t,A_t)$ is known:
\begin{equation}
H(Y\mid H_t,O_t,U_t,A_t)
=
H(Y\mid H_t,A_t).
\end{equation}
By the definition of conditional mutual information,
\begin{align}
I(H_t;Y\mid O_t,U_t,A_t)
&=
H(Y\mid O_t,U_t,A_t)
\nonumber\\
&\quad-
H(Y\mid H_t,O_t,U_t,A_t)
\\
&=
H(Y\mid O_t,U_t,A_t)
-
H(Y\mid H_t,A_t).
\end{align}
Substituting this identity into
Eq.~\ref{eq:risk_before_mi} yields
\begin{align}
\mathcal L(\theta,U)-\mathcal L_H^\star
&=
\underbrace{
I(H_t;O_{t+1}\mid O_t,U_t,A_t)
}_{\delta_{\mathrm{rep}}(U)}
\nonumber\\[-1mm]
&\quad+
\underbrace{
\mathbb E
D_{\mathrm{KL}}
\!\left(
p(O_{t+1}\mid O_t,U_t,A_t)
\Vert
q_\theta(O_{t+1}\mid O_t,U_t,A_t)
\right)
}_{\epsilon_{\mathrm{WM}}(U)}.
\end{align}
This is Eq.~\ref{eq:risk_decomposition}.
\end{proof}

The two terms have distinct sources. The representation term
$\delta_{\mathrm{rep}}(U)$ vanishes exactly when the supplied
conditioning is sufficient for the next-observation law relative to
the full history under the evaluation distribution. The model term
$\epsilon_{\mathrm{WM}}(U)$ vanishes when the frozen world model
matches the true conditional transition law given that conditioning.

The limiting cases in Eq.~\ref{eq:rep_error_extremes} follow
immediately. With no additional conditioning,
\begin{align}
\delta_{\mathrm{rep}}(\varnothing)
&=
I(H_t;O_{t+1}\mid O_t,A_t)
\\
&=
\mathcal G.
\end{align}
With the complete history available,
\begin{equation}
\delta_{\mathrm{rep}}(H_t)
=
I(H_t;O_{t+1}\mid O_t,H_t,A_t)
=
0.
\end{equation}
Likewise, predictive sufficiency of the canonical state gives
\begin{equation}
\delta_{\mathrm{rep}}(Z_t^\star)=0,
\end{equation}
where, as in the main text, the current observation $O_t$ is supplied
separately and $Z_t^\star$ denotes only the residual predictive
coordinate.

\subsection{Proof of the Improvement Corollary}
\label{app:proof_improvement_corollary}

\begin{proof}[Proof of Corollary~\ref{cor:strict_improvement}]
By Theorem~\ref{thm:risk_decomposition},
\begin{equation}
\mathcal L(\theta,U)
=
\mathcal L_H^\star
+
\delta_{\mathrm{rep}}(U)
+
\epsilon_{\mathrm{WM}}(U).
\label{eq:corollary_start}
\end{equation}
By Eq.~\ref{eq:appendix_gap_identity},
\begin{equation}
H(O_{t+1}\mid O_t,A_t)
=
\mathcal L_H^\star+\mathcal G.
\label{eq:obs_bayes_risk}
\end{equation}
If
\begin{equation}
\delta_{\mathrm{rep}}(U)
+
\epsilon_{\mathrm{WM}}(U)
<
\mathcal G,
\end{equation}
then combining
Eqs.~\ref{eq:corollary_start} and
\ref{eq:obs_bayes_risk} gives
\begin{align}
\mathcal L(\theta,U)
&<
\mathcal L_H^\star+\mathcal G
\\
&=
H(O_{t+1}\mid O_t,A_t).
\end{align}
Therefore,
\begin{equation}
\mathcal L(\theta,U)
<
H(O_{t+1}\mid O_t,A_t),
\end{equation}
which is Eq.~\ref{eq:strict_improvement}.
\end{proof}
\section{StateAliasBench Construction and Validation}
\label{app:statealiasbench}

This appendix provides the complete construction, validation, and
information-access details for StateAliasBench. We first describe the
benchmark composition, causal templates, and strict-pair acceptance protocol,
then specify the provenance of the structured state and family-specific
semantic validators. We further audit equality at the actual world-model
request boundary, characterize the information available to the
predictive-state estimator, detail the Temporal Deadline execution protocol,
and evaluate repeated-execution stability of the controlled transitions.

\subsection{Benchmark Composition and Causal Templates}
\label{app:benchmark_composition}

StateAliasBench contains four transition-relevant state families and
25 causal templates: 9 Persistent Artifact templates, 5 Clipboard
templates, 5 Ordered Collection templates, and 6 Temporal Deadline
templates. Each family contributes 200 verified strict pairs, yielding
800 pairs and 1,600 executed branches in total.

We additionally audit the released strict benchmark independently of model
predictions. All 25 declared templates are present in the released records,
all four family-specific validators return \texttt{PASS}, and all 800
released pairs satisfy their corresponding release-level validation criteria.

\subsubsection{Template Definitions and Released Parameter Support}
\label{app:template_definitions}

Each causal template specifies
(i) how the two branches establish different transition-relevant states,
(ii) how they subsequently converge to a common model-visible prediction
boundary, and
(iii) a shared probe action whose semantic consequence reveals the hidden-state
distinction. Table~\ref{tab:template_definitions} summarizes all 25 templates.

For non-temporal families, the parameterization reported below denotes the
complete support observed in the released strict split rather than an
unobserved generator-domain range. For Temporal Deadline, the retained
generator additionally specifies the candidate construction domain and the
deterministic branch-assignment rule. No template definition is inferred from
its name alone; the definitions are grounded in the released interaction
traces, recorded actions, state evidence, and post-action semantic probes.

\begin{table*}[t]
\centering
\scriptsize
\setlength{\tabcolsep}{4pt}
\renewcommand{\arraystretch}{1.12}

\caption{
Causal templates in StateAliasBench. ``State contrast'' summarizes how
the two histories establish different structured states before returning
to a common prediction boundary. Parameterization reports the support
observed in the released strict split.
}
\label{tab:template_definitions}

\begin{tabularx}{\textwidth}{
>{\raggedright\arraybackslash}p{0.19\textwidth}
>{\raggedright\arraybackslash}X
>{\raggedright\arraybackslash}p{0.21\textwidth}
>{\raggedright\arraybackslash}p{0.18\textwidth}
}
\toprule
Template
& State contrast and convergence
& Shared probe / semantic outcome
& Released support \\
\midrule

\multicolumn{4}{l}{\textbf{Persistent Artifact}} \\
\addlinespace[2pt]

\path{copy_direct}
& A confirms COPY while B cancels; both return to the same focused
destination search.
& Search the target; present vs.\ absent.
& 20 pairs; 20 targets; 5 file extensions. \\

\path{move_direct}
& A confirms MOVE while B cancels; both return to the same focused
destination search.
& Search the target; present vs.\ absent.
& 20 pairs; 20 targets; 5 extensions. \\

\path{delete_direct}
& A confirms DELETE while B cancels; both return to the same focused
source search.
& Search the target; absent vs.\ present.
& 20 pairs; 20 targets; 5 extensions. \\

\path{copy_interference}
& Branches differ in target COPY while an independent decoy COPY is
executed oppositely; both converge to the same neutral folder.
& Open target destination; target present vs.\ absent.
& 34 pairs; one decoy operation per branch. \\

\path{move_interference}
& Target MOVE differs while an independent decoy MOVE is inverted;
both converge to the same neutral folder.
& Open target destination; target present vs.\ absent.
& 33 pairs; one decoy operation per branch. \\

\path{delete_interference}
& Target DELETE differs while a path-isolated decoy DELETE is inverted;
both converge to the same neutral folder.
& Open target source; target absent vs.\ present.
& 33 pairs; one decoy operation per branch. \\

\path{copy_delete_source}
& Both COPY; only A subsequently deletes the source copy.
& Open source; absent vs.\ present.
& 14 pairs; two target transitions. \\

\path{copy_delete_destination}
& Both COPY; only A subsequently deletes the destination copy.
& Open destination; absent vs.\ present.
& 13 pairs; two target transitions. \\

\path{move_delete_destination}
& Both MOVE; only A subsequently deletes the moved destination copy.
& Open destination; absent vs.\ present.
& 13 pairs; two target transitions. \\

\midrule
\multicolumn{4}{l}{\textbf{Clipboard}} \\
\addlinespace[2pt]

\path{original_strict}
& A/B copy different payloads, including single-write and overwrite
histories, then converge to the same blank editor.
& Paste; editor equals branch-specific payload.
& 40 pairs; 80 unique payloads; 7--9 chars. \\

\path{copy_cancel_control}
& Each copies a payload, then selects another source but cancels without
copying; both converge to the same blank editor.
& Paste; earlier copied payload is preserved.
& 40 pairs; 80 unique payloads. \\

\path{same_page_partial_selection}
& A/B copy different visible lines from the same source page.
& Paste; branch-specific selected line appears.
& 40 pairs; two selection locations. \\

\path{copy_then_mutate_source}
& A/B copy different payloads, then source text is changed without
another copy.
& Paste; pre-mutation payload appears.
& 40 pairs; 80 payloads; 40 neutral mutations. \\

\path{paste_non_consuming}
& A/B copy different payloads and paste/delete them in a scratch note
before converging.
& Paste again; original payload persists.
& 40 pairs; 80 unique payloads. \\

\midrule
\multicolumn{4}{l}{\textbf{Ordered Collection}} \\
\addlinespace[2pt]

\path{playing_queue}
& Same songs are appended in different orders; both return to the same
Now Playing anchor.
& Open queue; branch-specific ordering is revealed.
& 40 pairs; queue size 5. \\

\path{queue_mixed_insertion}
& A/B perform different head/middle/tail drag insertions on the same
six-item queue.
& Open queue; branch-specific ordering is revealed.
& 40 pairs; queue size 6; three drags/branch. \\

\path{identical_prefix_hidden_suffix}
& Same 10-item playlist and identical first eight items, with the last
two inserted in opposite order.
& Open playlist; hidden suffix ordering is revealed.
& 40 pairs; playlist size 10. \\

\path{reset_boundary}
& A/B remove all items to an empty boundary and rebuild the same item
set in different orders.
& Open playlist; post-reset ordering is revealed.
& 40 pairs; playlist size 4--5. \\

\path{multi_playlist_interleaved}
& Additions to target and distractor playlists are interleaved while
the target ordering differs.
& Open target playlist; target order only is scored.
& 40 pairs; target size 3--4. \\

\midrule
\multicolumn{4}{l}{\textbf{Temporal Deadline}} \\
\addlinespace[2pt]

\path{start_time_shift}
& Same-duration timers are started at different visible times and
converge to the same observation time.
& Shared wait + notification probe; fired vs.\ not fired.
& 34 pairs; start offset 1--6 min; wait 3--29 min. \\

\path{duration_shift}
& Same start time but short vs.\ long timer duration.
& Shared wait + notification probe.
& 34 pairs; duration 2--26 min; wait 3--19 min. \\

\path{run_pause_contrast}
& Same timer and elapsed interval; one remains running and the other
is paused.
& Shared wait + notification probe.
& 33 pairs; pre-control elapsed 1--5 min. \\

\path{pause_reset_contrast}
& Both pause; one resumes the remainder while the other resets and
restarts at full duration.
& Shared wait + notification probe.
& 33 pairs; elapsed 1--5 min; wait 3--29 min. \\

\path{pause_resume_control}
& Both pause; one remains paused while the other resumes.
& Shared wait + notification probe.
& 33 pairs; elapsed 1--5 min. \\

\path{pause_delete_restart}
& Both pause/delete, then restart using short vs.\ long durations.
& Shared wait + notification probe.
& 33 pairs; restart duration 2--26 min. \\

\bottomrule
\end{tabularx}
\end{table*}

\subsubsection{Strict-Pair Construction and Acceptance Criteria}
\label{app:strict_pair_acceptance}

Each causal template is instantiated with different objects, values,
applications, interaction parameters, and execution traces to produce two
branches with distinct structured states $\widetilde Z_t^{f,A}$ and
$\widetilde Z_t^{f,B}$. The branches are then brought to a common prediction
boundary, where the complete model-visible current observation and shared next
action are matched. The shared action is executed in the real environment,
and the resulting family-specific semantic outcome is validated independently
against the corresponding environment-side state evidence.

A candidate strict pair is retained only when all of the following conditions
hold:
\begin{enumerate}
    \item the two branches reach valid and stable prediction-boundary
    environment states;
    \item the complete model-visible current observation and shared action
    match at the evaluated WM input boundary;
    \item the independently recorded structured states differ,
    $\widetilde Z_t^{f,A}\neq\widetilde Z_t^{f,B}$;
    \item the shared action produces distinct family-specific semantic outcomes,
    $Y_{t+1}^{f,A}\neq Y_{t+1}^{f,B}$; and
    \item both branches pass the corresponding family-specific semantic
    validator.
\end{enumerate}

Executions with failed actions, invalid environment states, unstable captures,
missing targets, or semantic inconsistencies are rejected or retried rather
than admitted as strict pairs. Consequently, incidental screenshot variation
or execution failure is never used as evidence of state aliasing.

Observation equality is evaluated after the exact model-specific input
transformation rather than only on raw device screenshots; the complete audit
is given in Appendix~\ref{app:model_visible_matching}. Hidden-state
annotations and semantic targets are obtained from read-only environment-side
evidence as described in Appendices~\ref{app:ground_truth_state}
and~\ref{app:semantic_validation}. Temporal templates additionally use the
controlled execution protocol in Appendix~\ref{app:temporal_protocol}.

The strict benchmark intentionally contains multiple realizations of the same
underlying state distinction and includes compositions of state-update
primitives in addition to simpler transitions. Predictive-state training data
and its separation from the strict evaluation set are described in
Appendix~\ref{app:train_eval_separation} and Appendix~\ref{app:state_training_data}.

\subsection{Train--Evaluation Split Integrity}
\label{app:train_eval_separation}

We audit the separation between the exact four-family source pool used for
predictive-state training and the frozen StateAliasBench evaluation split.
Only examples used for gradient updates are counted on the training side;
validation and test records are excluded.

The purpose of this audit is to detect duplicated examples or trajectories
across training and evaluation, rather than to require disjointness of
low-level GUI primitives. StateAliasBench intentionally reuses common
interface states---for example neutral Settings pages, blank editors, or
shared application views---and common interaction primitives such as taps,
navigation actions, and waits. Such reuse is expected in a compositional GUI
benchmark and does not by itself constitute data leakage.

We therefore evaluate split integrity at five leakage-relevant levels:
(i) exact record identity,
(ii) complete trajectory identity,
(iii) semantic target identity,
(iv) canonical content fingerprint, and
(v) complete normalized history-trajectory signature.
The canonical content fingerprint hashes the full model-visible
training/evaluation record, including the ordered message structure and
preprocessed visual content, while the trajectory signature hashes the
complete ordered sequence of normalized historical actions.

As summarized in Table~\ref{tab:train_eval_overlap}, all five criteria have
zero train--evaluation intersection. Thus, no evaluation example duplicates a
training record, target identity, complete model-visible record, or full
interaction trajectory.

\begin{table}[t]
\centering
\small
\caption{
Train--evaluation split-integrity audit for predictive-state recovery.
All reported criteria test duplicate examples or complete trajectories,
rather than reuse of lower-level GUI frames or action primitives.
}
\label{tab:train_eval_overlap}
\begin{tabular}{lrrr}
\toprule
Audit criterion & Train & Eval & Intersection \\
\midrule
Record identifier
& 2,538 & 1,600 & 0 \\
Trajectory identifier
& 1,658 & 800 & 0 \\
Semantic target identity
& 825 & 611 & 0 \\
Canonical content fingerprint
& 2,516 & 1,600 & 0 \\
Whole-trajectory signature
& 1,741 & 854 & 0 \\
\bottomrule
\end{tabular}
\end{table}

Lower-level visual states and action primitives are intentionally reusable
across splits and are therefore not treated as disjointness criteria; the audit
instead tests whether their complete composition forms a duplicated training
or evaluation example.

\subsection{Ground-Truth Structured State and Provenance}
\label{app:ground_truth_state}

StateAliasBench associates each branch with a structured transition-relevant
state $\widetilde Z_t^f$. Importantly, this ground-truth state is obtained from
environment-side evidence that is independent of the model-visible GUI,
rather than inferred from the screenshots used by either the evaluated world
model or the predictive-state estimator.

Where direct application or system state is available, we record the realized
state through read-only environment-side evidence. For Persistent Artifact,
the annotation records whether the target artifact exists at the specified
source and destination locations at the prediction boundary. For Clipboard,
the annotation records the current clipboard payload and payload type using an
independently recorded clipboard-state label. For Ordered Collection, the
annotation records the collection type and realized complete ordering of the
target queue or playlist together with environment-side order evidence.

\paragraph{Temporal Deadline state.}
Temporal Deadline requires a different treatment because the relevant state
evolves continuously with time. We represent its structured state as
\begin{equation}
\widetilde Z_t^{\mathrm{temp}}
=
(\ell_t,\rho_t),
\qquad
\rho_t=
\begin{cases}
\tau_t^{\mathrm{end}}-\tau_t,
& \ell_t=\mathrm{running},\\
r_t^{\mathrm{pause}},
& \ell_t=\mathrm{paused},\\
\bot,
& \text{otherwise},
\end{cases}
\label{eq:temporal_state}
\end{equation}
where $\ell_t$ denotes the timer lifecycle state and $\rho_t$ its effective
remaining-time state. For a running timer, the effective remaining time is the
difference between the timer deadline $\tau_t^{\mathrm{end}}$ and the current
logical time $\tau_t$; for a paused timer, it is the retained paused remainder
$r_t^{\mathrm{pause}}$. This representation distinguishes timers that can be
visually indistinguishable at the prediction boundary but evolve differently
under the same subsequent wait or control action.

The canonical timer annotation is constructed under the benchmark's
controlled logical-time protocol. It records the timer lifecycle and effective
remaining time at the prediction boundary and serves as the Temporal Deadline
supervision target. Simulator-side logical or virtual timestamps, timer
backend diagnostics, and time-warp implementation variables are used only for
construction and auditing and are never provided to the predictive-state
estimator.

Across all families, environment-side state records are used only for
benchmark construction, supervision, and validation. They are never included
in the observation-only world-model input. At predictive-state inference time,
the estimator receives only the family/schema instruction and model-visible
interaction history; benchmark-internal state records, backend/ADB evidence,
logical-time variables, and the ground-truth $\widetilde Z_t^f$ are excluded.

For every accepted strict pair,
\[
\widetilde Z_t^{f,A}
\neq
\widetilde Z_t^{f,B}.
\]
The release audit re-reads the corresponding family-specific environment-side
evidence and independently verifies this state contrast before the pair is
included. Thus, $\widetilde Z_t^f$ is an external annotation of the realized
environment state, rather than a quantity derived from the prediction target
or from the evaluated model itself.

\subsection{Family-Specific Semantic Validation}
\label{app:semantic_validation}

Strict-pair validity is defined through family-specific semantic outcomes
rather than arbitrary screenshot inequality. For each family $f$, let
\[
Y_{t+1}^f
=
\psi_f
\!\left(
\operatorname{Parse}(O_{t+1})
\right)
\]
denote the semantic consequence of the shared probe action. The semantic
variable is deliberately narrower than the full post-action GUI: it retains
only the outcome whose dependence on the transition-relevant state is being
tested.

At benchmark-construction time, $Y_{t+1}^f$ is obtained from independently
recorded post-action environment evidence rather than from the evaluated world
model. For Persistent Artifact, the outcome records whether the tracked
artifact is present on the queried Files surface. For Clipboard, it is the
payload inserted into the target editor by the shared PASTE action. For
Ordered Collection, it is the effective ordered sequence exposed by opening
the target playlist or playing queue. For Temporal Deadline, it is the timer
lifecycle outcome revealed by the controlled wait-and-notification probe.

A released strict pair must satisfy
\[
O_t^A = O_t^B,\qquad
A_t^A = A_t^B,\qquad
\widetilde Z_t^{f,A}
\neq
\widetilde Z_t^{f,B},
\]
together with
\[
Y_{t+1}^{f,A}
\neq
Y_{t+1}^{f,B}.
\]
The semantic contrast is therefore tied to the diagnosed state variable rather
than to incidental visual variation.

\paragraph{Generated-prediction readout.}
For a generated next observation $\widehat O_{t+1}$, we evaluate the same
family-level semantic variable,
\begin{equation}
\widehat Y_{t+1}^f
=
\psi_f
\!\left(
\operatorname{Parse}_{\mathrm{vis}}
(\widehat O_{t+1})
\right),
\label{eq:strict_generated_semantics}
\end{equation}
but instantiate $\operatorname{Parse}_{\mathrm{vis}}$ on the rendered
prediction rather than on raw HTML source. Generated HTML is rendered offline
in Chromium with external network requests disabled. The parser extracts text,
editable-control values, and limited control/icon evidence only from the
initial rendered viewport.

A text or control region contributes semantic evidence only if
(i) its element is connected and is not hidden by
\texttt{display:none}, \texttt{visibility:hidden}, or near-zero opacity;
(ii) it has positive rendered geometry;
(iii) at least $98\%$ of its rendered area lies within the initial viewport;
and
(iv) at least $80\%$ of sampled points are not materially occluded by another
element. Thus, hidden DOM text, below-the-fold content, materially clipped
content, and substantially occluded content cannot satisfy a strict semantic
target. Visible values of \texttt{input}, \texttt{textarea}, and editable
controls are included. The same visibility rule is used for both evaluated
WMs, under their respective rendering viewports.

Before family-specific matching, visible textual evidence is normalized by
HTML-entity decoding, case folding, and whitespace collapsing. Scoring is
deterministic and uses neither OCR nor an LLM-based judge.

\paragraph{Persistent Artifact.}
Persistent Artifact uses two probe types.

For a filename-search probe, the semantic variable is whether the target
artifact appears as a result in the benchmark's Downloads-scoped search
surface. An explicit visible ``No matches'' state resolves the prediction as
absent. Otherwise, when the typed target appears in the visible search
control, presence requires an additional occurrence in a visible result text
node; the query echo itself is therefore not sufficient evidence of a
positive result. The generated surface must also visibly identify the
Downloads scope and contain a plausible multi-row software-keyboard layout.

For an open-folder probe, presence is determined from the viewport-visible
folder listing. The readout additionally checks that the primary folder scope
and its parent scope agree with the queried path, rejects a contradictory
visible software keyboard, and treats duplicated filename/folder rows as an
invalid listing. Artifact identifiers are compared after text normalization
using string containment together with these scope and structural checks,
rather than byte-exact HTML matching. The frozen strict evaluation was audited
to contain no target/anchor substring collisions that would make this matching
ambiguous.

The resulting semantic outcome is therefore
\[
Y_{t+1}^{\mathrm{pers}}
\in
\{\mathrm{present},\mathrm{absent}\}.
\]

\paragraph{Clipboard.}
Let $v_b$ denote the branch-specific expected clipboard payload. A prediction
is semantically correct only if $v_b$ occurs in visible editor-body evidence
rather than merely in top-level application chrome or hidden metadata. We use
a scale-relative boundary at $6.5\%$ of viewport height to exclude the
top application/title region. The sibling branch's clipboard payload must
simultaneously be absent.

Thus,
\[
\widehat Y_{t+1}^{\mathrm{clip}}
=
v_b
\]
only when the expected payload is visibly grounded in the editor body and no
counterfactual sibling payload is present. Clipboard payloads in the strict
split are nonempty and branch specific.

\paragraph{Ordered Collection.}
For Ordered Collection, the semantic target is the complete ordered sequence
\[
Y_{t+1}^{\mathrm{ord}}
=
(x_1,\ldots,x_n).
\]
The generated-output parser identifies candidate primary rows using their
rendered position, horizontal alignment, and typography, then orders the
resulting row labels from top to bottom. A prediction is correct only when
every expected item occurs exactly once as a primary visible row, no foreign
benchmark item appears in that row set, and the extracted sequence equals the
target sequence exactly. Partial sequences, duplicate items, extra
counterfactual items, and the sibling branch's ordering are therefore scored
incorrect.

This scoring rule is intentionally order sensitive: merely containing the
correct item set does not satisfy the semantic target.

\paragraph{Temporal Deadline.}
Temporal Deadline distinguishes three post-action outcomes,
\[
Y_{t+1}^{\mathrm{temp}}
\in
\{\mathrm{none},\mathrm{running},\mathrm{fired}\}.
\]

A \emph{fired} prediction requires visible fired-state evidence
(\emph{Time's up}, or an equivalent zero timer state) together with a visible
\emph{Stop} control. It is rejected if the same rendered surface contains
contradictory running-timer evidence, a positive remaining duration, a
\emph{Pause} control, alarm-state content, duplicated state-bearing controls,
or other internally inconsistent timer structure.

A \emph{running} prediction requires a visible nonzero remaining duration and
a visible \emph{Pause} control, with no fired-state or alarm contradiction.
If the environment-side target has remaining time $r$ at the end of the
controlled wait, the rendered countdown is accepted only in
\[
[r-6,r]
\]
seconds, clipped below at zero. This narrow tolerance accounts for ordinary
capture delay after the deterministic wait rather than permitting an
arbitrary nonzero timer value.

A \emph{none} prediction requires absence of timer-state evidence while still
being grounded on a recognizable notification-shade surface; merely omitting
timer text on an unrelated GUI is insufficient. For all Temporal outcomes,
the required notification-shade footer must remain visibly complete.
Timer-like outputs that contain mutually contradictory state evidence are
classified as invalid rather than assigned to either branch.

\paragraph{Correctness aggregation.}
For branch $b\in\{A,B\}$, semantic correctness is
\[
\mathrm{BC}_i^b
=
\mathbf 1
\!\left[
\widehat Y_{i,t+1}^{f,b}
=
Y_{i,t+1}^{f,b}
\right].
\]
The strict-pair metric is
\[
\mathrm{PairCorrect}_i
=
\mathrm{BC}_i^A
\mathrm{BC}_i^B,
\]
so a pair is correct only when both branch-specific futures are recovered.
This prevents a model from receiving credit for producing one plausible
future while failing to distinguish the two aliased states.

\paragraph{Failure handling.}
Semantic scoring is sample preserving: unsuccessful or malformed predictions
are not silently removed. A failed model request is retained and scored
incorrect after the configured inference retries are exhausted. Missing,
partial, duplicated, contradictory, or otherwise invalid family-specific
outcomes fail the corresponding semantic rule.

HTML is not repaired before scoring. Chromium may nevertheless render a
syntactically incomplete document, so absence of a final closing HTML tag is
recorded as a diagnostic rather than used as an automatic semantic-failure
criterion. Structural dependence on unavailable external resources is
rejected by an offline self-containedness check. If rendering itself fails,
the evaluator retains the sample and falls back to text extraction rather
than dropping it; such fallback text does not have the viewport-visibility
guarantees above.

For the staged interface evaluation in
Table~\ref{tab:state_interface_progression}, all 6,400 generated predictions
have corresponding semantic scores: there are no inference failures, renderer
failures, missing scores, or samples removed from aggregation. Consequently,
the reported differences across interface conditions are not produced by
condition-specific filtering or failed-example exclusion.

For model predictions, these semantic rules deliberately do not require
pixel-level agreement with the recorded future GUI. They evaluate whether the
prediction realizes the state-dependent consequence specified by the strict
pair, while general next-observation fidelity is evaluated separately in
Appendix~\ref{app:visual_fidelity}.

\subsection{Complete Model-Visible Pair Matching}
\label{app:model_visible_matching}

For the observation-only condition, pair equality is verified at the complete
serialized request boundary used by each evaluated world model, rather than
only at the raw screenshot level.

For Code2World, the request contains its fixed system prompt, the current GUI
after app-region cropping and model-specific action visualization, a
normalized action description, the family-specific transition contract, and
any pair-shared entity context required by the task. For gWorld, the request
contains the corresponding model-native prompt, the cropped current GUI
without the Code2World visual hint, the normalized action, the family-specific
transition contract, and the same applicable pair-shared entity context.

Benchmark bookkeeping fields---including pair identifier, branch identifier,
template identifier, target observation, scoring metadata, and ground-truth
state---are not included in the observation-only request. Timestamp fields are
removed from the current action representation, and benchmark-side state
evidence is never supplied to this condition.

We canonicalize the complete client-side request payload and compute a
SHA-256 hash for each branch. Across all 800 strict pairs, branches A and B
have identical canonical request hashes for Code2World (800/800 pairs) and
for gWorld (800/800 pairs). Thus, under the observation-only condition, both
branches provide identical serialized conditioning to the evaluated world
model while requiring different validated semantic futures.

This equality requirement applies specifically to the observation-only
condition. Raw-history and state-conditioned variants intentionally supply
branch-dependent historical or state information and are therefore not
expected to have identical A/B requests.

\subsection{Predictive-State Information Access}
\label{app:state_information_access}

The predictive-state estimator receives an ordered interaction history
together with a family-specific task/schema instruction $q_f$. At inference
time, its effective input consists of the system/task instruction, the ordered
historical GUI observations and actions, and the current neutral observation:
\[
q_f,\quad
O_0,A_0,O_1,\ldots,A_{t-1},O_t.
\]
The prediction-boundary shared action $A_t$ is not included in the
state-estimator input.

The four family-specific schemas specify the structure of the state to be
recovered: clipboard payload, persistent-artifact existence,
ordered-collection contents, or temporal timer state. StateAliasBench does
not require the unified estimator to infer the family itself; the benchmark
family determines which schema instruction and output parser are used. The
same unified estimator parameters are used across all families.

Importantly, benchmark-internal identifiers and label-side evidence are not
provided to the estimator. Template, pair, branch, and sample identifiers are
used only for bookkeeping. Backend/ADB state, timer backend snapshots,
top-level supervision fields, and ground-truth $\widetilde Z_t^f$ are excluded
from inference input. During training, the canonical structured state appears
only as the autoregressive assistant target used for supervision.

\begin{table}[t]
\centering
\small
\caption{
Information-access audit for the predictive-state estimator on
StateAliasBench.
}
\label{tab:state_information_access}
\begin{tabular}{lccc}
\toprule
Information & Train input & Test input & Role \\
\midrule
Family/schema instruction $q_f$ & Yes & Yes & Task specification \\
Ordered GUI history & Yes & Yes & Model-visible \\
Historical actions & Yes & Yes & Model-visible \\
Current GUI observation & Yes & Yes & Model-visible \\
Relative duration (Temporal) & Yes & Yes & Model-visible \\
Wall-clock information in pixels & Possible & Possible & Model-visible \\
Future/shared action $A_t$ & No & No & Evaluation only \\
Template/pair/branch ID & No & No & Bookkeeping \\
Backend/ADB state & No & No & Label provenance \\
Logical/virtual simulator time & No & No & Label/audit only \\
Time-warp implementation label & No & No & Audit only \\
Ground-truth $\widetilde Z_t^f$ & Target only & No & Supervision \\
\bottomrule
\end{tabular}
\end{table}

\subsection{Temporal Execution and Time Information}
\label{app:temporal_protocol}

Temporal Deadline requires temporal evidence to remain observable from the
interaction history while preventing access to simulator-side timing state.
The predictive-state estimator may therefore use evidence available to a user
observing the interaction, including wall-clock information displayed in GUI
screenshots and relative durations associated with historical actions.

In contrast, structured simulator timing information is excluded from the
estimator input. This includes virtual execution timestamps, logical/virtual
epochs, timer backend snapshots, scoring time bases, and the implementation
label used by the benchmark's controlled time-warp mechanism. The estimator
therefore observes interaction-visible temporal evidence but not the canonical
timer state $\widetilde Z_t^{\mathrm{temp}}$ or the logical-time variables from
which that state is constructed.

For benchmark execution, long waits are implemented using the audited
logical-time/time-warp protocol rather than literal long wall-clock sleeps.
The same protocol is applied to both branches of a strict pair. The benchmark
advances the controlled logical time, executes the shared observable probe
action, and determines the semantic outcome from independently recorded
timer-fire and lifecycle evidence.

Temporal strict pairs are accepted only when the two branches reach the same
model-visible prediction boundary and shared action while retaining different
canonical timer states, and when the controlled transition produces the
expected branch-specific semantic contrast. Transient capture-time variation
is not used as the semantic label. The logical-time mechanism is used only to
execute, reconstruct, and validate the environment transition; it is never
exposed as structured timing input to the predictive-state estimator.

\subsection{Repeated-Execution Stability}
\label{app:replay_stability}

A single realized difference in semantic outcome does not, in a stochastic
environment, by itself establish a difference between the underlying
controlled transition laws. We therefore perform an additional
repeated-execution audit covering all 25 causal templates.

We deterministically select one verified strict pair from each template and
replay both branches three times, yielding 50 replayed branches and 150 fresh
controlled executions. No world model or predictive-state estimator is
involved in this audit. Temporal templates use the same audited logical-time
advancement protocol as benchmark construction.

Across all 50 branches, the family-specific semantic outcome is identical
across the three repeated executions. We observe no failed action, reset
failure, capture failure, or within-branch semantic disagreement. Under
faithful reconstruction of the recorded benchmark trajectory, all 150
executions reproduce their expected semantic outcomes.

\subsection{Representative Strict-Pair Trajectories}
\label{app:representative_pairs}

Figures~\ref{fig:traj_persistent}--\ref{fig:traj_temporal} provide one
representative strict-pair trajectory from each state family. Rather than
showing every intermediate interaction frame, we retain the history steps that
are sufficient to reveal how the branch-specific hidden state is established;
omitted intermediate steps are indicated by ellipses. The displayed history
consists only of model-visible GUI observations and recorded actions.
Ground-truth structured states $\widetilde Z_t^{f,A}$ and
$\widetilde Z_t^{f,B}$ are shown separately as audit-only evidence and are
not part of the world-model input.

Each example exhibits the same diagnostic structure:
\[
H_t^A \neq H_t^B,\qquad
O_t^A = O_t^B,\qquad
\widetilde Z_t^{f,A} \neq \widetilde Z_t^{f,B},
\]
followed by the same shared action
\[
A_t^A = A_t^B,
\]
but different validated semantic outcomes
\[
Y_{t+1}^{f,A} \neq Y_{t+1}^{f,B}.
\]
The examples therefore illustrate why the current GUI alone is insufficient:
the transition-relevant distinction is recoverable from the interaction
history but absent from the common prediction-boundary observation.

\paragraph{Persistent Artifact.}
Figure~\ref{fig:traj_persistent} shows a file-move example. The two branches
differ in whether the target artifact has actually been moved, but subsequent
navigation returns both branches to the same empty search interface. At this
shared prediction boundary, the visible GUI is identical while the realized
file-system state differs. Typing the same artifact name therefore reveals
different futures: one branch returns the moved file, while the other does not.

\begin{figure*}[t]
    \centering
    \includegraphics[width=\textwidth]
    {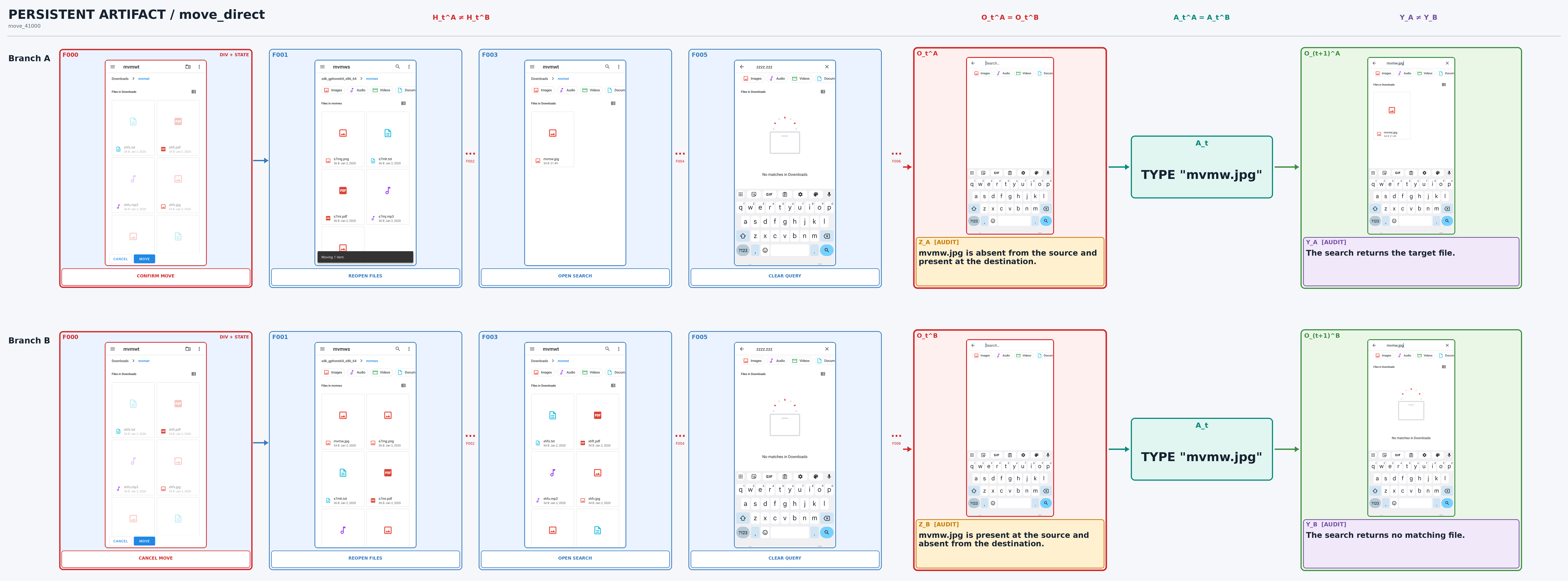}
    \caption{
    \textbf{Persistent Artifact: \texttt{move\_direct}.}
    Representative strict pair \texttt{move\_41000}. The histories differ
    in whether the file move is completed, but both branches return to the
    same prediction-boundary GUI $O_t$. Audit-only state evidence shows
    different source/destination existence states. Executing the same
    search action $A_t$ produces different next observations and semantic
    artifact-presence outcomes.
    }
    \label{fig:traj_persistent}
\end{figure*}

\paragraph{Clipboard.}
Figure~\ref{fig:traj_clipboard} shows two histories that place different text
values into the system clipboard. Both branches subsequently open the same
blank editor, making their current GUI observations indistinguishable. The
same paste action nevertheless inserts the branch-specific clipboard payload,
directly exposing the hidden clipboard state.

\begin{figure*}[t]
    \centering
    \includegraphics[width=\textwidth]
    {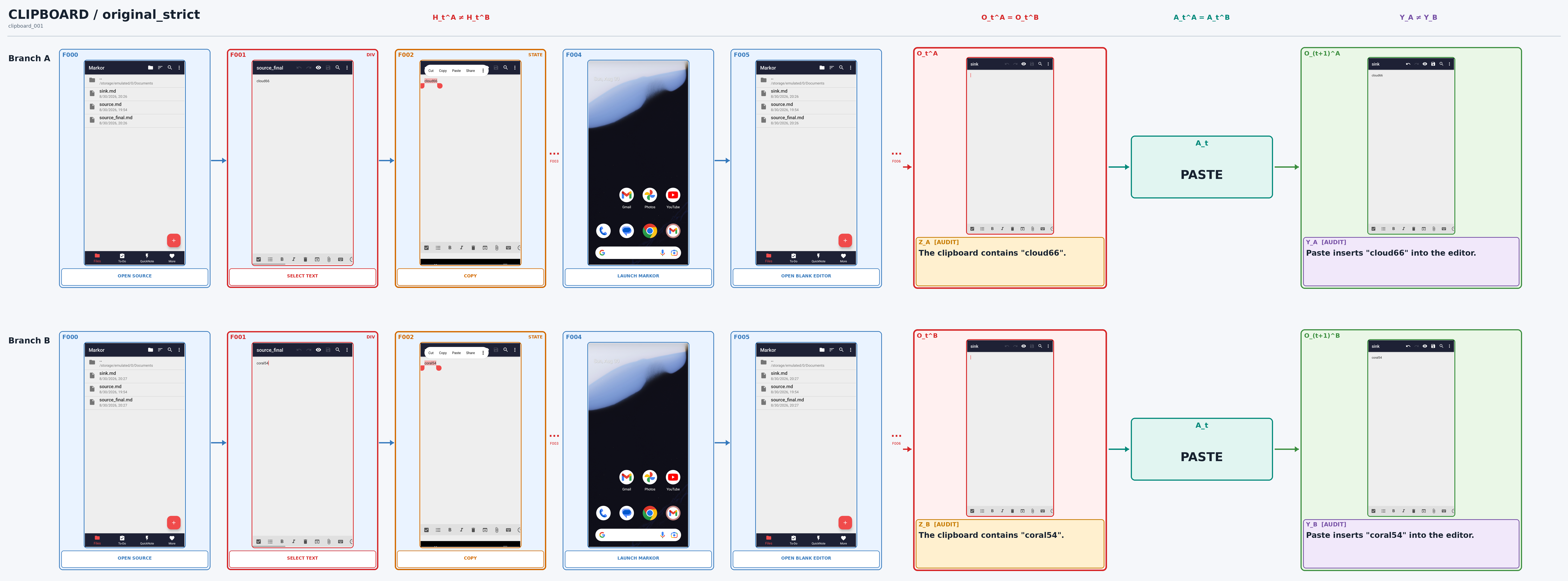}
    \caption{
    \textbf{Clipboard: \texttt{original\_strict}.}
    Representative strict pair \texttt{clipboard\_001}. The two histories
    copy different payloads before converging to the same blank editor
    $O_t$. The current clipboard contents are not visible in that
    observation. Applying the same paste action $A_t$ yields different
    editor contents, revealing the branch-specific clipboard state.
    }
    \label{fig:traj_clipboard}
\end{figure*}

\paragraph{Ordered Collection.}
Figure~\ref{fig:traj_ordered} illustrates an order-sensitive hidden state.
The two branches construct different effective playing-queue orders through
different interaction histories, yet both return to the same Now Playing
interface before prediction. Opening the queue with the same action exposes
the previously hidden ordering difference.

\begin{figure*}[t]
    \centering
    \includegraphics[width=\textwidth]
    {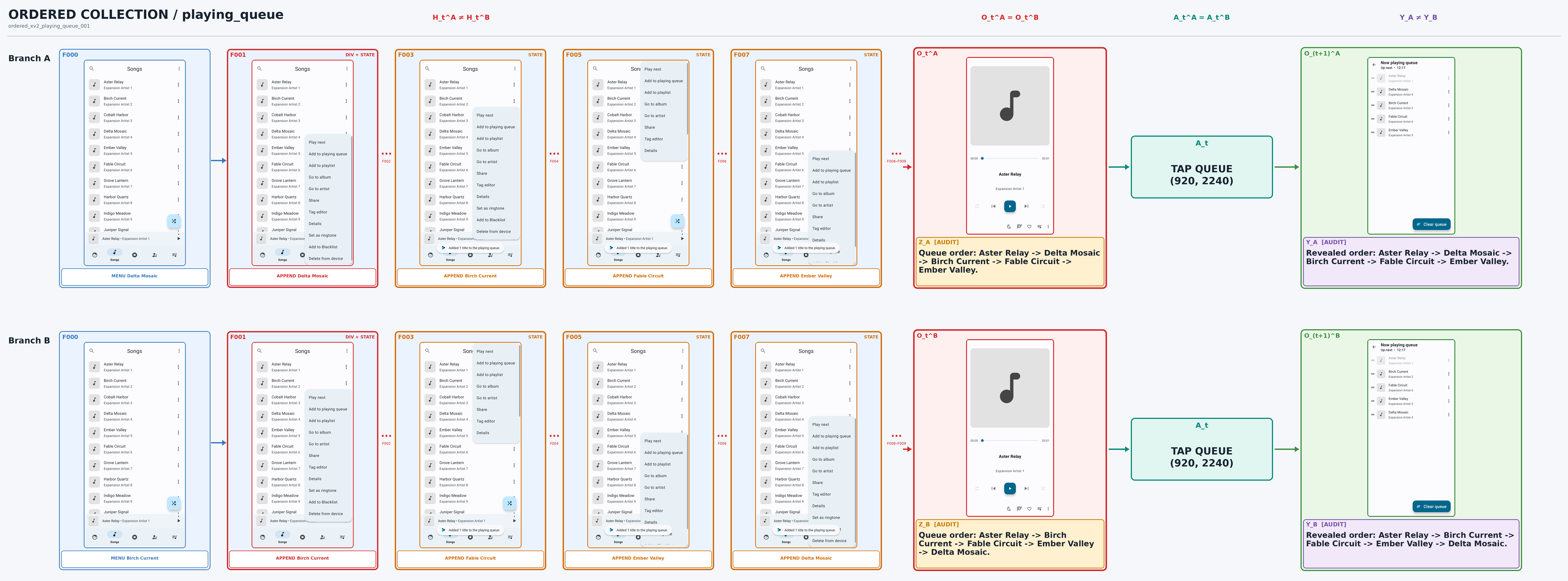}
    \caption{
    \textbf{Ordered Collection: \texttt{playing\_queue}.}
    Representative strict pair
    \texttt{ordered\_xv2\_playing\_queue\_001}. Branch-specific histories
    construct different effective queue orders while converging to the same
    Now Playing observation $O_t$. The shared queue-opening action $A_t$
    reveals different ordered sequences in the next observation.
    }
    \label{fig:traj_ordered}
\end{figure*}

\paragraph{Temporal Deadline.}
Figure~\ref{fig:traj_temporal} shows a temporal example in which the two
histories establish different effective timer deadlines. The history contains
user-visible temporal evidence, including displayed clock anchors and relative
action durations, but both branches are eventually brought to the same Date
\& Time interface. The same controlled wait and notification-shade action
then produces different timer-fire outcomes because the effective
remaining-time state differs.

\begin{figure*}[t]
    \centering
    \includegraphics[width=\textwidth]
    {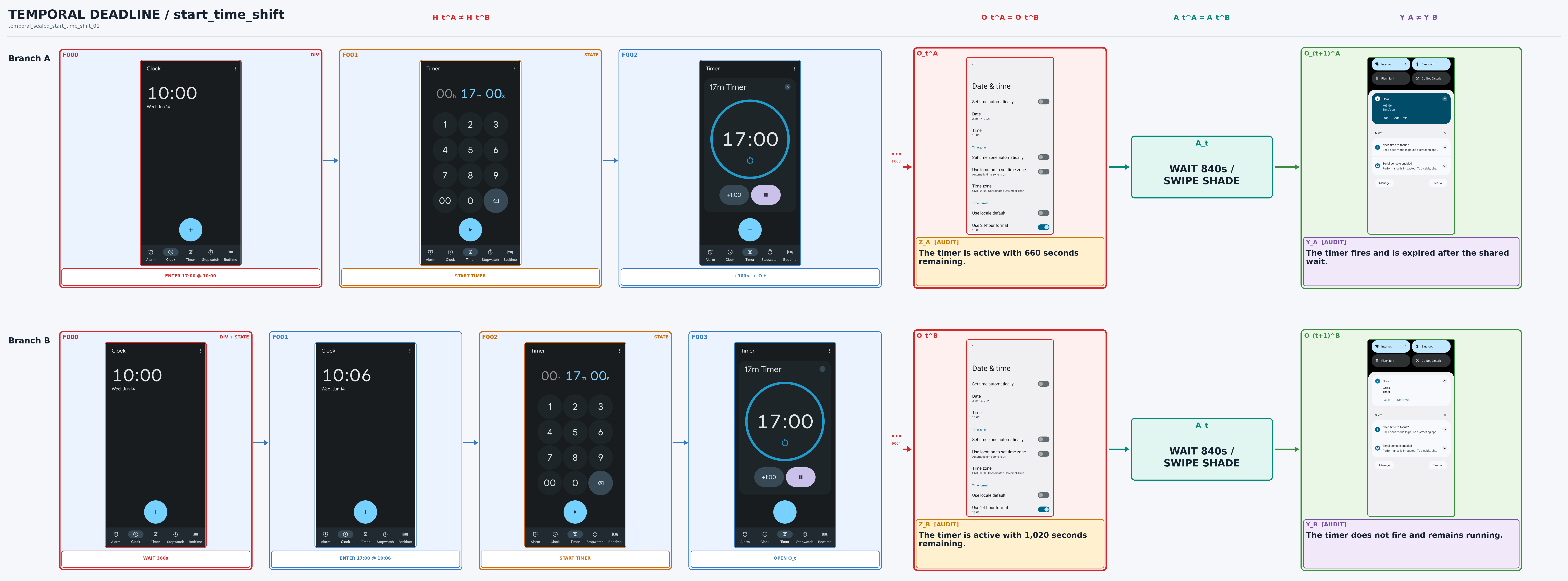}
    \caption{
    \textbf{Temporal Deadline: \texttt{start\_time\_shift}.}
    Representative strict pair
    \texttt{temporal\_sealed\_start\_time\_shift\_01}. The histories
    establish different effective remaining times while eventually producing
    the same prediction-boundary GUI $O_t$. The same controlled wait-and-probe
    action $A_t$ causes the timer to fire in one branch but not the other.
    Simulator-side logical-time state is shown only through audit annotations
    and is not supplied to the predictive-state estimator.
    }
    \label{fig:traj_temporal}
\end{figure*}

\FloatBarrier
\section{Predictive-State Recovery: Training and Implementation Details}
\label{app:state_recovery}

This appendix provides the complete training and implementation details for
predictive-state recovery. We describe the family-specific training data,
state schemas and multimodal serialization, specialist teachers, unified
multi-teacher distillation, direct joint-training controls, and inference
cost. All specialist, Joint CE, and distilled unified estimators are trained
from the same underlying family-specific training pools where applicable;
their differences therefore arise from model specialization and supervision
rather than access to additional interaction trajectories.

\subsection{Training Data and Split Construction}
\label{app:state_training_data}

We train the predictive-state estimator on family-specific interaction
histories paired with canonical structured-state targets. Across all families,
each example follows the common prediction interface
\[
(H_t,q_f)
\longmapsto
\widetilde Z_t^f,
\]
where $H_t$ is the ordered GUI interaction history, $q_f$ specifies the state
schema for family $f$, and $\widetilde Z_t^f$ is the corresponding
structured-state supervision.

The released training set contains 2,538 examples in total, comprising
898 Persistent Artifact, 160 Clipboard, 1,080 Ordered Collection, and
400 Temporal Deadline examples. The corresponding validation sets contain
110, 20, 96, and 48 examples, respectively.
Table~\ref{tab:state_training_data} summarizes the construction of each
family-specific dataset.

\begin{table*}[t]
\centering
\small
\caption{
Family-specific training data for predictive-state recovery. The specialist
teachers and unified estimators are trained from the same underlying
family-specific datasets.
}
\label{tab:state_training_data}

\begin{tabular}{
p{0.16\textwidth}
r
r
p{0.55\textwidth}}
\toprule
Family & Train & Val. & Construction \\
\midrule

Persistent Artifact
& 898
& 110
& Interaction histories centered on a single tracked artifact. Each example
contains at most one state transition affecting the tracked artifact, while
operations involving unrelated files or paths may appear as distractors.
The target state specifies whether the artifact exists at its source and
destination locations.
\\

Clipboard
& 160
& 20
& Atomic clipboard histories containing a single effective copy event.
The examples isolate persistence of clipboard content across subsequent GUI
navigation and provide the current clipboard payload as the structured-state
target.
\\

Ordered Collection
& 1,080
& 96
& Order-sensitive interaction histories covering playlist and playing-queue
construction, insertion, repositioning, reset, and interleaved collection
operations. The target state is the complete effective item ordering at the
final observation.
\\

Temporal Deadline
& 400
& 48
& Independently collected timer histories spanning eight atomic temporal
transition types. The trajectories vary timer creation, elapsed time,
pause/resume behavior, reset, deletion, and restart operations. The target
state specifies whether the timer is active and its effective remaining time
at the final observation.
\\

\midrule
Total
& 2,538
& 274
& --- \\
\bottomrule
\end{tabular}
\end{table*}

\paragraph{Atomic supervision and compositional evaluation.}
The training data are designed primarily to teach the state update semantics
needed to recover each family-specific structured state. Persistent Artifact
examples isolate the state of one tracked artifact, Clipboard examples contain
a single effective copy event, and Temporal Deadline training spans atomic
timer-transition types. Ordered Collection requires longer order-sensitive
constructions, but supervision is still defined by the final canonical
collection state.

The strict StateAliasBench evaluation is deliberately not restricted to these
simplest training patterns. Its causal templates include compositions of
previously observed update primitives, interfering operations, hidden suffixes,
resets, and branch-specific histories that subsequently converge to a common
prediction boundary; see Appendix~\ref{app:template_definitions}. Thus,
evaluation tests whether the learned update semantics can be composed to
recover the current state from complete histories rather than merely
recognizing duplicated training trajectories.

\paragraph{Common training data across estimators.}
The four family-specific specialists are trained on their corresponding
datasets in Table~\ref{tab:state_training_data}. The unified Joint CE and
multi-teacher-distillation estimators are trained on the union of the same
four training sets. Consequently, comparisons among the specialists, direct
joint training, and multi-teacher distillation control for the underlying
training examples; the distilled estimator differs in its supervision
objective rather than in access to additional interaction trajectories.

\paragraph{Family-balanced optimization.}
The four datasets differ substantially in size. To prevent the larger
families from dominating unified training, we retain every training example
and balance families through loss weighting rather than resampling. Let $N_f$
denote the number of training examples from family $f$, and let
\[
N=\sum_f N_f=2538.
\]
The loss associated with an example from family $f$ is weighted by
\[
w_f=\frac{N}{4N_f}.
\]
This yields
\[
\begin{aligned}
w_{\mathrm{persistent}} &\approx 0.7066, &
w_{\mathrm{clipboard}} &= 3.9656,\\
w_{\mathrm{ordered}} &= 0.5875, &
w_{\mathrm{temporal}} &= 1.5863.
\end{aligned}
\]
Since
\[
N_f w_f=\frac{N}{4}
\]
for every family, each family contributes equal aggregate weight to the
training objective despite unequal dataset sizes. The same family-balancing
scheme is used for the direct Joint CE and distilled unified estimators.

\paragraph{Separation from StateAliasBench evaluation.}
The predictive-state training data are distinct from the strict-pair
evaluation examples in StateAliasBench. As detailed in
Appendix~\ref{app:train_eval_separation}, the two splits have zero intersection
in exact record identities, trajectory identities, semantic target identities,
canonical content fingerprints, and complete normalized history-trajectory
signatures. Low-level GUI states, actions, and state-update primitives may
recur compositionally across datasets, as expected for an interactive GUI
environment, while complete training and evaluation examples remain
disjoint.

\subsection{Input Representation and State Schemas}
\label{app:state_input_representation}

The predictive-state estimator takes an ordered multimodal interaction history
together with a family-specific schema instruction. Consistent with
Section~\ref{sec:learning_predictive_state}, for family $f$ it implements
\[
\widehat Z_t^f
=
\operatorname{Decode}\,
G_{\phi}(H_t,q_f),
\]
where
\[
H_t=(O_0,A_0,O_1,\ldots,A_{t-1},O_t)
\]
contains all model-visible observations and historical actions up to the
prediction boundary. The future shared action $A_t$ is excluded from the
state-estimator input.

Across all families, the interaction history is serialized in temporal order
by interleaving GUI observations and textual action representations. The final
observation $O_t$ is included without a subsequent action. During training,
the canonical structured state $\widetilde Z_t^f$ appears only as the
assistant-side autoregressive target; during inference, this target is removed
and the estimator generates the state autoregressively.

Table~\ref{tab:state_schemas} summarizes the family-specific state interfaces.

\begin{table*}[t]
\centering
\small
\caption{
Family-specific predictive-state schemas. The same unified estimator is
conditioned by $q_f$ to emit the corresponding canonical structured state.
}
\label{tab:state_schemas}

\begin{tabular}{
p{0.17\textwidth}
p{0.37\textwidth}
p{0.36\textwidth}}
\toprule
Family & Schema instruction & Canonical state target \\
\midrule

Persistent Artifact
&
Predict the current persistent state of the specified artifact from the
ordered Files interaction history while ignoring unrelated artifacts and
paths.
&
\texttt{\{"source\_exists": bool,}
\newline
\texttt{"destination\_exists": bool\}}
\\

Clipboard
&
Infer the current hidden clipboard state from the complete ordered GUI
history and actions.
&
\texttt{\{"clipboard\_type": "text",}
\newline
\texttt{"clipboard\_text": string\}}
\\

Ordered Collection
&
Infer the current ordered contents of the specified playlist or playing
queue while preserving the exact effective item order.
&
\texttt{\{"collection\_type": string,}
\newline
\texttt{"ordered\_items": [string, ...]\}}
\\

Temporal Deadline
&
Infer the current timer state from the ordered visible history, including
interaction-visible temporal evidence.
&
\texttt{\{"timer\_active": bool,}
\newline
\texttt{"remaining\_seconds": int\}}
\\

\bottomrule
\end{tabular}
\end{table*}

\paragraph{Multimodal history serialization.}
Persistent Artifact and Temporal Deadline serialize the history as an initial
task specification followed by alternating GUI images and textual action
records. Clipboard and Ordered Collection additionally use explicit frame and
action indices to make temporal ordering unambiguous. In all cases, the final
GUI observation is the prediction boundary from which the structured state
must be recovered.

Actions are represented as textual JSON-like records containing the
model-visible execution arguments relevant to the corresponding interaction,
such as coordinates, key codes, entered text, or elapsed duration. They are
provided as part of the language context rather than through a separate action
encoder.

\paragraph{Visual preprocessing.}
Each family uses the image preprocessing associated with its released training
data. Persistent Artifact crops the application region from the native Files
screenshot and resizes it to $256\times512$. Clipboard and Temporal Deadline
use $256\times512$ RGB inputs, while Ordered Collection uses
$224\times352$ RGB inputs. The same preprocessing is used consistently within
each family's training and evaluation pipeline.

\paragraph{Hard supervision and loss masking.}
The ground-truth canonical state $\widetilde Z_t^f$ provides the
\emph{hard supervision target}. The training objective is applied only to its
assistant-side serialized JSON tokens. Let $\mathcal T_i$ denote the
supervised token positions for example $i$. Prompt, image, history-action,
and padding tokens are masked from the language-modeling loss:
\[
\mathcal L_{\mathrm{CE}}^{(i)}
=
-\frac{1}{|\mathcal T_i|}
\sum_{t\in\mathcal T_i}
\log
p_{\phi}
\!\left(
y_{i,t}
\mid
H_i,q_{f_i},y_{i,<t}
\right).
\]
At inference time, the assistant target is completely removed and a generation
prompt is appended after the model-visible history. The family specialists
introduced below provide an additional \emph{soft} teacher distribution for
the unified estimator; they do not replace the hard structured-state target.

\subsection{Family-Specific Specialist Teachers}
\label{app:specialist_teachers}

We first train one predictive-state specialist $T_f$ for each state family.
All specialists are parameter-efficient adaptations of the same 2B Qwen3.5
base model and optimize the canonical state-generation objective defined in
Appendix~\ref{app:state_input_representation}. The base model is kept fixed
and only low-rank attention adapters are optimized.

Each specialist uses LoRA rank $r=16$, scaling factor $\alpha=32$, and
dropout $0.05$. LoRA modules are applied to the attention projections of the
language backbone, including the $q$, $k$, $v$, and output projections of
standard self-attention and the corresponding projections of
linear-attention blocks. Vision and MLP parameters remain frozen.

We optimize all specialists with AdamW using
\[
\eta=2\times10^{-4},
\qquad
(\beta_1,\beta_2)=(0.9,0.95),
\qquad
\epsilon=10^{-8},
\]
weight decay $0.01$, cosine learning-rate decay, and gradient clipping at
norm $1.0$. Training uses bfloat16 arithmetic and an effective batch size of
eight examples.

Table~\ref{tab:specialist_training} reports the family-specific training
schedule.

\begin{table}[t]
\centering
\small
\caption{
Training schedules of the four family-specific specialist teachers.
All models use rank-16 LoRA and an effective batch size of eight.
}
\label{tab:specialist_training}

\begin{tabular}{lrrr}
\toprule
Family & Train & Epochs & Optim. steps \\
\midrule
Persistent Artifact & 898  & 4 & 452 \\
Clipboard            & 160  & 4 & 80  \\
Ordered Collection   & 1080 & 2 & 270 \\
Temporal Deadline    & 400  & 4 & 200 \\
\bottomrule
\end{tabular}
\end{table}

Temporal Deadline additionally balances its eight atomic transition types
within each effective batch, ensuring that no temporal transition type
dominates specialist training. The resulting four specialists serve both as
independent predictive-state baselines and as teachers for the unified
estimator described next.

\subsection{Unified Multi-Teacher Distillation}
\label{app:mtd_training}

Deploying one independently fine-tuned estimator per state family would
require family-specific model management. We instead distill the four
specialists into a single unified estimator that shares one parameter set
across all state schemas.

The unified student is initialized from the same base model with a fresh LoRA
adapter of rank $r=32$, scaling factor $\alpha=64$, and dropout $0.05$. For a
training example $i$ from family $f_i$, the student receives two distinct
forms of supervision:

\begin{enumerate}
    \item the ground-truth canonical state $\widetilde Z_t^{f_i}$ provides the
    hard autoregressive CE target; and
    \item the corresponding specialist $T_{f_i}$ provides the soft token-level
    KD target.
\end{enumerate}

No teacher replaces or modifies the hard ground-truth target.

\paragraph{Hard-label objective.}
Let $\mathcal T_i$ denote the supervised assistant-token positions and let
$x_i=(H_i,q_{f_i})$. The hard-label objective is
\[
\mathcal L_{\mathrm{CE}}^{(i)}
=
-\frac{1}{|\mathcal T_i|}
\sum_{t\in\mathcal T_i}
\log
p_S(y_{i,t}\mid x_i,y_{i,<t}),
\]
where $p_S$ denotes the unified student's token distribution.

\paragraph{Family-specific soft target.}
For an example from family $f_i$, let
$z_{T_{f_i},t}(v)$ denote the corresponding specialist's logit for token $v$
at supervised position $t$ under teacher forcing. The temperature-scaled
teacher distribution is
\[
p_{T_{f_i},t}^{(\tau)}(v)
=
\frac{
\exp\!\left(z_{T_{f_i},t}(v)/\tau\right)
}{
\sum_{u\in\mathcal V}
\exp\!\left(z_{T_{f_i},t}(u)/\tau\right)
},
\qquad
\tau=2.
\]

We retain the teacher's top-$K$ support
\[
\mathcal K_{i,t}
=
\operatorname{TopK}
\!\left(
p_{T_{f_i},t}^{(\tau)},K
\right),
\qquad
K=64,
\]
and renormalize the retained probabilities:
\[
\widetilde p_{T_{f_i},t}^{\tau,K}(v)
=
\frac{
p_{T_{f_i},t}^{(\tau)}(v)
}{
\sum_{u\in\mathcal K_{i,t}}
p_{T_{f_i},t}^{(\tau)}(u)
},
\qquad
v\in\mathcal K_{i,t}.
\]

The corresponding distillation loss is
\[
\mathcal L_{\mathrm{KD}}^{(i)}
=
-\frac{\tau^2}{|\mathcal T_i|}
\sum_{t\in\mathcal T_i}
\sum_{v\in\mathcal K_{i,t}}
\widetilde p_{T_{f_i},t}^{\tau,K}(v)
\log
p_{S,t}^{(\tau)}(v).
\]

\paragraph{Combined objective.}
Combining hard supervision, family-specific teacher supervision, and the
family-balancing weight from
Appendix~\ref{app:state_training_data}, the released per-example objective is
\[
\boxed{
\mathcal L_{\mathrm{MTD}}^{(i)}
=
w_{f_i}
\left[
(1-\lambda)\mathcal L_{\mathrm{CE}}^{(i)}
+
\lambda\mathcal L_{\mathrm{KD}}^{(i)}
\right]
}
\]
with
\[
\lambda=0.25,
\qquad
\tau=2,
\qquad
K=64.
\]
Thus,
\[
\mathcal L_{\mathrm{MTD}}^{(i)}
=
w_{f_i}
\left[
0.75\mathcal L_{\mathrm{CE}}^{(i)}
+
0.25\mathcal L_{\mathrm{KD}}^{(i)}
\right].
\]

\paragraph{Family-specific teacher routing.}
Each training example is distilled only from the specialist associated with
its own state family:
\[
T_i=T_{f_i}.
\]
There is no cross-family averaging or teacher ensemble. The family/schema
instruction $q_f$ specifies the desired output state representation, while the
routed specialist provides family-specific token-level soft supervision.

\paragraph{Offline teacher targets.}
Teacher distributions are computed once under teacher forcing and cached
before student training. At each supervised position, the cache stores the
top-64 token identifiers and their temperature-scaled probabilities. During
student optimization, the retained teacher probabilities are renormalized over
this support. Consequently, the four teacher models are not loaded during
student optimization and are not required at inference time.

\paragraph{Unified optimization.}
The student is trained for four epochs using AdamW with learning rate
$2\times10^{-4}$, cosine decay, $5\%$ warmup, weight decay $0.01$, gradient
clipping at $1.0$, bfloat16 arithmetic, and gradient checkpointing. Training
uses all 2,538 examples once per epoch with the family-balanced objective
above. The final deployed predictive-state estimator consists of a single
shared base model and one rank-32 LoRA adapter supporting all four state
schemas.

\subsection{Comparison with Direct Joint Training}
\label{app:mtd_ablation}

We compare multi-teacher distillation against two natural alternatives:
independently deployed family specialists and direct joint training of a
single unified estimator. To isolate the effect of specialist teacher
supervision, the direct Joint CE baseline uses the same four training datasets
and the same family-balancing weights as MTD, but optimizes only the hard
ground-truth objective,
\[
\mathcal L_{\mathrm{Joint}}^{(i)}
=
w_{f_i}\mathcal L_{\mathrm{CE}}^{(i)}.
\]
Thus, Joint CE and MTD differ in teacher supervision rather than access to
training examples.

\paragraph{Structured-state recovery.}
We first compare the estimators directly at the structured-state level.
Table~\ref{tab:mtd_ablation} reports BranchExact and PairCorrect under the
strict evaluation protocol. Here, branch correctness requires exact canonical
recovery of the branch-specific $\widetilde Z_t^f$, and PairCorrect requires
both members of an A/B pair to be recovered correctly.

\begin{table*}[t]
\centering
\small
\caption{
Predictive-state recovery under family-specific specialists, direct
family-balanced joint training, and multi-teacher distillation.
Each family cell reports BranchExact / PairCorrect (\%). Avg. and Worst are
computed from PairCorrect across the four state families. Best results are in
\textbf{bold}; second-best results are underlined.
}
\label{tab:mtd_ablation}

\begin{tabular}{lcccccc}
\toprule
Method
& Persistent
& Clipboard
& Ordered
& Temporal
& Avg.
& Worst \\
\midrule

Specialists
& \underline{91.75} / 85.5
& \textbf{94.75} / \textbf{89.5}
& \textbf{92.75} / \underline{88.0}
& \textbf{92.50} / \textbf{88.5}
& \textbf{87.88}
& \textbf{85.5}
\\

Family-balanced Joint CE
& \textbf{96.00} / \textbf{92.0}
& 74.75 / 69.5
& \textbf{92.75} / \textbf{89.0}
& 85.00 / 84.0
& 83.63
& 69.5
\\

MTD (Ours)
& 89.75 / \underline{86.0}
& \underline{89.00} / \underline{85.0}
& 87.25 / 82.0
& \underline{88.50} / \underline{86.5}
& \underline{84.88}
& \underline{82.0}
\\

\bottomrule
\end{tabular}
\end{table*}

Direct joint training performs strongly on several families but exhibits a
substantial cross-family imbalance: Clipboard PairCorrect falls to $69.5\%$,
even though Persistent Artifact reaches $92.0\%$. Multi-teacher distillation
raises the worst-family PairCorrect from $69.5\%$ to $82.0\%$, while its
macro average increases from $83.63\%$ to $84.88\%$. Independently trained
specialists remain stronger on average at $87.88\%$, reflecting the expected
trade-off between family-specific specialization and a single shared model.

\paragraph{Matched end-to-end world-model evaluation.}
Structured-state recovery alone does not establish whether differences between
Joint CE and MTD propagate to the final world-model prediction. We therefore
perform an end-to-end matched comparison using frozen Code2World-8B. Both
estimators are evaluated on the same StateAliasBench strict pairs and share
the same current observations, candidate actions, prediction-time context,
complete deterministic state-conditioning interface, world-model checkpoint,
prompt protocol, renderer, scorer, and generation configuration. The only
upstream experimental variable is the recovered-state estimator.

Table~\ref{tab:mtd_final_wm} reports the resulting Code2World-8B predictions.
MTD increases macro PairCorrect from $76.75\%$ to $79.38\%$ and
worst-family PairCorrect from $69.50\%$ to $71.50\%$. The improvement is not
uniform across families: MTD substantially improves Clipboard and modestly
improves Persistent Artifact and Temporal Deadline, while Joint CE remains
stronger on Ordered Collection.

\begin{table*}[t]
\centering
\small
\caption{
Matched end-to-end Code2World-8B evaluation of the two unified
predictive-state estimators. Each family cell reports
BranchExact / PairCorrect (\%). Avg. and Worst are computed from final-WM
PairCorrect. Both methods use the same complete state-conditioning interface
and frozen world model. Best results are in \textbf{bold}; second-best
results are underlined.
}
\label{tab:mtd_final_wm}

\begin{tabular}{lcccccc}
\toprule
Method
& Persistent
& Clipboard
& Ordered
& Temporal
& Avg.
& Worst \\
\midrule

Family-balanced Joint CE
& \textbf{89.50} / \underline{79.50}
& \underline{74.75} / \underline{69.50}
& \textbf{87.25} / \textbf{79.00}
& \underline{82.50} / \underline{79.00}
& \underline{76.75}
& \underline{69.50}
\\

MTD (Ours)
& \underline{86.75} / \textbf{80.00}
& \textbf{89.00} / \textbf{85.00}
& \underline{81.50} / \underline{71.50}
& \textbf{85.75} / \textbf{81.00}
& \textbf{79.38}
& \textbf{71.50}
\\

\bottomrule
\end{tabular}
\end{table*}

\paragraph{State-to-world-model attribution.}
The matched evaluation further shows that differences in final-WM performance
are closely tied to structured-state recovery. Among the 621 pairs for which
both estimators recover the state correctly, Joint CE and MTD obtain identical
final PairCorrect of $91.63\%$. On the 48 pairs recovered correctly only by
Joint CE, final PairCorrect is $93.75\%$ for Joint CE and $22.92\%$ for MTD.
Conversely, on the 58 pairs recovered correctly only by MTD, final PairCorrect
is $0\%$ for Joint CE and $94.83\%$ for MTD. Both methods obtain $0\%$ on
the 73 pairs for which both recovered states are incorrect.

These results show that the effect of the recovery objective propagates
through the shared state-conditioning interface to final world-model
prediction. MTD improves the aggregate and worst-family performance of the
unified Code2World pipeline, but does not uniformly dominate direct joint
training on every state family.

\paragraph{Specialization versus unified deployment.}
The family-specific specialists provide the strongest average structured-state
recovery and serve as a family-specific capacity reference. They are not,
however, required by the unified MTD estimator at inference time. MTD instead
consolidates their family-specific supervision into a single estimator while
reducing the cross-family degradation observed under direct joint
optimization.

The two evaluations above therefore address complementary questions:
specialist performance measures what can be achieved with independently
adapted family-specific estimators, whereas Joint CE versus MTD isolates the
effect of specialist soft supervision within a single shared model. The
matched Code2World-8B evaluation further confirms that these recovery
differences propagate to final world-model behavior rather than remaining
confined to the intermediate structured-state metric.

\subsection{Computational Cost}
\label{app:state_efficiency}

We additionally measure the inference overhead introduced by predictive-state
recovery in the Code2World pipeline. The profiling set contains 100 balanced
StateAliasBench transitions, with 25 examples from each state family.

\paragraph{World-model prompt overhead.}
Conditioning Code2World on the recovered state increases the served
world-model prompt length from an average of $3243.33$ to $3427.26$ prompt
tokens, corresponding to an average marginal increase of $183.93$ tokens:
\[
\frac{183.93}{3243.33}\times100
=
5.67\%.
\]
This quantity measures only the marginal prompt-token cost of inserting the
structured-state block into the Code2World world-model request. It does not
include the separate multimodal input consumed by the predictive-state
estimator itself.

\paragraph{State-recovery latency.}
On an NVIDIA RTX PRO 5000 72GB GPU with batch size one, the unified
predictive-state estimator requires
\[
0.768~\mathrm{s}
\]
mean generation-and-decoding time per transition. The median is
$0.589$\,s and the $95$th percentile is $1.815$\,s. The measured interval
covers autoregressive generation, CUDA completion, and output decoding after
the estimator input has been prepared. Image loading, image preprocessing,
prompt construction/tokenization, host-to-device transfer, JSON parsing, and
model initialization are excluded.

\paragraph{Deployment footprint.}
The final unified estimator uses one rank-32 LoRA adapter containing
approximately $12.39$M trainable adapter parameters. In contrast, four
family-specific rank-16 specialists contain approximately
$4\times6.19$M adapter parameters. More importantly, the unified estimator
requires only one active predictive-state model at inference time and directly
supports all four state schemas without an ensemble of specialist teachers.
\section{State-Conditioned World-Model Interface}
\label{app:state_wm_interface}

The main text represents the state-conditioning interface abstractly as
\[
C_t^f=\Gamma_f(\widehat Z_t^f,A_t),
\]
while suppressing prediction-time task context and WM-specific serialization.
The interface is agnostic to which learned estimator supplies
$\widehat Z_t^f$: the same deterministic construction consumes either a
family-specific specialist prediction or the unified MTD prediction.

We distinguish the recovered current state from the deterministic consequence
of applying the candidate action to that state. Let
\begin{equation}
D_t^f
=
g_f(\widehat Z_t^f,A_t,c_t),
\label{eq:resolved_consequence}
\end{equation}
where $c_t$ denotes prediction-time task/entity context available
independently of the branch-specific hidden-state value. The complete
model-compatible conditioning is
\begin{equation}
C_t^{f,m}
=
\Gamma_{f,m}
\!\left(
\widehat Z_t^f,A_t,c_t
\right),
\label{eq:appendix_state_interface}
\end{equation}
where $m$ identifies the target WM. The mapping $\Gamma_{f,m}$ combines the
recovered state, its action-resolved consequence, a family-specific target
surface, prediction-time context projection, and WM-native serialization.

Neither $g_f$ nor $\Gamma_{f,m}$ invokes an additional learned model or
accesses the future observation, semantic target, ground-truth hidden state,
backend/ADB state, or evaluation metadata. This appendix specifies these
deterministic mappings, their information provenance, their model-native
serialization, and the staged experiments used to separate state recovery
from state-to-WM realization. We additionally provide the exact request
structure used by the decisive interface experiments and decompose their
reported scores into semantic and surface-validity components.

\subsection{Deterministic State Conditioning}
\label{app:deterministic_state_conditioning}

The evaluated world models are frozen and were not trained to directly consume
the structured state schemas introduced in
Section~\ref{sec:predictive_state}. We therefore compile recovered state into
conditioning compatible with each WM's native prediction interface.

The auxiliary context $c_t$ contains only prediction-time information needed
to interpret the recovered state and candidate action. For Persistent
Artifact, for example, it specifies the tracked artifact and the path, folder,
or search scope inspected by the candidate action. Such context is defined
independently of the branch-specific state value.

The complete deterministic interface is decomposed into five stages:
\begin{align}
\bar A_t
&=
\operatorname{NormalizeAction}_m(A_t),
\\
D_t^f
&=
g_f(\widehat Z_t^f,\bar A_t,c_t),
\\
S_t^{f,m}
&=
\operatorname{SurfaceSpec}_{f,m}
\!\left(
\widehat Z_t^f,D_t^f,\bar A_t,c_t
\right),
\\
\bar c_t^{f,m}
&=
\operatorname{ProjectContext}_{f,m}
\!\left(
c_t,\widehat Z_t^f,D_t^f
\right),
\\
C_t^{f,m}
&=
\operatorname{Serialize}_{f,m}
\!\left(
O_t,
\bar A_t,
\widehat Z_t^f,
D_t^f,
S_t^{f,m},
\bar c_t^{f,m}
\right).
\label{eq:full_interface_compiler}
\end{align}

$\operatorname{NormalizeAction}_m$ removes non-predictive bookkeeping fields,
including timestamps and execution indices, and converts action names and
coordinates to the representation expected by WM $m$.
The family rule $g_f$ computes only the deterministic semantic consequence
implied by the recovered state and candidate action.
$\operatorname{SurfaceSpec}_{f,m}$ specifies the GUI surface on which that
consequence must be realized.
$\operatorname{ProjectContext}_{f,m}$ retains only prediction-time context
relevant to that resolved surface.
Finally, $\operatorname{Serialize}_{f,m}$ expresses the resulting information
through the native multimodal request structure of the frozen WM.

Table~\ref{tab:state_interface_families} summarizes the family-level
semantics.

\begin{table*}[t]
\centering
\scriptsize
\setlength{\tabcolsep}{4pt}
\renewcommand{\arraystretch}{1.04}
\caption{
Family-specific deterministic realization of recovered predictive state.
$D_t^f$ is the action-resolved consequence; the final column describes the
complete model-compatible realization.
}
\label{tab:state_interface_families}

\begin{tabular}{
p{0.15\textwidth}
p{0.20\textwidth}
p{0.27\textwidth}
p{0.30\textwidth}}
\toprule
Family
& Recovered state
& Action-resolved consequence $D_t^f$
& Complete model-compatible realization \\
\midrule

Persistent Artifact
&
Source and destination existence of the tracked artifact.
&
Resolve whether the tracked artifact must be present or absent on the Files
surface inspected by the candidate folder or filename-search action.
&
Compile the result into a Downloads-scoped search surface or the requested
folder listing, project context onto that inspected surface, and impose the
corresponding visible-row or empty-state constraints.
\\

Clipboard
&
Current clipboard type and payload.
&
Resolve the payload inserted by the candidate PASTE action.
&
Construct a post-action editor surface containing the recovered payload as
visible editor content while preserving recognizable editor chrome.
\\

Ordered Collection
&
Collection type and complete ordered item sequence.
&
Resolve the playlist or playing-queue state exposed by the candidate action,
preserving the supplied sequence without sorting or deduplication.
&
Construct the corresponding collection surface with the required number and
order of visible primary rows and serialize the complete sequence through the
WM-compatible representation.
\\

Temporal Deadline
&
Timer activity and effective remaining time.
&
Combine timer activity and remaining time with the elapsed duration in the
candidate action to resolve \emph{none}, \emph{running}, or \emph{fired},
together with the post-action remainder when applicable.
&
Compile the resolved timer state into a full-screen notification-shade
realization with the corresponding visible timer text, countdown, and control
structure.
\\

\bottomrule
\end{tabular}
\end{table*}

\paragraph{Persistent Artifact.}
Let
\[
\widehat Z_t^{\mathrm{pers}}
=
(e_t^{\mathrm{src}},e_t^{\mathrm{dst}})
\]
contain Boolean source- and destination-existence variables for the tracked
artifact.

For a filename-search probe, the interface considers the tracked source and
destination paths whose path components place them inside the benchmark's
Downloads scope. The artifact is resolved as present when at least one
corresponding in-scope state variable is true. The required post-action
surface is a Downloads-scoped Files search. A present state requires a visible
result row containing the tracked artifact, whereas an absent state requires
the explicit empty state ``No matches in Downloads'' and no target result row.

For a folder-opening probe, the action and task context determine the exact
artifact path being inspected. The corresponding existence variable determines
whether the requested folder contains the tracked artifact. A present state
requires one target row; an absent state requires the same requested folder
surface without that row.

The model-visible entity context is then projected onto the resolved surface.
If the artifact must be visible, its identity and relevant scoped path are
retained. If it must be absent, the projected context retains the inspected
path and surface type but omits the target identity. This prevents an
absent-state request from reintroducing the filename through auxiliary
context.

\paragraph{Clipboard.}
Let $v_t$ denote the recovered clipboard payload. For the shared PASTE action,
\[
D_t^{\mathrm{clip}}
=
\text{``insert }v_t\text{ into the focused editor.''}
\]
The interface does not infer, paraphrase, or otherwise modify $v_t$. The
target surface is a recognizable note/editor GUI whose main editor body
contains the recovered payload as visible text.

\paragraph{Ordered Collection.}
Let
\[
\widehat Z_t^{\mathrm{ord}}
=
(\tau_t,[x_1,\ldots,x_n]),
\]
where $\tau_t$ specifies playlist versus playing queue. The resolved state is
\[
D_t^{\mathrm{ord}}
=
(\tau_t,[x_1,\ldots,x_n]).
\]
The deterministic interface performs no sorting, deduplication, completion, or
inference of additional items. The target collection surface contains exactly
$n$ primary rows in the supplied order. A playing queue is realized as a
compact queue view, whereas a playlist is realized as a playlist-detail view.

\paragraph{Temporal Deadline.}
Let
\[
\widehat Z_t^{\mathrm{temp}}
=
(a_t,r_t),
\]
where $a_t\in\{0,1\}$ denotes whether the timer is active and $r_t$ is its
effective remaining time in seconds. Let $\Delta_t$ denote the wait duration
specified by the candidate action. The deterministic consequence is
\[
D_t^{\mathrm{temp}}
=
\begin{cases}
(\mathrm{none},\bot),
& a_t=0,
\\
(\mathrm{fired},0),
& a_t=1\ \wedge\ \Delta_t\ge r_t,
\\
(\mathrm{running},r_t-\Delta_t),
& a_t=1\ \wedge\ \Delta_t<r_t.
\end{cases}
\]

A fired state is realized as one Clock notification containing
\emph{Time's up} and a \emph{Stop} control. A running state contains the
updated remaining time and a \emph{Pause} control. A none state contains no
timer notification; a neutral non-timer card is used in the controlled
realization so that all branches remain grounded on the same requested
notification-shade surface.

\subsection{Model-Specific Serialization}
\label{app:model_specific_serialization}

Code2World-8B and gWorld receive the same recovered state and deterministic
family-level consequence but use different native request conventions.
Table~\ref{tab:wm_serialization} summarizes these differences.

\begin{table*}[t]
\centering
\scriptsize
\setlength{\tabcolsep}{4pt}
\renewcommand{\arraystretch}{1.04}
\caption{
WM-specific realization of the same deterministic predictive-state
information.
}
\label{tab:wm_serialization}

\begin{tabular}{
p{0.20\textwidth}
p{0.36\textwidth}
p{0.36\textwidth}}
\toprule
Component
& Code2World-8B
& gWorld-8B \\
\midrule

Current observation
&
Application crop with the model-native action visual hint.
&
Same application crop without the Code2World visual hint.
\\

Action representation
&
Lowercase model-native action names; point actions use scalar coordinate
fields.
&
Uppercase model-native action names; point actions use coordinate arrays.
\\

Message structure
&
Fixed system instruction followed by a multimodal user request.
&
User-only multimodal request.
\\

Rendering convention
&
Fixed $1080\times2400$ target canvas.
&
$360\times800$ CSS viewport, captured at device scale factor $3$.
\\

Persistent Artifact
&
Resolved Files blueprint and projected path/entity context are inserted into
the Code2World-native transition request.
&
The same resolved Files semantics are expressed through the gWorld-native
mobile-UI request.
\\

Clipboard
&
The resolved payload is represented once as authoritative visible editor
content.
&
The same payload is inserted into the gWorld-native editor request.
\\

Ordered Collection
&
The canonical ordered list is the sole authoritative list representation,
avoiding a redundant natural-language copy of the item sequence.
&
The canonical list is combined with the native playlist/queue surface
constraint.
\\

Temporal Deadline
&
The resolved timer outcome is inserted directly into the User Intent and final
visible-result check of the Code2World request.
&
A dedicated short notification-shade request is used instead of the generic
gWorld request so that the resolved timer state appears only once.
\\

\bottomrule
\end{tabular}
\end{table*}

For Temporal Deadline, the two WMs therefore receive the same deterministic
outcome but not the same textual realization. A fired timer is represented by
the visible state-bearing sequence
\[
\text{Clock}
\rightarrow
\text{Time's up}
\rightarrow
\text{Stop},
\]
whereas a running timer uses
\[
\text{Clock}
\rightarrow
\text{Timer}
\rightarrow
\texttt{MM:SS}
\rightarrow
\text{Pause}.
\]
The gWorld realization additionally fixes a compact
$360\times800$ notification-shade layout and explicitly excludes the
pre-action Settings surface, duplicated cards and controls, and unrelated
system content. Code2World realizes the same semantic outcome under its
fixed-canvas generation protocol.

\subsection{Prediction-Time Information Provenance}
\label{app:state_interface_provenance}

We audit the information available to the learned-state condition at the
world-model prediction boundary. Table~\ref{tab:state_interface_provenance}
separates prediction-time inputs and deterministic interface products from
supervision, future observations, and evaluation-side information.

\begin{table*}[t]
\centering
\scriptsize
\setlength{\tabcolsep}{4pt}
\caption{
Information provenance for the learned state-conditioned WM.
}
\label{tab:state_interface_provenance}

\begin{tabular}{p{0.25\textwidth}p{0.34\textwidth}cc}
\toprule
Quantity & Source & Available at prediction time & Used \\
\midrule

Current observation $O_t$
& Current GUI observation
& Yes & Yes \\

Candidate action $A_t$
& Agent / benchmark action
& Yes & Yes \\

Family/task specification
& Fixed task protocol
& Yes & Yes \\

Task/entity context $c_t$
& Prediction-time task specification, when applicable
& Yes & Yes \\

Recovered state $\widehat Z_t^f$
& Learned state estimator
& Yes & Yes \\

Action-resolved consequence $D_t^f$
& Fixed deterministic state/action rules
& Yes & Yes \\

Surface specification $S_t^{f,m}$
& Fixed family/WM rules
& Yes & Yes \\

Projected context $\bar c_t^{f,m}$
& Deterministic projection of prediction-time context
& Yes & Yes \\

Complete conditioning $C_t^{f,m}$
& Fixed deterministic interface
& Yes & Yes \\

Ground-truth state $\widetilde Z_t^f$
& Benchmark supervision
& No & No \\

Future observation $O_{t+1}$
& Environment target
& No & No \\

Semantic outcome $Y_{t+1}$
& Evaluation annotation
& No & No \\

Backend / ADB state
& Environment-side validation
& No & No \\

Scoring metadata
& Evaluation pipeline
& No & No \\

\bottomrule
\end{tabular}
\end{table*}

All branch-dependent conditioning under the learned-state condition is
therefore deterministically derived from
\[
\widehat Z_t^f,\qquad A_t,\qquad c_t,
\]
together with fixed family- and WM-specific rules. In particular, the
interface cannot inspect the target future GUI, semantic label, or
ground-truth structured state when constructing a world-model request.

\subsection{Progressive Realization of the State-Conditioning Interface}
\label{app:state_interface_control}

We next isolate how the representation supplied to a frozen WM affects its
ability to use a recovered state. All conditions in this subsection use the
same family-specific specialist predictions, so the recovered-state source is
held fixed throughout. We evaluate progressively richer realizations of these
fixed predictions for Persistent Artifact and Temporal Deadline.

\paragraph{Canonical.}
Canonical supplies the recovered structured state together with the generic
family transition contract and the same prediction-time context used by the
matched controls. It contains no precomputed action-resolved consequence.

\paragraph{+Consequence.}
This condition adds only the deterministic action-resolved consequence
$D_t^f$ in Eq.~\ref{eq:resolved_consequence}. For Persistent Artifact this
specifies the presence or absence of the tracked artifact on the inspected
surface. For Temporal Deadline it specifies the post-wait timer lifecycle and,
when applicable, the remaining duration.

\paragraph{+Surface.}
This condition additionally realizes the consequence through one
family-specific surface component. Persistent Artifact receives a
state-dependent Files surface blueprint. Temporal Deadline receives a
branch-independent notification-shade rendering scaffold; the timer outcome
continues to be determined by $D_t^f$.

\paragraph{Full interface.}
Full applies the production $\Gamma_{f,m}$ construction to the same
specialist-recovered state. Unlike the preceding additive controls, Full uses
the complete WM-native compilation path, including context projection and
native placement of the state-bearing directive. These results are the
\textbf{Specialist State} entries in
Table~\ref{tab:conditioning_full}.

\begin{table}[t]
\centering
\scriptsize
\setlength{\tabcolsep}{3.5pt}
\renewcommand{\arraystretch}{1.0}
\caption{
Progressive realization of a fixed family-specialist recovered state,
measured by PairCorrect (\%). Canonical $\rightarrow$ +Consequence and
+Consequence $\rightarrow$ +Surface are strictly request-matched comparisons.
Full uses the production WM-native interface.
}
\label{tab:state_interface_progression}

\begin{tabular}{llrrrr}
\toprule
Family & WM
& Canon.
& +Cons.
& +Surf.
& Full \\
\midrule

Persistent
& Code2World-8B
& 1.0
& 4.0
& 26.0
& \textbf{78.5} \\

Persistent
& gWorld-8B
& 12.5
& 14.5
& 40.5
& \textbf{85.5} \\

Temporal
& Code2World-8B
& 0.0
& 0.0
& 0.5
& \textbf{84.0} \\

Temporal
& gWorld-8B
& 0.0
& 0.0
& 22.5
& \textbf{88.5} \\

\bottomrule
\end{tabular}
\end{table}

The Canonical $\rightarrow$ +Consequence comparison is strictly matched:
removing the added consequence span from the richer request recovers the
Canonical request byte-for-byte. The same property holds for
+Consequence $\rightarrow$ +Surface after removing the surface component.
These request-matching audits pass for all $400/400$ branches in each
evaluated model--family combination.

The largest increase occurs between +Surface and Full. For Persistent
Artifact, PairCorrect increases from $26.0\%$ to $78.5\%$ on Code2World-8B
and from $40.5\%$ to $85.5\%$ on gWorld-8B. For Temporal Deadline, it
increases from $0.5\%$ to $84.0\%$ and from $22.5\%$ to $88.5\%$,
respectively.

This final transition is not a single-component ablation. Full replaces the
additive representation with the production WM-native compilation path. For
Persistent Artifact this jointly changes surface integration,
state-dependent context projection, and placement of the resolved Files
directive. For Temporal Deadline the same deterministic timer consequence is
compiled directly into the model-specific final notification request rather
than appended to a generic transition request. We therefore use the staged
experiment to characterize the level of compilation required for a fixed
recovered state to become usable by the frozen WM, rather than to assign the
Full--Surface difference to a single internal component.

\subsubsection{Semantic and Surface Decomposition of the Interface Gain}
\label{app:interface_score_decomposition}

The official strict scorer evaluates both the family-specific semantic outcome
and a small set of surface-validity conditions required to ensure that the
outcome is realized on the requested GUI surface. To separate these two
effects, we decompose the already archived scores without rerunning the world
models, renderer, or scorer.

For each branch, define
\[
\operatorname{SemCorrect}
=
\mathbf 1
\!\left[
\widehat Y_{t+1}^f=Y_{t+1}^f
\right].
\]
For Persistent Artifact, this is only equality between the parsed
present/absent state and the expected state; Downloads scope, folder scope,
keyboard structure, duplicate-listing checks, and offline validity are
excluded from this semantic indicator.

For Temporal Deadline,
$\operatorname{SemCorrect}$ is equality between the stored parser outcome
$\{\mathrm{none},\mathrm{running},\mathrm{fired},
\mathrm{invalid}\}$ and the expected timer outcome. This definition uses the
archived parser result rather than reparsing outputs under a modified rule.

We separately define $\operatorname{SurfaceValid}$ from the non-outcome gates
used by the official scorer. For Persistent search results these include the
Downloads-scoped search surface, plausible keyboard geometry, and offline
self-containedness; for folder probes they include the required location
grounding, unique listing structure, absence of a contradictory software
keyboard, and offline self-containedness. For Temporal Deadline they include
notification-shade grounding, footer completeness, and offline
self-containedness.

Under these definitions, all $6{,}400$ archived branch scores in the staged
experiment satisfy
\[
\operatorname{FullCorrect}
=
\operatorname{SemCorrect}
\wedge
\operatorname{SurfaceValid}.
\]

Table~\ref{tab:interface_score_components} reports the resulting decomposition.
``Sem. B'' and ``Surf. B'' denote branch-level semantic correctness and
surface validity, respectively; ``Full B'' is the unchanged official
BranchExact result. ``Sem. P'' requires semantic correctness on both branches
of a strict pair, whereas ``Full P'' is the unchanged official PairCorrect.

\begin{table*}[t]
\centering
\scriptsize
\setlength{\tabcolsep}{4pt}
\renewcommand{\arraystretch}{1.02}
\caption{
Decomposition of the archived +Surface and Full results into semantic
correctness and surface validity. All values are percentages. No prediction
is regenerated and no scoring rule is modified.
}
\label{tab:interface_score_components}

\begin{tabular}{llrrrrrr}
\toprule
Family & WM & Condition
& Sem. B
& Surf. B
& Full B
& Sem. P
& Full P \\
\midrule

Persistent
& Code2World-8B
& +Surface
& 60.5 & 59.0 & 42.8 & 27.5 & 26.0 \\

Persistent
& Code2World-8B
& Full
& \textbf{94.0} & \textbf{92.8} & \textbf{88.2}
& \textbf{89.5} & \textbf{78.5} \\

\addlinespace[2pt]

Persistent
& gWorld-8B
& +Surface
& 70.2 & 92.0 & 64.8 & 41.5 & 40.5 \\

Persistent
& gWorld-8B
& Full
& \textbf{93.0} & \textbf{98.8} & \textbf{92.0}
& \textbf{87.5} & \textbf{85.5} \\

\addlinespace[2pt]

Temporal
& Code2World-8B
& +Surface
& 12.5 & 58.0 & 10.5 & 1.0 & 0.5 \\

Temporal
& Code2World-8B
& Full
& \textbf{91.0} & \textbf{98.2} & \textbf{90.0}
& \textbf{86.0} & \textbf{84.0} \\

\addlinespace[2pt]

Temporal
& gWorld-8B
& +Surface
& 46.5 & 94.5 & 44.8 & 25.0 & 22.5 \\

Temporal
& gWorld-8B
& Full
& \textbf{92.5} & \textbf{100.0} & \textbf{92.5}
& \textbf{88.5} & \textbf{88.5} \\

\bottomrule
\end{tabular}
\end{table*}

To further characterize the +Surface $\rightarrow$ Full transition, we join
the two conditions by branch identifier. A \emph{FullCorrect gain} is a branch
that is incorrect under +Surface and correct under Full. We partition such
gains into three mutually exclusive cases:
\begin{enumerate}
    \item \textbf{semantic-only}: semantic correctness changes from false to
    true while surface validity was already true;
    \item \textbf{surface-only}: semantic correctness was already true and
    surface validity changes from false to true; and
    \item \textbf{both}: both semantic correctness and surface validity change
    from false to true.
\end{enumerate}

\begin{table*}[t]
\centering
\scriptsize
\setlength{\tabcolsep}{4pt}
\renewcommand{\arraystretch}{1.02}
\caption{
Branch-level attribution of the +Surface $\rightarrow$ Full transition.
``Gains'' counts branches changing from official incorrect to correct;
``Losses'' counts the reverse. The final three columns partition the gains.
Each model--family condition contains 400 branches.
}
\label{tab:interface_gain_attribution}

\begin{tabular}{llrrrrrr}
\toprule
Family & WM
& Gains
& Losses
& Net
& Semantic-only
& Surface-only
& Both \\
\midrule

Persistent
& Code2World-8B
& 200 & 18 & +182
& 60 & 65 & 75 \\

Persistent
& gWorld-8B
& 117 & 8 & +109
& 88 & 20 & 9 \\

Temporal
& Code2World-8B
& 319 & 1 & +318
& 167 & 8 & 144 \\

Temporal
& gWorld-8B
& 193 & 2 & +191
& 173 & 7 & 13 \\

\bottomrule
\end{tabular}
\end{table*}

The decomposition shows that the Full-interface improvement is not reducible
to satisfying surface-validity gates. For Temporal Deadline,
$311/319$ Code2World gains and $186/193$ gWorld gains involve a correction of
the parsed timer outcome; only $8$ and $7$ gains, respectively, are
surface-only. Persistent Artifact shows the same pattern on gWorld, where
$97/117$ gains involve semantic correction. Code2World Persistent is more
mixed: $135/200$ gains involve semantic correction and $65/200$ are
surface-only. Thus, model-native compilation improves both the semantic use of
the supplied recovered state and the reliability with which the corresponding
GUI surface is rendered.

This decomposition follows the existing scorer exactly and is not a
counterfactual rescoring experiment. In particular, a small number of
duplicate or internally contradictory timer-layout patterns participate
upstream in the existing timer-state parser; we retain the archived
$\widehat Y_{t+1}^{\mathrm{temp}}$ rather than redefining the timer parser
post hoc. Detailed semantic-readout rules are given in
Appendix~\ref{app:semantic_validation}.

The staged analysis addresses a specific question: given the same recovered
state, how must that state be realized for a frozen GUI WM to use it
reliably? It is therefore complementary to the Observation and Raw History
conditions in the main experiment, rather than an attempt to attribute the
entire end-to-end gain to one isolated prompt component.

\subsection{Full-Interface Construction and Reproducibility}
\label{app:full_interface_reproducibility}

We now specify the Full request at the model-facing boundary. For family $f$
and WM $m$, write
\[
\mathcal R_{f,m}
=
\left[
O_t,\,
\bar A_t,\,
Q_f,\,
\bar c_t^{f,m},\,
\mathcal D_{f,m}
(\widehat Z_t^f,D_t^f,S_t^{f,m})
\right],
\]
where $Q_f$ denotes the fixed family transition contract and
$\mathcal D_{f,m}$ denotes the model-native realization of the recovered
state, resolved consequence, and target surface.

Table~\ref{tab:full_interface_rules} gives the deterministic family-level
compilation rules.

\begin{table*}[t]
\centering
\scriptsize
\setlength{\tabcolsep}{4pt}
\renewcommand{\arraystretch}{1.04}
\caption{
Deterministic rules used by the Full state-conditioning interface.
}
\label{tab:full_interface_rules}

\begin{tabular}{
p{0.14\textwidth}
p{0.25\textwidth}
p{0.28\textwidth}
p{0.27\textwidth}}
\toprule
Family
& State/action resolution
& Required visible surface
& Context/serialization rule \\
\midrule

Persistent Artifact
&
For a Downloads search, resolve presence from tracked in-scope source and
destination paths. For a folder probe, resolve presence at the exact inspected
artifact path.
&
Present search: Downloads-scoped heading and one target result row.
Absent search: explicit ``No matches in Downloads'' and no result row.
Present folder: requested folder and one target row.
Absent folder: requested folder with no tracked-artifact row.
&
Visible branches retain artifact identity and the relevant scoped path.
Absent branches retain the inspected path and surface mode but omit target
identity. The resulting Files blueprint is inserted through the WM-native
transition representation.
\\

Clipboard
&
PASTE deterministically inserts the recovered clipboard payload.
&
A recognizable editor containing the recovered payload as visible editor-body
text.
&
The payload is serialized once as authoritative post-action content; no
additional state is inferred from the screenshot.
\\

Ordered Collection
&
Preserve the recovered collection type and ordered item sequence exactly.
&
A playlist-detail or playing-queue surface containing exactly the supplied
number of primary rows in the supplied order.
&
The canonical list remains the authoritative sequence. WM-specific
serialization avoids redundant list representations that could duplicate item
titles.
\\

Temporal Deadline
&
Inactive $\rightarrow$ none;
active with $\Delta_t\ge r_t$ $\rightarrow$ fired;
active with $\Delta_t<r_t$ $\rightarrow$ running with
$r_t-\Delta_t$ seconds remaining.
&
None: notification shade without timer state.
Fired: one Clock card with \emph{Time's up} and \emph{Stop}.
Running: one Clock/Timer card with the updated countdown and \emph{Pause}.
&
The resolved outcome is compiled directly into the native notification
request. Code2World uses its fixed-canvas format; gWorld uses a dedicated
compact $360\times800$ notification-shade request.
\\

\bottomrule
\end{tabular}
\end{table*}

\subsubsection{Exact Model-Native Request Structure}
\label{app:exact_wm_request_templates}

The preceding rules determine the semantic contents of the interface. We now
specify the literal model-visible request structure. Angle-bracketed fields in
the templates below are the only metavariables; each is deterministically
instantiated by the rules above. They are not free-form prompts chosen on a
per-example basis.

The serialized inference payload has the common outer form
\begin{Verbatim}[fontsize=\scriptsize,breaklines=true]
{
  "model": <MODEL_ID>,
  "messages": <MESSAGES>,
  "max_tokens": 8192,
  "temperature": 0.0,
  "seed": 20260912,
  "repetition_penalty": 1.0
}
\end{Verbatim}
with no top-$p$ or top-$k$ override. The model identifier is
\texttt{Code2World} for Code2World-8B and \texttt{gWorld-8B} for gWorld.
The observation image is serialized as an in-message image content part.
The image bytes themselves are benchmark inputs and therefore are not
reproduced as base64 text below.

\paragraph{Code2World fixed system message.}
Every Code2World request uses the following fixed system message.

\begin{Verbatim}[fontsize=\scriptsize,breaklines=true]
You are an expert **UI State Transition Simulator** and **Frontend Developer**.
Your task is to predict the **NEXT UI STATE** based on a screenshot of the current state and a user interaction.

### 1. IMAGE INTERPRETATION RULES
The input image contains visual cues denoting the user's action. You must interpret them as follows:
*   **Red Circle**: Indicates a **Click** or **Long Press** target at that location.
*   **Red Arrow**: Indicates a **Scroll** or **Swipe**.
    *   The arrow points in the direction of finger movement.
    *   *Example*: An arrow pointing UP means the finger slides up, pushing content up (Scrolling Down).
*   **Note**: These cues exist ONLY to show the action. **DO NOT render these red circles or arrows in your output HTML.**

### 2. CRITICAL STRUCTURAL RULES (MUST FOLLOW)
*   **Format**: Output ONLY raw HTML. Start with `<!DOCTYPE html>` and end with `</html>`.
*   **Root Element**: All visible content MUST be wrapped in:
    `<div id="render-target"> ... </div>`
*   **Container Style**: `#render-target` must have:
    `width: 1080px; height: 2400px; position: relative; overflow: hidden;`
    (Apply background colors and shadows here, NOT on the body).
*   **Body Style**: The `<body>` tag must have `margin: 0; padding: 0; background: transparent;`.
*   **Layout**: Do NOT center the body. Let `#render-target` sit at (0,0).

### 3. CONTENT GENERATION LOGIC
*   **Transition**: Analyze the action. If the user clicks a button, show the *result* (e.g., a menu opens, a checkbox checks, page navigates).
*   **Images**: Use semantic text placeholders. DO NOT use real URLs.
    *   Format: `<div style="...">[IMG: description]</div>`
*   **Icons**: Use simple inline SVG paths or Unicode.

### 4. OUTPUT REQUIREMENT
*   Do NOT generate Markdown blocks (```html).
*   Do NOT provide explanations or conversational text.
*   Output the code directly.
\end{Verbatim}

\paragraph{Persistent Artifact state/result serialization.}
For a Full Persistent request, the user message contains, in order,
(i) the current image content part,
(ii) one authoritative state/result text part, and
(iii) the WM-native transition text.

The authoritative text part has the following form:
\begin{Verbatim}[fontsize=\scriptsize,breaklines=true]
<authoritative_nonvisual_state_and_result_constraints>
<PERSISTENT_RESOLVED_JSON>
</authoritative_nonvisual_state_and_result_constraints>
These constraints are deterministic consequences of the supplied current state and action. They are authoritative and non-negotiable: make the resulting visible UI satisfy every must_show, must_not_show, exact value, outcome, count, and order constraint before adding styling. Never display the JSON, field names, reasoning, or debugging annotations.
\end{Verbatim}

For example, an absent Downloads-search branch is instantiated as
\begin{Verbatim}[fontsize=\scriptsize,breaklines=true]
<authoritative_nonvisual_state_and_result_constraints>
{"current_state":{"inspected_artifact_exists":false},"family":"persistent_artifact","required_post_action_consequence":{"inspected_path":"Downloads","must_show_exact_text":["No matches in Downloads"],"required_surface":"search_result","tracked_artifact_visible_after_action":false,"visible_scope_label":"Files in Downloads"}}
</authoritative_nonvisual_state_and_result_constraints>
These constraints are deterministic consequences of the supplied current state and action. They are authoritative and non-negotiable: make the resulting visible UI satisfy every must_show, must_not_show, exact value, outcome, count, and order constraint before adding styling. Never display the JSON, field names, reasoning, or debugging annotations.
\end{Verbatim}

In this absent branch, the model-visible entity context is projected to
\begin{Verbatim}[fontsize=\scriptsize,breaklines=true]
{"inspected_path":"Downloads","reveal_mode":"search_result"}
\end{Verbatim}
rather than retaining the target artifact identity or its source and
destination paths.

\paragraph{Persistent Artifact: Code2World user template.}
After the image content part and authoritative state/result block, Code2World
receives the following text. The fixed Files blueprint is selected
deterministically from the resolved present/absent state and probe type.

\begin{Verbatim}[fontsize=\scriptsize,breaklines=true]
<image>
### INPUT CONTEXT
1.  **User Intent**: "Predict the immediate Android UI after the supplied interaction."
2.  **Interaction Details**:
    *   **Description**: <CODE2WORLD_ACTION_DESCRIPTION>
    *   **Action Data**: <CODE2WORLD_ACTION_JSON>
Tracked entity specification (identity only, not current state): <PROJECTED_ENTITY_CONTEXT>
3.  **State-dependent transition contract**: Filesystem contract: entity paths identify one tracked artifact. For an open-folder action, render the requested folder and include the tracked artifact if and only if the supplied state says that exact path exists. For a filename search from the benchmark Files screen, keep the search scoped to Downloads: show the artifact if a tracked path under Downloads exists; otherwise show an explicit 'No matches in Downloads' empty result. Every visible search scope or result-location label must say Downloads. Do not copy the artifact from the screenshot when its supplied current-state value is absent.
4.  **Family-specific visible-result constraint**: Location is part of correctness. For open-folder, the breadcrumb and heading must end at the exact folder requested by the action, without appending the artifact name or a generic Files component; it must not show a software keyboard, IME, focused search field, or search-results surface. For filename search, visibly label the result surface 'Files in Downloads' and keep every result row in that scope. Keep the required keyboard complete with three conventional QWERTY letter rows; implement it compactly rather than adding decorative keyboard toolbar/icons. Keep HTML under 5000 tokens and close </html>.
5.  **Final mandatory visible-result check**: <FILES_BLUEPRINT>

### COMMAND
Based on the visual cues in the image, the interaction data, the transition contract, and any supplied nonvisual current-state block, generate the **HTML for the RESULTING UI STATE** (what the screen looks like *after* this action).
\end{Verbatim}

For the absent Downloads-search example above,
\texttt{<FILES\_BLUEPRINT>} is
\begin{Verbatim}[fontsize=\scriptsize,breaklines=true]
FINAL FILES BLUEPRINT, visible top-to-bottom: [typed search bar] [filter chips] [Heading: Files in Downloads] [empty state: No matches in Downloads] [complete three-row QWERTY keyboard]. Render every bracketed element as visible GUI, not comments. Show no result row and never return a blank body.
\end{Verbatim}

\paragraph{Persistent Artifact: gWorld user template.}
gWorld has no separate system message. Its user message contains the current
image content part, the same authoritative state/result block, and the
following model-native text:

\begin{Verbatim}[fontsize=\scriptsize,breaklines=true]
You are an expert mobile UI World Model that can accurately predict the next state given an action.
Critical state application: <FILES_BLUEPRINT>
Given a screenshot of a mobile interface and an action, you must generate clean, responsive HTML code that represents the state of the interface AFTER the action is performed.
First generate reasoning about what the next state should look like based on the action.
Afterwards, generate the HTML code representing the next state that logically follows the action.
You will render this HTML in a mobile viewport to see how similar it looks and acts like the mobile screenshot.

Requirements:
1. Provide reasoning about what the next state should look like based on the action
2. Generate complete, valid HTML5 code
3. Make all structure and layout self-contained. Put layout CSS in a local `<style>` block; do not use CDN scripts, linked stylesheets, Tailwind/Bootstrap/MUI utility classes, or remote images. Use a system sans-serif font fallback rather than importing a web font.
4. Render responsive mobile HTML for a 360px by 800px CSS viewport. Fill `100vw` by at least `100vh`; do not constrain the whole interface into a smaller inset mockup. The renderer captures this logical viewport at device-pixel-ratio 3 to produce a 1080px by 2400px image.
5. For images, use inline SVG placeholders with explicit width and height attributes that match the approximate dimensions from the screenshot. Matching the approximate color is also good.
6. Use modern web standards and best practices
7. Return ONLY the HTML code, no explanations or markdown formatting
8. The generated HTML should render properly without any network access.
9. Generated HTML should look like the screen that logically follows the current screen and the action.
10. State-dependent transition contract: Filesystem contract: entity paths identify one tracked artifact. For an open-folder action, render the requested folder and include the tracked artifact if and only if the supplied state says that exact path exists. For a filename search from the benchmark Files screen, keep the search scoped to Downloads: show the artifact if a tracked path under Downloads exists; otherwise show an explicit 'No matches in Downloads' empty result. Every visible search scope or result-location label must say Downloads. Do not copy the artifact from the screenshot when its supplied current-state value is absent.
11. Treat any supplied nonvisual current-state block as authoritative, apply it to the action, and never render its field names.
Tracked entity specification (identity only, not current state): <PROJECTED_ENTITY_CONTEXT>
12. Authoritative state-conditioned result directive: <FILES_BLUEPRINT>
13. Keep the local CSS concise and prioritize a complete state-bearing `<body>`; finish `</html>` well within the token budget. Do not repeat CSS declarations.
14. Before finishing, verify that the visible body satisfies the Critical state application above. <GWORLD_FILES_SURFACE_CONSTRAINT>

Action:
<GWORLD_ACTION_JSON>

Output format:
# Next State Reasoning: <your reasoning about what the next state should look like>
# HTML: <valid_html_code>

Generate the next state reasoning and the next state in html:
\end{Verbatim}

For the filename-search surface, the final family-specific constraint is
\begin{Verbatim}[fontsize=\scriptsize,breaklines=true]
Render exactly one compact Android Files search screen: app bar, typed search field, filter chips, the visible heading 'Files in Downloads', exactly the result/empty area specified by the authoritative directive, and one compact three-row QWERTY keyboard. The typed query is not evidence that a result exists. If the directive says no matches, render 'No matches in Downloads' and no filename/result row; if it specifies one result, put exactly that one row immediately below the Downloads heading. Implement the keyboard as 26 short single-letter spans sharing one class; use no SVGs, file thumbnails, comments, or decorative panels. Use at most 14 CSS rules and finish complete HTML under 3500 tokens.
\end{Verbatim}

The wording above is the archived wording used by the reported gWorld
Persistent evaluation.

\paragraph{Temporal Deadline: Code2World user template.}
Full Temporal requests do not append a separate canonical-state JSON block.
The deterministic temporal rule first resolves the post-action outcome, and
the result is inserted directly into the model-native request. After the fixed
Code2World system message and current-image content part, the user text is

\begin{Verbatim}[fontsize=\scriptsize,breaklines=true]
<image>
### INPUT CONTEXT
1.  **User Intent**: "<TEMPORAL_RESULT_SPEC>"
2.  **Interaction Details**:
    *   **Description**: <CODE2WORLD_TEMPORAL_ACTION_DESCRIPTION>
    *   **Action Data**: <CODE2WORLD_TEMPORAL_ACTION_JSON>
3.  **State-dependent transition contract**: The post-action notification state has already been resolved from the supplied state and wait action into the User Intent. Treat that resolved description as complete and authoritative; do not recompute, expand, or supplement its state-bearing card.
4.  **Family-specific visible-result constraint**: Use one compact notification-shade shell and exactly the single card specified by User Intent. Keep HTML under 2200 tokens with at most 12 visible elements and 8 short CSS rules. Make each specified card line one visible text element in the stated order, then finish the complete </html> after the footer.
5.  **Final mandatory visible-result check**: <TEMPORAL_RESULT_SPEC>

### COMMAND
Based on the visual cues in the image, the interaction data, the transition contract, and any supplied nonvisual current-state block, generate the **HTML for the RESULTING UI STATE** (what the screen looks like *after* this action).
\end{Verbatim}

For a recovered active timer with $r_t=300$ seconds followed by a
$\Delta_t=360$ second wait, the deterministic rule selects
\begin{Verbatim}[fontsize=\scriptsize,breaklines=true]
FINAL POST-ACTION SCREEN SPECIFICATION. Render a recognizable Android notification shade containing one compact Clock notification card. The card's entire visible text, in top-to-bottom order, is exactly this three-line sequence: Clock → Time's up → Stop. The middle line is the sole state title and the last line is its primary button. Outside the card, use an icon-only compact quick-settings header plus the footer labels Manage and Clear all. Those three lines are the complete card; finish the page there.
\end{Verbatim}

and the action fields are
\begin{Verbatim}[fontsize=\scriptsize,breaklines=true]
{"action_type":"wait_then_open_notification_shade",
 "duration_seconds":360,
 "semantic":"Wait for the specified interval, then open the notification shade"}
\end{Verbatim}

For a running timer the same template substitutes the updated
\texttt{MM:SS} remainder and \emph{Pause}; for an inactive timer it substitutes
the no-timer notification-shade specification.

\paragraph{Temporal Deadline: gWorld user template.}
Temporal Full uses a dedicated short gWorld request rather than the generic
14-item gWorld request. The user message consists of the current-image content
part followed by

\begin{Verbatim}[fontsize=\scriptsize,breaklines=true]
Predict the mobile UI after the supplied action, then generate its complete self-contained HTML.

The action opens the notification shade over the whole screen. Discard the pre-action app layout; it must not remain behind or below the shade.

AUTHORITATIVE POST-ACTION RESULT (use once, do not restate or duplicate):
<TEMPORAL_RESULT_SPEC>

Use one simple Android-style 360x800 notification-shade DOM: one full-screen `<main>`, one `Notifications` heading, one rounded notification `<section>`, and one footer directly below the card containing `Manage` and `Clear all`. Inside the card, use one text element per specified line and exactly one visible `<button>` whenever `Stop` or `Pause` is specified; never omit that button. The arrows in the result specification mean top-to-bottom order and must not themselves be printed. Use a muted gray/blue shade and a white card so this is a recognizable GUI, not an answer sheet.

Use this exact safe box model in `<style>` (you may add only font, color, border-radius, and text styling): `*{box-sizing:border-box} html,body{margin:0;width:100%;height:100%;overflow:hidden} main{width:360px;height:800px;padding:24px 16px} section{width:328px;max-height:260px;padding:20px;margin:0 0 16px} footer{width:328px;display:flex;justify-content:space-between}`. Every card edge and footer label must remain inside x=16..344 and y=0..760.

Hard exclusions: no pre-action Settings/Date & time content, status bar, wall clock, quick-settings row, SVG, image, icon, decorative glyph, duplicated text/control/card, overlay, pseudo-element content, animation, transform, absolute/fixed positioning, external resource, script, or comment. Use no more than 10 visible body elements and 8 concise CSS rules. The body must fill 100vw by 100vh and the HTML must close within 1800 tokens.

Action: <GWORLD_TEMPORAL_ACTION_JSON>

Output exactly in this model-native format:
# Next State Reasoning: <one short sentence>
# HTML: <!DOCTYPE html><html>...complete HTML...</html>
\end{Verbatim}

The same deterministic
$\mathrm{none}/\mathrm{running}/\mathrm{fired}$ rule selects
\texttt{<TEMPORAL\_RESULT\_SPEC>}; no target observation or evaluation label
is consulted.

Clipboard and Ordered Collection use the same outer Code2World and generic
gWorld request structures as Persistent Artifact, with their corresponding
state blocks, family contracts, and surface specifications defined in
Table~\ref{tab:state_interface_families} and
Table~\ref{tab:full_interface_rules}. Neither uses the dedicated Temporal
message topology.

\subsubsection{Request Reconstruction Audit}

We verify the model-facing specification above against the requests used for
the reported staged experiment. For representative Persistent Artifact and
Temporal Deadline examples, we reconstruct all four conditions
\[
\{
\text{Canonical},
\text{+Consequence},
\text{+Surface},
\text{Full}
\}
\]
under both Code2World-8B and gWorld-8B, yielding 16 request payloads.

After reinserting the recorded image content and deterministic per-example
fields, each payload is canonicalized using the same UTF-8 JSON
serialization used during inference and hashed with SHA-256. All
$16/16$ reconstructed payloads exactly reproduce the request hashes archived
with the original predictions.

The archive retains request hashes, outputs, model usage, and decoding
metadata rather than a second copy of each raw HTTP body. The hash audit
therefore verifies equality between the reconstructed specification and the
original model-facing request boundary. In addition, the two explicitly
matched transitions are verified exhaustively:
removing the added consequence component from +Consequence recovers Canonical
byte-for-byte for all $400/400$ branches per model--family combination, and
removing the added surface component from +Surface similarly recovers
+Consequence for all $400/400$ branches.

The Full transition is deliberately not described as a one-span comparison.
It uses the native request structures specified above and therefore changes
the placement and serialization of already available state-derived
information.

\subsection{Rendering, Scoring, and Evaluation Coverage}
\label{app:state_interface_scoring}

The detailed family-specific semantic parser is specified in
Appendix~\ref{app:semantic_validation}. Briefly, generated HTML is rendered
offline in Chromium, and strict semantic evidence is taken from the initial
rendered viewport rather than from unrestricted raw DOM text. Hidden,
materially clipped, substantially occluded, and below-the-fold content cannot
satisfy the strict semantic target.

The same renderer and scorer are used for every condition of a given WM.
Code2World-8B uses the $1080\times2400$ evaluation viewport. gWorld uses its
native $360\times800$ CSS viewport and is captured at device scale factor $3$.

The complete staged experiment contains
\[
4\ \text{conditions}
\times
2\ \text{families}
\times
2\ \text{WMs}
\times
400\ \text{branches}
=
6{,}400
\]
predictions. All $6{,}400$ predictions have corresponding score records.
There are zero inference failures, zero rendering failures, zero missing
scores, and zero failures of the offline self-containedness gate. No branch is
silently removed or replaced.

A missing final HTML close tag is retained as a diagnostic rather than treated
as an automatic semantic failure, because Chromium may still recover and
render a semantically valid document. Such outputs remain in the denominator.

\subsection{World-Model Configuration}

The staged interface experiment uses frozen released WM checkpoints and
deterministic decoding. Table~\ref{tab:wm_inference_config} reports the
configuration.

\begin{table}[t]
\centering
\scriptsize
\setlength{\tabcolsep}{5pt}
\caption{
World-model configuration used for the staged state-interface evaluation.
}
\label{tab:wm_inference_config}

\begin{tabular}{lcc}
\toprule
 & Code2World-8B & gWorld-8B \\
\midrule
Checkpoint
& released Code2World-8B
& released gWorld-8B \\
Maximum output tokens
& 8192
& 8192 \\
Temperature
& 0
& 0 \\
Generation seed
& 20260912
& 20260912 \\
Repetition penalty
& 1.0
& 1.0 \\
Top-$p$/Top-$k$ override
& None
& None \\
WM parameter updates
& None
& None \\
\bottomrule
\end{tabular}
\end{table}

\subsection{Interpretation and Scope of the Interface Study}

The staged experiment separates two roles that are otherwise easy to
conflate. The predictive-state estimator determines \emph{what}
transition-relevant state is available from interaction history, whereas the
deterministic interface determines \emph{how} that recovered state is
transformed into a model-compatible post-action specification for an existing
frozen GUI world model.

Canonical state serialization alone is poorly utilized by the evaluated WMs.
Adding an explicit action-resolved consequence and an intermediate surface
specification improves several settings, but the largest gain appears only
when the same recovered state is compiled through the complete WM-native
interface. Because Full jointly changes action-conditioned consequence
resolution, surface integration, context projection, and native placement of
the resulting constraints, the Full--Surface difference is interpreted as the
effect of the complete compilation path rather than of any single internal
component.

The score decomposition further shows that the Full gain is not explained
solely by surface-validity checks. Most +Surface $\rightarrow$ Full gains for
Temporal Deadline involve a correction of the parsed timer outcome itself, and
Persistent Artifact on gWorld shows the same qualitative pattern. Surface
realization also contributes, particularly for Persistent Artifact on
Code2World. The final metric therefore reflects both state-dependent semantic
correctness and faithful realization of that consequence on the requested GUI
surface.

The Full pipeline is intentionally hybrid. For families with known transition
semantics, $g_f$ deterministically maps the recovered current state and
candidate action to the corresponding action-resolved consequence before
WM-native serialization. The frozen WM is then responsible for realizing that
resolved consequence as a coherent next GUI under its native generation
interface. This design isolates state recovery from state realization: the
former is evaluated directly at the structured-state level
(Appendix~\ref{app:mtd_ablation}), while the latter is evaluated through the
fixed downstream interface.

Taken together, the experiments characterize two complementary properties:
whether transition-relevant state can be recovered from interaction history,
and whether that recovered state can be made usable by an otherwise frozen GUI
world model. The staged interface study addresses the second property while
holding the recovered state fixed, whereas the estimator ablations address the
first under a shared downstream interface.
\section{Additional Robustness and Fidelity Evaluation}
\label{app:additional_evaluation}
\label{app:robustness}

This appendix provides the complete protocols for the distribution-shift and
next-observation-fidelity evaluations reported in the main paper. We consider
three complementary robustness settings:
\textbf{cross-application transfer},
\textbf{long-history stress}, and
\textbf{semantic interference}.
All three preserve the family-specific structured-state semantics while
changing the interaction context from which the state must be recovered.
We additionally specify the frozen evaluation subset, rendering protocol, and
feature metrics used to measure next-observation fidelity.

\subsection{Evaluation Protocol and Metrics}
\label{app:robustness_protocol}

The primary strict evaluation and the three distribution-shift suites use
different evaluation units.

\paragraph{Primary strict evaluation.}
StateAliasBench Strict contains 200 matched A/B pairs per state family.
For each pair, both branches must recover their respective canonical
structured states correctly:
\[
\operatorname{PairCorrect}_i
=
\mathbf{1}
\left[
\widehat Z_{i,A}
=
\widetilde Z_{i,A}
\;\land\;
\widehat Z_{i,B}
=
\widetilde Z_{i,B}
\right].
\]
PairCorrect is appropriate here because the purpose of the strict evaluation
is to verify that the estimator distinguishes both members of an aliased pair,
rather than recovering only one branch correctly.

\paragraph{Distribution-shift evaluation.}
Cross-app, $d$-Stress, and Interference instead contain 200 independent
trajectories per state family and do not have an A/B pairing structure.
We therefore report canonical-state Exact:
\[
\operatorname{Exact}
=
\frac{1}{N}
\sum_{i=1}^{N}
\mathbf{1}
\left[
\widehat Z_i
=
\widetilde Z_i
\right].
\]
PairCorrect is not defined for these unpaired suites. Consequently, the
absolute values reported for Strict and for the three distribution-shift
suites are not directly comparable.

All robustness evaluations are inference-only. The trained predictive-state
estimators remain frozen, and no parameter update, prompt adaptation, or
evaluation-set fine-tuning is performed on any robustness trajectory.

Table~\ref{tab:robustness_summary} reproduces the complete family-level
results summarized in the main paper.

\begin{table*}[t]
\centering
\small
\caption{
Predictive-state recovery under the primary strict setting and three
distribution shifts. Strict reports PairCorrect (\%) over 200 A/B pairs per
family; Cross-app, $d$-Stress, and Interference report canonical-state Exact
(\%) over 200 independent trajectories per family. Because these settings use
different evaluation units, Strict percentages should not be directly compared
with the distribution-shift percentages.
}
\label{tab:robustness_summary}

\begin{tabular}{lcccccccc}
\toprule
&
\multicolumn{4}{c}{\textbf{Specialists}}
&
\multicolumn{4}{c}{\textbf{Predictive State (Ours)}} \\
\cmidrule(lr){2-5}
\cmidrule(lr){6-9}

\textbf{Family}
& \textbf{Strict}
& \textbf{Cross-app}
& \textbf{$d$-Stress}
& \textbf{Interf.}
& \textbf{Strict}
& \textbf{Cross-app}
& \textbf{$d$-Stress}
& \textbf{Interf.} \\
\midrule

Persistent Artifact
& 85.5 & 87.0 & 87.0 & 87.0
& 86.0 & 75.5 & 82.5 & 82.5 \\

Clipboard
& 89.5 & 85.5 & 89.5 & 89.5
& 85.0 & 88.5 & 85.0 & 88.5 \\

Ordered Collection
& 88.0 & 89.0 & 87.0 & 86.0
& 82.0 & 83.0 & 85.5 & 79.0 \\

Temporal Deadline
& 88.5 & 83.5 & 85.0 & 87.5
& 86.5 & 72.0 & 88.0 & 89.5 \\

\midrule

\textbf{Average}
& \textbf{87.88}
& \textbf{86.25}
& \textbf{87.13}
& \textbf{87.50}
& \textbf{84.88}
& \textbf{79.75}
& \textbf{85.25}
& \textbf{84.88} \\

\bottomrule
\end{tabular}
\end{table*}

\paragraph{Aggregate behavior.}
Among the three distribution-shift suites, which share the same
single-trajectory Exact metric, cross-application transfer is the most
challenging for the unified estimator, with 79.75\% average Exact, compared
with 85.25\% under long-history stress and 84.88\% under semantic
interference. Within the cross-application suite, the lowest Exact accuracies
occur for Persistent Artifact (75.5\%) and Temporal Deadline (72.0\%), whereas
Clipboard reaches 88.5\%. These results indicate that, among the distribution
shifts evaluated here, changing application context produces the largest
aggregate degradation for the unified estimator. The primary Strict result is
reported separately with PairCorrect and is therefore not used as a numerical
baseline for these comparisons.

\subsection{Cross-Application Generalization}
\label{app:cross_app}

Cross-application evaluation tests whether the estimator preserves the same
structured-state semantics when the surrounding application context changes.
Each family contains 200 independent trajectories and retains the canonical
state schema used in the primary evaluation. The estimator remains frozen.

The construction changes either the application in which the relevant state
transition is realized or the application shown at the final prediction
boundary, depending on the family. Thus, the target semantics remain fixed
while visual appearance, application identity, paths, entities, or terminal
context change.

\paragraph{Persistent Artifact.}
Persistent-artifact operations in the primary setting are realized in
Android Files. In the cross-application suite, Copy, Move, and Delete
histories are instead realized through Markor and terminate in a Markor query
context. The application appearance, file identities, paths, and interaction
realization therefore change, while the prediction target remains
\[
\widetilde Z_t^{\mathrm{persistent}}
=
\{
\texttt{source\_exists},
\texttt{destination\_exists}
\}.
\]
The suite contains 67 Copy, 67 Move, and 66 Delete trajectories.

\paragraph{Clipboard.}
Clipboard transfer varies both the source and sink applications. The
evaluation uses Simple SMS Messenger, Chrome, Tasks.org, Joplin,
Simple Calendar, and Broccoli across 12 directed source-to-sink routes,
with source and sink always distinct. The target remains the current
clipboard type and complete payload:
\[
\widetilde Z_t^{\mathrm{clipboard}}
=
\{
\texttt{clipboard\_type},
\texttt{clipboard\_text}
\}.
\]
The state semantics are therefore independent of the applications through
which the payload is copied and subsequently queried.

\paragraph{Ordered Collection.}
For Ordered Collection, Retro Music remains the state-producing application,
while the final observation boundary is moved to an unrelated application.
The 200 trajectories terminate equally across Tasks.org, Joplin,
Simple Calendar, and Broccoli. This setting tests whether the estimator
preserves the ordered collection state after leaving the application in which
that state was constructed. The suite contains 80 playing-queue and
120 playlist trajectories, and the target remains
\[
\widetilde Z_t^{\mathrm{ordered}}
=
\{
\texttt{collection\_type},
\texttt{ordered\_items}
\}.
\]

\paragraph{Temporal Deadline.}
Temporal Deadline similarly preserves Google Clock as the state-producing
application but moves the final neutral observation to Tasks.org, Joplin,
Simple Calendar, or Broccoli, with 50 trajectories per sink application.
The structured target remains
\[
\widetilde Z_t^{\mathrm{temporal}}
=
\{
\texttt{timer\_active},
\texttt{remaining\_seconds}
\}.
\]
The evaluation therefore changes the terminal application context without
changing the timer-state semantics being recovered.

\paragraph{Interpretation.}
Cross-application transfer changes contextual cues that can otherwise correlate
with a state family while keeping the target schema fixed. The unified
estimator's lower aggregate score in this setting therefore indicates that
application context remains a meaningful source of distribution shift,
particularly for Persistent Artifact and Temporal Deadline. The evaluation
does not change the definition of the target state itself.

\subsection{Long-History Stress}
\label{app:history_depth}

Long-history stress tests whether transition-relevant state remains recoverable
as the amount of model-visible interaction preceding the prediction boundary
increases. We evaluate
\[
d\in\{4,8,16,32\},
\]
with 50 independent trajectories at each value of $d$ for every state family,
yielding 200 trajectories per family.

The operational meaning of $d$ is defined by each family-specific trajectory
construction as a controlled count of model-visible actions. The goal is not
simply to append arbitrary tokens, but to increase the amount of interaction
history through which the relevant state evidence must be retained.

\paragraph{Persistent Artifact.}
The target artifact state is first established and then followed by $d$
state-preserving interference actions before the final neutral context.
These additional actions do not alter the canonical source/destination state
of the tracked artifact.

\paragraph{Clipboard.}
After the final effective Copy event establishes the target clipboard state,
the trajectory inserts $d$ additional model-visible actions that do not modify
the current clipboard. Recovery therefore requires preserving the effective
payload despite increasing intervening interaction.

\paragraph{Ordered Collection.}
The target playlist or playing queue is first constructed, after which
state-preserving actions increase the visible history length before the final
prediction boundary. The canonical ordered contents remain unchanged.

\paragraph{Temporal Deadline.}
Temporal Deadline requires a different construction because the state itself
evolves with time. Here, $d$ controls the length of the visible
timer-interaction history rather than simply inserting neutral actions after a
single fixed state transition. The final timer-state schema and prediction
boundary remain controlled across the evaluation.

For each family, the aggregate $d$-Stress result is computed over all
200 trajectories:
\[
\operatorname{Exact}_{d\text{-Stress}}
=
\frac{
\sum_{d\in\{4,8,16,32\}}
n_d^{\mathrm{correct}}
}{200}.
\]
Since every depth contributes 50 trajectories, each depth contributes equally
to the aggregate score.

Importantly, this evaluation tests whether evidence occurring farther back in
the observed interaction history can still be recovered into the current
structured state. It does not test, and should not be interpreted as testing,
multi-step predictive sufficiency of the canonical state
$Z_t^\star$ defined in Section~\ref{sec:predictive_information_state}.

\subsection{Semantic Interference}
\label{app:semantic_interference}

History length and semantic ambiguity are distinct sources of state-recovery
difficulty. We therefore construct an interference suite with the controlled
history depth fixed at $d=8$ while varying the semantic similarity of competing
history events.

Each family contains 200 independent trajectories split into three
interference levels:
70 low,
70 medium, and
60 high.
The precise mechanism is family-specific, but the progression is designed to
make distractor events increasingly similar to the target state updates while
holding the evaluation size and final state schema fixed.

\paragraph{Persistent Artifact.}
The progression ranges from unrelated navigation, to cancelled operations on
unrelated files, and finally to near-name, same-extension file distractors.
These events increase semantic similarity to the tracked artifact operations
without changing the target artifact state.

\paragraph{Clipboard.}
The trajectory contains increasingly relevant application interactions while
ensuring that no additional effective Copy event changes the target clipboard
payload after it has been established.

\paragraph{Ordered Collection.}
Interference progresses from generic unrelated text to song- and
collection-related distractors, including target-item or playlist names.
The effective ordered target collection remains fixed.

\paragraph{Temporal Deadline.}
The suite increases the amount and semantic relevance of competing timer
interactions while maintaining a controlled final timer-state stratum and
prediction boundary. Because timer state evolves during the history, control
is imposed on the final canonical state rather than by requiring every
intermediate event to be state-preserving.

For Persistent Artifact, Clipboard, and Ordered Collection, the interference
sequence is explicitly constructed not to alter the target structured state.
Temporal Deadline instead controls the final timer-state stratum because its
state evolves during interaction.

The aggregate Interference score is a micro-average over all 200 trajectories:
\[
\operatorname{Exact}_{\mathrm{Interf}}
=
\frac{
n_{\mathrm{low}}^{\mathrm{correct}}
+
n_{\mathrm{medium}}^{\mathrm{correct}}
+
n_{\mathrm{high}}^{\mathrm{correct}}
}{70+70+60}.
\]
Because the three levels contain different numbers of trajectories, this
quantity is not the unweighted mean of the three level-wise percentages.

The fixed history depth separates this evaluation from $d$-Stress: any change
across interference levels is associated with increasingly confusable
interaction content rather than with increasing history length alone.

\subsection{Next-Observation Fidelity}
\label{app:visual_fidelity}

StateAliasBench primarily evaluates whether a predicted future realizes the
correct state-dependent semantic outcome. We additionally measure image-level
feature similarity to verify that improvements in state-sensitive correctness
are not obtained by sacrificing general next-observation fidelity.

\paragraph{Evaluation subset.}
We use a frozen set of 100 benchmark branches for each world model, balanced
across the four state families with 25 branches per family. For every selected
branch, the Observation and Predictive State conditions are compared against
the same recorded next observation $O_{t+1}$. The same branch identifiers and
ground-truth targets are used for Code2World and gWorld.

This protocol fixes the evaluated branches before comparing conditioning
methods, so differences in fidelity cannot arise from selecting different
examples for the two conditions.

\paragraph{Rendering protocol.}
Generated HTML is rendered offline with headless Chromium while external
network requests are disabled. Code2World uses a
$1080\times2400$ viewport at device scale factor 1. gWorld follows its native
rendering convention with a $360\times800$ CSS viewport at device scale
factor 3.

\paragraph{DINOv2 similarity.}
We use a local checkpoint corresponding to
\texttt{facebook/dinov2-giant}. Images are converted to RGB, resized with
shortest edge 256, and center-cropped to $224\times224$ using the associated
image processor. We use the final-layer CLS representation and
$\ell_2$-normalize the resulting feature vector.

\paragraph{SigLIP similarity.}
We additionally use a local checkpoint corresponding to
\texttt{google/siglip-so400m-patch14-384}. RGB images are resized to
$384\times384$ by the associated processor, and the projected image
representation returned by the model is $\ell_2$-normalized.

For either feature encoder $m$, let $P_m$ denote its image preprocessing and
$f_m$ its image representation. We compute
\[
e_m(I)
=
\frac{
f_m(P_m(I))
}{
\|f_m(P_m(I))\|_2
},
\]
and define prediction fidelity as cosine similarity to the recorded next
observation:
\[
s_m(\widehat O_{t+1},O_{t+1})
=
e_m(\widehat O_{t+1})^\top
e_m(O_{t+1}).
\]
For world model $w$ and conditioning method $c$, the reported score is the
arithmetic mean over the 100 selected branches:
\[
F_{m,w,c}
=
\frac{1}{100}
\sum_{i=1}^{100}
s_m(
\widehat O^{\,w,c}_{i,t+1},
O_{i,t+1}
).
\]

\begin{table}[t]
\centering
\small
\caption{
Next-observation feature similarity on the frozen 100-branch fidelity subset.
Each generated prediction is compared with the corresponding recorded next
observation.
}
\label{tab:fidelity_appendix}

\begin{tabular}{llrrr}
\toprule
WM & Metric & Observation & Ours & $\Delta$ \\
\midrule

Code2World
& DINOv2 & 0.4718 & 0.6218 & +0.1499 \\
& SigLIP & 0.7234 & 0.8552 & +0.1318 \\

\midrule

gWorld
& DINOv2 & 0.5534 & 0.6685 & +0.1151 \\
& SigLIP & 0.8021 & 0.8285 & +0.0265 \\

\bottomrule
\end{tabular}
\end{table}

The scorer uses fail-fast validation rather than excluding unsuccessful
samples. All selected comparisons and feature embeddings are present in the
archived evaluation, so every reported value is computed over the same
100 selected branches for the corresponding world model. No failed sample is
removed in a way that would change the evaluated subset across conditions.
\section{Downstream Evaluation in AndroidWorld}
\label{app:androidworld}

We evaluate whether predictive-state-conditioned world modeling transfers to
downstream interactive performance in AndroidWorld. The evaluation uses the
full set of 116 AndroidWorld tasks and considers three conditions:
\textsc{Agent}, \textsc{Agent+WM}, and \textsc{Agent+WM+State}. We evaluate
GPT-5.4 and Qwen3-VL-8B-Instruct agents.

In both WM-assisted settings, future prediction is performed by the released
Code2World-8B checkpoint. GPT-5.4 or Qwen3-VL-8B-Instruct is used as the
corresponding agent and action selector; these policy models are separate from
the Code2World checkpoint. Code2World-8B predicts action-conditioned future
GUIs as renderable HTML, which is converted into future screenshots before
action selection.

\subsection{Experimental Protocol}
\label{app:androidworld_protocol}

All three conditions are evaluated on the same 116 task instances. The
interaction budget follows the AndroidWorld protocol, with a maximum of twice
the human demonstration length for each task. No additional episode-level
timeout is imposed.

Under \textsc{Agent}, the policy directly interacts with the environment
without world-model assistance.

Under \textsc{Agent+WM}, the agent proposes candidate actions from the current
observation and task context. When world-model-assisted selection is invoked,
three candidate actions are evaluated separately by Code2World-8B. For each
candidate, Code2World-8B predicts the next GUI as HTML, and the prediction is
rendered into a future GUI image. The action selector then receives the current
observation, the candidate actions and their descriptions, and the
corresponding predicted future observations, and selects the action to execute.
The GPT-5.4 pipeline additionally permits a valid single-direct proposal to be
executed without invoking world-model-based candidate selection.

\textsc{Agent+WM+State} retains the same agent, Code2World-8B checkpoint,
candidate-generation procedure, rendering pipeline, and action selector as
\textsc{Agent+WM}. Before world-model prediction, the predictive-state
estimator recovers transition-relevant state from the model-visible interaction
history. The recovered state is incorporated only into the Code2World request
through the deterministic state-conditioning interface described in
Appendix~\ref{app:state_wm_interface}. It is not provided directly to either
the candidate-generating agent or the action selector. Its downstream effect
is therefore mediated through the predicted future GUI.

State recovery queries the four state schemas considered in this work:
Persistent Artifact, Clipboard, Ordered Collection, and Temporal Deadline.
When no state family is applicable, the system falls back to the standard
world-model condition. In the evaluated AndroidWorld trajectories, no case
produced multiple simultaneously applicable non-null state predictions.

\subsection{Downstream Results}
\label{app:androidworld_results}

Table~\ref{tab:androidworld_results} reports task success over the full
116-task benchmark. Adding Code2World-8B increases success from 56.03\% to
58.62\% for GPT-5.4 and from 41.38\% to 44.83\% for Qwen3-VL-8B.
Predictive-state conditioning further increases success to 60.34\% and
47.41\%, corresponding to absolute improvements of $+1.72$ and $+2.58$
percentage points over \textsc{Agent+WM}, respectively.

Relative to the improvement obtained by adding the WM alone, the additional
gains from predictive-state conditioning correspond to $66.4\%$ for GPT-5.4
and $74.8\%$ for Qwen3-VL-8B.

\begin{table}[t]
\centering
\small
\caption{
AndroidWorld task success (\%) on the full 116-task benchmark.
Code2World-8B is used as the world model in both WM-assisted conditions.
The final column reports the additional gain from predictive-state
conditioning relative to the improvement obtained by adding the WM alone.
Best task-success results are in \textbf{bold}; second-best results are
underlined.
}
\label{tab:androidworld_results}

\begin{tabular}{lcccc}
\toprule
Backbone
& \textsc{Agent}
& \textsc{Agent+WM}
& \textsc{Agent+WM+State}
& Rel. Gain \\
\midrule

GPT-5.4
& 56.03
& \underline{58.62}
& \textbf{60.34}
& +66.4\% \\

Qwen3-VL-8B
& 41.38
& \underline{44.83}
& \textbf{47.41}
& +74.8\% \\

\bottomrule
\end{tabular}
\end{table}

For Qwen3-VL-8B, the paired task-level comparison between
\textsc{Agent+WM} and \textsc{Agent+WM+State} yields 52 tasks solved by both
conditions, 61 failed by both, three state-conditioned rescues, and no
regressions. The rescued tasks are \texttt{FilesDeleteFile},
\texttt{FilesMoveFile}, and \texttt{SaveCopyOfReceiptTaskEval}. The aggregate
improvement therefore corresponds to a small number of additional solved
tasks rather than an exchange between previously solved and newly solved
instances.

\subsection{Representative Rescue Trajectory}
\label{app:androidworld_rescue}

Figure~\ref{fig:androidworld_rescue} presents a paired trajectory for
\texttt{FilesDeleteFile} with Qwen3-VL-8B. The
\textsc{Agent+WM} and \textsc{Agent+WM+State} conditions share the same
interaction prefix through step~5 and the same current observation at
step~6. At this point, their downstream behavior diverges through the
world-model prediction and subsequent action selection.

Under \textsc{Agent+WM}, the selected predicted future contains the near-name
file \texttt{q2a8\_banana.mp3}, whereas the target artifact is
\texttt{q2a8\_fancy\_banana.mp3}. The selector subsequently chooses
\texttt{SWIPE UP}; the trajectory continues through additional navigation and
eventually reaches the step limit without completing the task.

Under \textsc{Agent+WM+State}, the predictive-state estimator recovers a
Persistent Artifact state indicating that the source artifact exists and the
destination is absent. This recovered state is supplied only to the
Code2World-8B world-model interface. The resulting state-conditioned future
avoids the misleading near-name match, and the selector instead chooses
\texttt{SEARCH}. The subsequent trajectory searches for the exact target,
locates and long-presses the file, confirms deletion, and reaches the
completion condition after the recovered source state changes to absent.

\begin{figure*}[t]
    \centering
    \includegraphics[width=\textwidth]
    {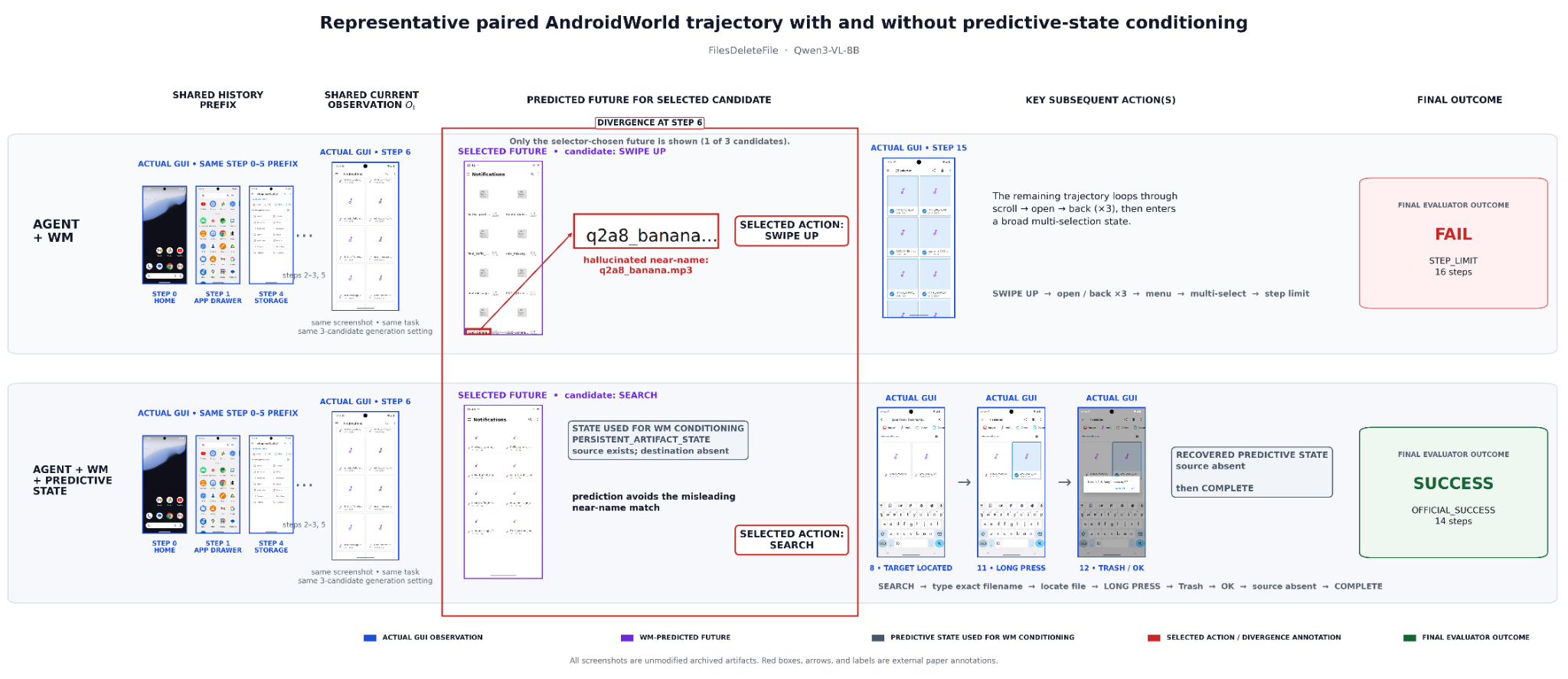}
    \caption{
    \textbf{Representative paired AndroidWorld rescue on
    \texttt{FilesDeleteFile}.}
    The \textsc{Agent+WM} and \textsc{Agent+WM+State} conditions share the
    same interaction prefix and current observation at the divergence point.
    Standard Code2World-8B prediction produces a misleading near-name file
    match and leads the selector to \texttt{SWIPE UP}. Predictive-state
    conditioning changes the simulated future and leads the selector to
    \texttt{SEARCH}, followed by successful localization and deletion of the
    target artifact. The recovered state is supplied only to the world model
    and is not directly exposed to the action selector.
    }
    \label{fig:androidworld_rescue}
\end{figure*}
\section{Limitations}
\label{app:limitations}

StateAliasBench instantiates state aliasing through four families:
Persistent Artifact, Clipboard, Ordered Collection, and Temporal Deadline.
These families cover distinct mechanisms by which transition-relevant
information can remain unobserved in the current GUI, but they represent only
a subset of the latent state variables encountered in general-purpose
interactive systems. The present empirical conclusions should therefore be
interpreted within this scope; broader coverage of application-specific and
cross-application state mechanisms remains necessary to characterize state
aliasing more comprehensively.

The structured state $\widetilde Z_t^f$ used in our implementation is a
family-specific approximation to the canonical predictive state
$Z_t^\star$. It is therefore not expected to recover all
transition-relevant information contained in the interaction history. In
general,
\[
I(H_t;O_{t+1}\mid O_t,\widetilde Z_t^f,A_t)
\]
may remain nonzero, leaving residual representation error even after state
recovery. In addition, supplying transition-relevant state does not eliminate
the approximation error of the frozen world model itself. Consequently, the
proposed state-conditioning mechanism can reduce both state-induced ambiguity
and downstream transition error, but does not imply complete recovery of the
underlying transition law.

\end{document}